\documentclass[11pt]{article}

\usepackage[final]{acl}

\usepackage{times}
\usepackage{latexsym}

\usepackage[T1]{fontenc}

\usepackage[utf8]{inputenc}

\usepackage{microtype}

\usepackage{inconsolata}

\usepackage{graphicx}

\usepackage{amsmath, amssymb, amsthm}

\newtheorem{theorem}{Theorem}[section]

\newtheorem{lemma}[theorem]{Lemma}

\theoremstyle{definition}

\newtheorem{assumption}[theorem]{Assumption}
\theoremstyle{remark}

\DeclareMathOperator*{\argmax}{top\_K\_arg\,max}

\usepackage{algorithm}
\usepackage{algorithmic}

\usepackage{enumitem}
\usepackage{multirow}

\usepackage{booktabs}

\usepackage{pifont}

\usepackage{makecell}

\usepackage{float}

\usepackage{url}

\usepackage{cleveref}

\crefname{subsection}{section}{Sections}    
\Crefname{subsection}{Section}{Sections}

\usepackage{graphicx}
\usepackage{subcaption}

\usepackage{tcolorbox}

\usepackage[table]{xcolor}

\newsavebox{\myentiretablebox}

\definecolor{backblue}{RGB}{210, 230, 250}

\title{Utility-Oriented Visual Evidence Selection for Multimodal Retrieval-Augmented Generation}

\author{
  \textbf{Weiqing Luo}\textsuperscript{\ding{169}} \quad 
  \textbf{Zongye Hu}\textsuperscript{\ding{169}} \quad
  \textbf{Xiao Wang} \textsuperscript{\ding{169}} \quad \\
  \textbf{Zhiyuan Yu}\textsuperscript{\ding{170}} \quad 
  \textbf{Haofeng Zhang}\textsuperscript{\ding{171}} \quad
  \textbf{Ziyi Huang}\textsuperscript{\ding{169}} \\ 
  \textsuperscript{\ding{169}}Arizona State University \quad
  \textsuperscript{\ding{170}}Texas A\&M University \quad
  \textsuperscript{\ding{171}}Morgan Stanley \\
  {\tt \{weiqing2,zongyehu,xwang213,zhuan236\}@asu.edu} \\
  {\tt zhiyuanyu@tamu.edu} \quad
  {\tt hz2553@columbia.edu}
}

\begin{document}
\maketitle
\begin{abstract}
Visual evidence selection is a critical component of multimodal retrieval-augmented generation (RAG), yet existing methods typically rely on semantic relevance or surface-level similarity, which are often misaligned with the actual utility of visual evidence for downstream reasoning. We reformulate multimodal evidence selection from an information-theoretic perspective by defining evidence utility as the information gain induced on a model’s output distribution. To overcome the intractability of answer-space optimization, we introduce a latent notion of evidence helpfulness and theoretically show that, under mild assumptions, ranking evidence by information gain on this latent variable is equivalent to answer-space utility. We further propose a training-free, surrogate-accelerated framework that efficiently estimates evidence utility using lightweight multimodal models. Experiments on MRAG-Bench and Visual-RAG across multiple model families demonstrate that our method consistently outperforms state-of-the-art RAG baselines while achieving substantial reductions in computational cost. We release our code at \url{https://github.com/Hcnaeg/utility-mrag}.
\end{abstract}

\section{Introduction}

Retrieval-Augmented Generation (RAG) has emerged as a powerful paradigm for improving the factuality and coverage of large language models (LLMs) by grounding generation in external evidence~\cite{lewis2020retrieval,gao2023retrieval}. In multimodal RAG, visual inputs are retrieved and incorporated alongside textual context to support downstream reasoning and response generation~\cite{mei2025survey}. Under visual setting, evidence selection is further complicated by the weak alignment between retrieval relevance and generative utility, whereby visually relevant images frequently do not contribute meaningful information for answering a query, despite appearing highly relevant under standard retrieval metrics~\cite{mortaheb2025rag,wu2024visual}. This raises a fundamental challenge: \textit{how should we select evidence that most effectively supports a multimodal model’s generation?}

\begin{figure}[t]
  \centering
  \includegraphics[width=0.5\textwidth]{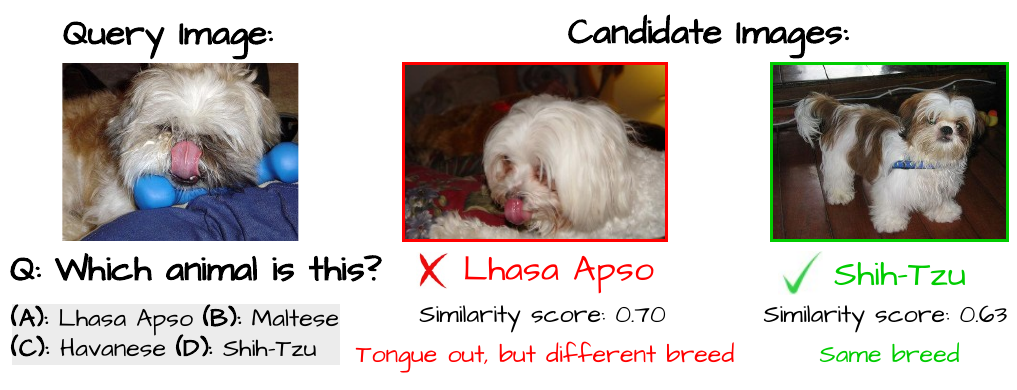}
  \vspace{-9mm}
  \caption{Illustrative example of the misalignment between relevance and utility in multimodal RAG. Although the left candidate image achieves a higher similarity score due to a shared salient attribute (tongue out), it depicts a different breed and does not provide discriminative information for answering the query. }
  
\vspace{-6mm}
\label{fig:intro_case}
\end{figure}

Existing multimodal RAG approaches predominantly rely on semantic relevance or similarity-based heuristics~\cite{mei2025survey,radford2021learning_clip,jia2021scaling_ALIGN,li2021align_ALBEF,lin2024mmembed}. While effective for improving retrieval recall, such criteria are only weakly connected to a model’s downstream generation behavior, implicitly assuming that relevance correlates with actual usefulness for answering a query. In practice, this assumption often breaks down: similarity-based signals, whether computed by contrastive vision-language models (e.g., CLIP-style encoders) or by relevance scorers built on multimodal large language models (MLLMs), frequently fail to reflect whether an image provides task-critical, discriminative information needed for correct reasoning~\cite{yang2022empirical,chen2025seeing}. More fundamentally, relevance-driven methods lack a principled mechanism to quantify \textit{how}, and \textit{to what extent}, a candidate evidence instance influences the model’s output distribution~\cite{nematov2025source}. As a result, they are ill-suited for utility-aware evidence selection in multimodal RAG. As illustrated in Figure~\ref{fig:intro_case}, visually salient yet task-irrelevant attributes can dominate similarity-based rankings, leading to the selection of evidence that appears plausible but ultimately fails to support the correct answer.

These limitations indicate that existing relevance-driven selection strategies are insufficient for multimodal RAG, motivating a more precise research question (\hypertarget{RQ1}{\textbf{\textit{RQ1}}}):
\textit{How can we select visual evidence that is truly \textbf{helpful} to multimodal RAG, beyond semantic relevance?}

To address \hyperlink{RQ1}{\textbf{\textit{RQ1}}}, we adopt an information-theoretic perspective and argue that evidence selection should be guided by a model-centric notion of utility that reflects how visual evidence influences a multimodal model’s generation (Section~\ref{sec: theory}). We formalize this notion as the information gain induced by conditioning on a candidate evidence instance (Section~\ref{sec:info_gain}). However, directly optimizing information gain in the answer space is impractical for MLLMs due to the high dimensionality and implicit nature of the output distribution, as well as linguistic and decoding variability. We therefore introduce a latent notion of evidence helpfulness that abstracts whether a visual input is useful from the model’s perspective, yielding a low-dimensional and tractable objective (Section \ref{sec:latent}). Under mild assumptions, we theoretically show that ranking evidence by information gain on this latent helpfulness variable is equivalent to ranking evidence in the full answer space, providing a principled alternative to relevance-based selection (Sections~\ref{sec:latent} and \ref{sec:odd_logits}).

Although the proposed utility criterion is theoretically well founded, naively computing it at scale with a high-capacity multimodal model is computationally prohibitive in practice. This practical consideration leads to our second research question (\hypertarget{RQ2}{\textbf{\textit{RQ2}}}):
\textit{Can evidence utility be estimated \textbf{efficiently} without relying on expensive main-model inference, while preserving selection quality?}

To address \hyperlink{RQ2}{\textbf{\textit{RQ2}}}, we propose a surrogate-accelerated evidence selection strategy that decouples utility estimation from expensive main-model inference (Section~\ref{sec:pipeline}). Our key observation is that assessing whether a visual input is helpful for answering a query is often a simpler discriminative judgment than full answer generation, and can be reliably approximated by a lightweight surrogate model. Leveraging this observation, we deploy a compact multimodal model to efficiently estimate evidence helpfulness and rank candidate evidence, before invoking the high-capacity target model only once for final generation. This design enables scalable utility-aware evidence selection while substantially reducing computational cost.

Our contributions can be summarized as follows:
\begin{itemize}
[leftmargin=*,itemsep=0pt]
\item \textbf{Problem formulation.} We identify the misalignment between semantic relevance and evidence utility in multimodal RAG, and formulate evidence selection as an information-theoretic problem by defining utility as information gain on the model’s output distribution.

\item \textbf{Theoretical Reformulation of Evidence Utility.} We resolve the intractability of answer-space information gain via a theoretically grounded latent helpfulness formulation, yielding a stable and tractable optimization objective. 

\item \textbf{Surrogate-Accelerated Selection.} We propose a training-free, surrogate-accelerated pipeline that efficiently estimates evidence helpfulness using a lightweight model, substantially reducing computational cost while preserving selection quality.

\end{itemize}

\section{Related Work}
\vspace{-2mm}
We briefly review representative work on multimodal RAG and evidence selection, with a more comprehensive discussion deferred to Appendix~\ref{sec:appendix_related_work}.

\textbf{Visual Evidence Selection for Multimodal RAG.} Early multimodal RAG studies rely on CLIP-style dual encoders to retrieve images via semantic similarity in a shared embedding space~\cite{radford2021learning_clip,cherti2023reproducible_openclip,tschannen2025siglip,zhou2025megapairs,jia2021scaling_ALIGN,li2021align_ALBEF}. 
More recent approaches leverage LLMs as dense retrievers~\cite{jiang2024e5,lin2024mmembed,meng2025vlm2vecV2,zhou2025megapairs,zhang2024gme} or relevance scorers, using pointwise or listwise judgments to improve retrieval quality~\cite{gu2025unime,liu2025lamra}. Despite architectural differences, these methods remain relevance-driven, selecting images based on semantic matching rather than their discriminative utility for answering a query. In contrast, our work explicitly models evidence utility rather than relevance.

\textbf{Uncertainty, Utility, and Efficiency in Evidence Selection.} Previous studies have explored model confidence via uncertainty quantification for LLMs, including likelihood- and entropy-based measures, Monte Carlo (MC) sampling, and consistency-based signals~\citep{kadavath2022language,kuhn2023semantic,wang2022self,lin2023generating}. In RAG settings, such signals have occasionally been used for auxiliary purposes such as triggering adaptive retrieval or assessing answer reliability~\citep{jiang2023active,xu2025mega,soudani2025why,ozaki2025understanding}. However, these methods are primarily designed to assess confidence in generated answers rather than to guide evidence selection, and when applied directly to RAG, answer-level uncertainty is ill-suited for comparing candidate evidence due to confounding generation noise and decoding variability. Moreover, these studies often require retraining or repeated sampling, incurring substantial computational cost. 

\section{Information-Theoretic Formulation of Evidence Selection}
\label{sec: theory}
\vspace{-2mm}

In this section, we study our core research problem (\hyperlink{RQ1}{\textbf{\textit{RQ1}}}) and present an information-theoretic formulation of evidence selection for RAG. Let $q$ denote the input query, $\mathcal{C}=\{c_i\}_{i\in I}$ a set of retrieved candidate evidences, $Y$ the model output, and $Z$ a latent variable representing evidence helpfulness. Given $q$ and $\mathcal{C}$, the goal is to select an optimal subset of top $K\ge 1$ evidences (in particular if $K=1$, a single evidence instance $c^* \in \mathcal{C}$) that best assists the LMM, denoted as $\mathcal{M}$, in generating $Y$.

\vspace{-2mm}
\subsection{Evidence Utility as Information Gain}
\label{sec:info_gain}
From an information-theoretic perspective, utility refers to \textit{predictive influence}, i.e., the extent to which an evidence instance $c$ alters the model’s belief about the output, rather than \textit{semantic relevance/similarity}, which may be redundant or misleading. Then, the utility of $c$ can be quantified via \textbf{Information Gain (IG)}, which measures the reduction in the uncertainty (entropy) of the model’s output distribution after conditioning on $c$. This is equivalent to the KL divergence between the posterior distribution (with evidence) and the prior distribution (without evidence): 
{
\setlength{\abovedisplayskip}{3pt}
\setlength{\belowdisplayskip}{3pt}
\begin{align}
\label{eq:utility_kl}
\text{Utility}(c) & := \text{IG}(Y; C=c \mid q)  \nonumber \\
& = D_{\text{KL}}\Big( P_{Y \mid C=c, q} \parallel P_{Y \mid q} \Big),
\end{align}
}

\noindent where $P_{Y \mid \cdot}$ represents the generative distribution of $Y$ generated by $\mathcal{M}$ conditional on the event $\cdot$. Under this formulation, an evidence candidate $c$ has high utility if and only if it induces a significant shift in the model's belief about the answer compared to its internal parametric knowledge. Therefore, our goal is to solve the following optimization problem:
{
\setlength{\abovedisplayskip}{3pt}
\setlength{\belowdisplayskip}{3pt}
\begin{equation}
\label{eq:utility_kl_max}
\argmax_{c \in \mathcal{C} }  \text{IG}(Y; C=c \mid q), 
\end{equation}}
where $\argmax$ represents the top $K$ solutions for which the objective obtains highest values.

While theoretically sound, directly solving Problem~(\ref{eq:utility_kl_max}) is computationally intractable. First, both $P_{Y \mid q}$ and $P_{Y \mid C=c, q}$ are implicit generative distributions induced by MLLMs, which are inaccessible in closed form. Second, the output space of $Y$ is high-dimensional and unbounded, making KL computation prohibitively expensive. Third, answer-space divergence conflates visual information gain with linguistic variability and hallucination noise, leading to unstable utility estimates. Hence, directly optimizing Problem (\ref{eq:utility_kl_max}) in the high-dimensional answer space is computationally intractable and statistically unstable. 

\vspace{-2mm}
\subsection{Reformulating Answer-Space Utility via Latent Helpfulness}
\label{sec:latent}
Due to these challenges, we therefore introduce a latent Bernoulli random variable $Z \in \{0,1\}$ to indicate whether a candidate evidence instance is \textit{helpful} to the model’s generation. Specifically, $\{Z=1 \mid C=c, q\}$ represents the event that a candidate evidence $c$ is \textit{semantically helpful} for answering the query $q$. We define a new optimization problem as follows: 
{
\begin{equation}
\label{eq:latent_kl_max}
\argmax_{c}  \text{IG}(Z; C=c \mid q).    
\end{equation}}
The latent variable $Z$ serves as a conceptual bridge that replaces an intractable divergence in the answer space with a tractable estimation objective, shifting the optimization from directly comparing high-dimensional output distributions to modeling a lower-dimensional, binary helpfulness outcome induced by each evidence instance. We next establish the theoretical validity of this reformulation.

\begin{assumption}[Constraints on helpful evidences]
\label{assumption:helpfulevidence_main}
Throughout this paper, we restrict attention to a feasible subset $\hat{\mathcal{C}} \in \mathcal{C}$ of candidate evidences that are not detrimental to the model’s generation. Formally, we define $\hat{\mathcal{C}} = \{c \in \mathcal{C}: p_c \ge \bar{p}\}$, where $p_c = P(Z=1\mid C=c, q)$, $\bar{p} = P(Z=1\mid q)$. 
\end{assumption}
This assumption means that our analysis focuses on ranking evidence within a feasible candidate set that is at least non-harmful to the model, rather than handling explicitly adversarial or systematically misleading inputs. 

\begin{assumption}[Connection between response and helpfulness; Informal; See Assumption \ref{assumption:helpfulnessscore} in  Appendix \ref{sec:proof} for the formal version]
\label{assumption_informal:helpfulnessscore}
The conditional response distribution of the model admits a monotone mixture structure with respect to evidence helpfulness.
\end{assumption}

\begin{figure*}[!htb]
  \centering
  \includegraphics[width=\textwidth]{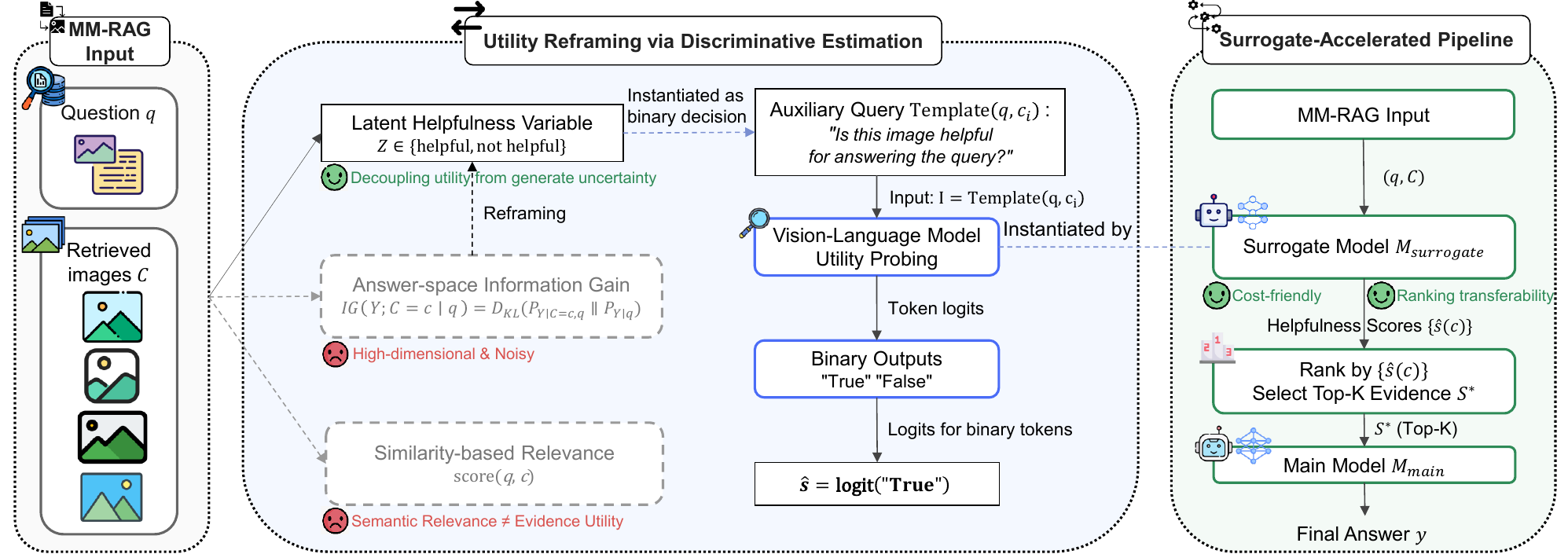}
  \caption{Overview of the proposed utility-oriented visual evidence selection framework for multimodal RAG.
We reformulate evidence selection as discriminative utility estimation via a latent helpfulness variable, and employ a surrogate-accelerated pipeline to efficiently rank and select informative visual evidence for downstream generation.}

\vspace{-5mm}
\label{fig:main_framework}
\end{figure*}

Intuitively, this assumption states that more helpful evidence should induce a larger shift in the model's posterior distribution over responses. This justifies using latent helpfulness as a tractable proxy for answer-space utility.
\begin{theorem}
\label{thm:equivalent1}
Suppose that Assumptions \ref{assumption:helpfulevidence_main} and \ref{assumption_informal:helpfulnessscore} hold. Then we have
{
\setlength{\abovedisplayskip}{3pt}
\setlength{\belowdisplayskip}{3pt}
\begin{align*}
IG(Z;C=c_1)\ge IG(Z;C=c_2)\quad\Longrightarrow\quad \\ IG(Y;C=c_1)\ge IG(Y;C=c_2).
\end{align*}
}
for all $c_1, c_2 \in \{c\in \hat{C}: \lambda(c)\ge \bar\lambda\}$
and the optimal solutions to Problem \eqref{eq:latent_kl_max} must be optimal to Problem \eqref{eq:utility_kl_max}, i.e.,
{
\setlength{\abovedisplayskip}{3pt}
\setlength{\belowdisplayskip}{3pt}
\begin{align*}
\argmax_{c\in \hat{C}: \lambda(c)\ge \bar\lambda} IG(Z;C=c) \\ \in \argmax_{c\in \hat{C}: \lambda(c)\ge \bar\lambda} IG(Y;C=c).
\end{align*}
}
In particular, if the optimal solution to Problem \eqref{eq:utility_kl_max} is unique, then
{
\setlength{\abovedisplayskip}{3pt}
\setlength{\belowdisplayskip}{3pt}
\begin{align*}
\argmax_{c\in \hat{C}: \lambda(c)\ge \bar\lambda} IG(Z;C=c) \\ = \argmax_{c\in \hat{C}: \lambda(c)\ge \bar\lambda} IG(Y;C=c),
\end{align*}
}
for $K=1$. Formal definitions of $\lambda(c)$ and $\bar\lambda\ $ are presented in Appendix \ref{assumption:helpfulnessscore}.
\end{theorem}

By Theorem \ref{thm:equivalent1}, we can convert solving Problem \eqref{eq:utility_kl_max} to solving a new and easier Problem \eqref{eq:latent_kl_max}.

\subsection{Simplifying Latent Information Gain via Helpfulness Probability}
\label{sec:odd_logits}
Although Problem~\eqref{eq:latent_kl_max} is lower-dimensional than the answer-space objective in Problem~\eqref{eq:utility_kl_max}, it still requires estimating a KL divergence involving both conditional and marginal distributions of $Z$. Computing the information gain $IG(Z;C=c, q)$ in practice incurs additional computational and statistical overhead. To further simplify the objective, we introduce the following optimization problem:
{
\setlength{\abovedisplayskip}{3pt}
\setlength{\belowdisplayskip}{3pt}
\begin{equation}
\label{eq:latent_1_max}
\argmax_{c}  P(Z=1 \mid C=c, q).    
\end{equation}
}

\begin{theorem}
\label{thm:equivalent2}
Suppose that Assumption \ref{assumption:helpfulevidence_main} holds. 
Then we have
{
\setlength{\abovedisplayskip}{3pt}
\setlength{\belowdisplayskip}{3pt}
\begin{align*}
IG(Z;C=c_1)\ge IG(Z;C=c_2)
\quad\Longleftrightarrow\quad \\
P(Z=1\mid C=c_1)\ge P(Z=1\mid C=c_2).
\end{align*}
For all $c_1, c_2 \in \hat{C}$, Problem \eqref{eq:latent_kl_max} has the identical solutions as in Problem \eqref{eq:latent_1_max}, i.e., 
\begin{align*}
&\argmax_{c\in \hat{C}} IG(Z;C=c) \\ 
= &
\argmax_{c\in \hat{C}} P(Z=1\mid C=c).
\end{align*}}
\end{theorem}

By Theorem \ref{thm:equivalent2}, we thus can convert solving Problem \eqref{eq:latent_kl_max} to solving a new and easier Problem \eqref{eq:latent_1_max}. These results motivate a practical design principle: rather than scoring evidence by relevance or similarity, evidence should be selected by directly modeling helpfulness likelihood and ranking candidates accordingly.

\section{Surrogate-Accelerated Discriminative Utility Estimation}
\label{sec:overview}

Building upon our theoretical results, we propose a novel evidence selection framework, namely \textbf{Surrogate-Accelerated Discriminative Utility Estimation}, to answer \hyperlink{RQ2}{\textbf{\textit{RQ2}}}. As shown in Figure~\ref{fig:main_framework}, our framework follows a \textit{retrieve–rank–generate} paradigm with two core components: \textit{Utility-Centric Metric Design}, which estimates evidence helpfulness via a generation-aware utility objective, and \textit{Surrogate-Accelerated Execution}, which employs a lightweight surrogate model to efficiently rank candidate evidences for the target model.

\subsection{Utility-Centric Metric Design}
\label{sec:decoupling}

As discussed in Section~\ref{sec: theory}, directly optimizing evidence utility in the high-dimensional answer space is computationally intractable. Practically, we construct an \textbf{auxiliary decision task} to probe model's internal belief regarding the latent variable $Z$.

\paragraph{Auxiliary Query Construction.}
We instantiate the condition for $Z$ by formulating a meta-cognitive query auxiliary query $q_{\text{aux}}$ that explicitly asks the model to evaluate the relationship between the candidate evidence and the user's original intent. Formally, given a query $q$ and a candidate image $c_i$, the model input is
{
\setlength{\abovedisplayskip}{3pt}
\setlength{\belowdisplayskip}{3pt}
\begin{equation}
\label{eq:prompt}
\mathcal{I} = \text{Template}(q, c_i, q_{\text{aux}}),
\end{equation}}
where $q_{\text{aux}}$ corresponds to the question \textit{``Is this evidence helpful for answering the query?''} and target output space is constrained to a binary set $\mathcal{V}_{\text{bin}} = \{v^+: \text{"True"}, v^-: \text{"False"}\}$.

\paragraph{Score Computation via Logits.}
We compute the helpfulness score from the model’s final-layer logits. Let $\ell(v \mid \mathcal{I})$ denote the logit of token $v$ given input $\mathcal{I}$. Then, the surrogate helpfulness score $\hat{s}(c_i)$ is computed as:
{
\setlength{\abovedisplayskip}{3pt}
\setlength{\belowdisplayskip}{3pt}
\begin{equation}
\label{eq:score_calculation}
\hat{s}(c_i) = \ell(v^+ \mid \mathcal{I}).
\end{equation}
}
Theoretically, our method optimizes Problem~(\ref{eq:latent_1_max}), whose optima are guaranteed to be optimal for Problem~(\ref{eq:utility_kl_max}). Practically, $q_{\text{aux}}$ (Question~(\ref{eq:prompt})) directly assesses evidence helpfulness, yielding rankings driven by intrinsic evidence informativeness rather than answer-generation uncertainty.

\subsection{Surrogate-Accelerated Selection Pipeline}
\label{sec:pipeline}

While the proposed scoring objective is theoretically sound, computing it for every candidate in a large pool $\mathcal{C}$ using the heavy main model $\mathcal{M}_{\text{main}}$ (e.g., 7B+ parameters) incurs significant latency. To address this, we implement the selection process via a Surrogate-Accelerated Pipeline.

\paragraph{Utility Transferability Hypothesis.}
We posit that models of varying sizes exhibit a \textbf{discriminative consensus}: evidence that is clearly unhelpful or contradictory to a lightweight surrogate model $\mathcal{M}_{\text{surrogate}}$ is highly likely to be unhelpful to $\mathcal{M}_{\text{main}}$. As a result, surrogate-induced rankings can closely approximate those of the main model, enabling the expensive $O(N)$ scanning over $N$ candidate evidences to be offloaded to the lightweight surrogate.

\paragraph{Inference Workflow.}
The overall inference pipeline follows a retrieve–select paradigm:
\begin{enumerate}[itemsep=0pt, topsep=0pt]
    \item \textbf{Scoring:} Given the retrieved candidate set $\mathcal{C}$, the lightweight $\mathcal{M}_{\text{surrogate}}$ computes the helpfulness score $\hat{s}(c_i)$ for each $c_i \in \mathcal{C}$ in parallel.
    \item \textbf{Selection:} Rank all $c_i \in \mathcal{C}$ based on $\hat{s}(c_i)$ (Eq.~\ref{eq:score_calculation}) and select the top-$k$ subset $S^*$.
    \item \textbf{Generation:} The main $\mathcal{M}_{\text{main}}$ generates the final response using only the selected subset $S^*$, requiring a single inference pass.
\end{enumerate}

\paragraph{Complexity Analysis.}
A standard reranking strategy incurs $(N+1)\times\Phi(\mathcal{M}_{\text{main}})$ inference cost. Our pipeline reduces this to $\mathcal{C}_{\text{ours}} = N\times\Phi(\mathcal{M}_{\text{surrogate}}) + \Phi(\mathcal{M}_{\text{main}})$,
yielding substantial computation cost reduction since $\Phi(\mathcal{M}_{\text{surrogate}}) < \Phi(\mathcal{M}_{\text{main}})$ and the $\mathcal{M}_{\text{main}}$ is invoked only once.

\begin{table*}[!htb]
\begin{lrbox}{\myentiretablebox}
\begin{tabular}{@{}l|ll|llllllllll@{}}
\toprule
 &  &  & \multicolumn{5}{c|}{\textbf{MRAG-Bench~\cite{hu2024mragbench}}} & \multicolumn{5}{c}{\textbf{Visual-RAG~\cite{wu2025visualrag}}} \\ \cmidrule(l){4-13} 
 &  &  & \multicolumn{10}{c}{Main Model} \\ \cmidrule(l){4-13} 
 & \multirow{-3}{*}{Image Selection Method} & \multirow{-3}{*}{\# Params} & Qwen3-VL-8B & MiniCPM-V4.5 & Gemma3-12B & Ovis2.5-9B & \multicolumn{1}{l|}{InternVL3.5-8B} & Qwen3-VL-8B & MiniCPM-V4.5 & Gemma3-12B & Ovis2.5-9B & InternVL3.5-8B \\ \cmidrule(l){2-13} 
\multirow{-4}{*}{Top K} & Zero-Shot &  & 59.35 & 57.95 & 56.84 & 59.05 & \multicolumn{1}{l|}{42.87} & 52.41 & 53.07 & 51.07 & 52.67 & 54.28 \\ \midrule
 & CLIP-B~\cite{radford2021learning_clip} & 151M & 57.80(-1.55) & 58.83(+0.88) & 55.95(-0.89) & 57.13(-1.92) & \multicolumn{1}{l|}{42.65(-0.22)} & 54.28(+1.87) & 58.56(+5.49) & 51.87(+0.80) & 66.04(+13.37) & \textbf{64.17(+9.89)} \\
 & CLIP-L~\cite{radford2021learning_clip} & 428M & 57.80(-1.55) & 58.98(+1.03) & 56.25(-0.59) & 57.95(-1.10) & \multicolumn{1}{l|}{44.49(+1.62)} & 54.14(+1.73) & 58.29(+5.22) & 54.95(+3.88) & 62.03(+9.36) & 61.23(+6.95) \\
 & OpenCLIP~\cite{cherti2023reproducible_openclip} & 151M & 58.46(-0.89) & 60.46(+2.51) & 57.50(+0.66) & 58.24(-0.81) & \multicolumn{1}{l|}{42.65(-0.22)} & 52.94(+0.53) & 58.56(+5.49) & 52.67(+1.60) & 64.97(+12.30) & 60.29(+6.01) \\
 & SigLIP 2 Base~\cite{tschannen2025siglip} & 375M & 59.42(+0.07) & 61.27(+3.32) & 58.24(+1.40) & 58.91(-0.14) & \multicolumn{1}{l|}{44.35(+1.48)} & 52.94(+0.53) & 56.68(+3.61) & 55.48(+4.41) & 64.04(+11.37) & 61.76(+7.48) \\
 & SigLIP 2 So400m~\cite{tschannen2025siglip} & 1.1B & 61.42(+2.07) & 61.57(+3.62) & 59.87(+3.03) & 59.35(+0.30) & \multicolumn{1}{l|}{44.49(+1.62)} & 53.34(+0.93) & 59.63(+6.56) & 53.34(+2.27) & 64.84(+12.17) & 62.57(+8.29) \\
 & BGE-VL-large~\cite{zhou2025megapairs} & 428M & 58.98(-0.37) & 60.53(+2.58) & 55.65(-1.19) & 59.05(+0.00) & \multicolumn{1}{l|}{44.64(+1.77)} & 52.81(+0.40) & 57.35(+4.28) & 54.28(+3.21) & 65.64(+12.97) & 58.42(+4.14) \\ \cmidrule(l){2-13} 
 & E5-V~\cite{jiang2024e5} & 7.8B & 58.83(-0.52) & 59.20(+1.25) & 55.06(-1.78) & 57.95(-1.10) & \multicolumn{1}{l|}{43.46(+0.59)} & 53.88(+1.47) & 58.42(+5.35) & 55.48(+4.41) & 67.91(+15.24) & 62.57(+8.29) \\
 & GME~\cite{zhang2024gme} & 2.2B & 64.38(+5.03) & 65.19(+7.24) & 59.42(+2.58) & 61.35(+2.30) & \multicolumn{1}{l|}{47.01(+4.14)} & 55.88(+3.47) & 59.09(+6.02) & 53.07(+2.00) & 67.51(+14.84) & 62.30(+8.02) \\
& VLM2Vec~\cite{jiang2024vlm2vec} & 7.7B & 61.71(+2.36) & 61.79(+3.84) & 58.76(+1.92) & 60.83(+1.78) & \multicolumn{1}{l|}{46.12(+3.25)} & 54.55(+2.14) & 57.49(+4.42) & 54.81(+3.74) & 49.33(-3.34) & 48.53(-5.75) \\
 & VLM2Vec-V2.0~\cite{meng2025vlm2vecV2} & 2.2B & 60.68(+1.33) & 61.86(+3.91) & 57.50(+0.66) & 60.24(+1.19) & \multicolumn{1}{l|}{45.31(+2.44)} & 55.61(+3.20) & 55.88(+2.81) & 53.88(+2.81) & 50.80(-1.87) & 48.53(-5.75) \\
 & BGE-MLLM-S1~\cite{zhou2025megapairs} & 7.6B & 60.46(+1.11) & 60.90(+2.95) & 55.28(-1.56) & 58.61(-0.44) & \multicolumn{1}{l|}{45.68(+2.81)} & 56.15(+3.74) & 56.55(+3.48) & 56.55(+5.48) & 67.11(+14.44) & 63.37(+9.09) \\
 & BGE-MLLM-S2~\cite{zhou2025megapairs} & 7.6B & 60.31(+0.96) & 61.57(+3.62) & 57.58(+0.74) & 58.54(-0.51) & \multicolumn{1}{l|}{44.72(+1.85)} & 53.61(+1.20) & 55.35(+2.28) & 55.35(+4.28) & 65.11(+12.44) & 63.90(+9.62) \\
 & UniME-V2~\cite{gu2025unime} & 7.1B & 62.82(+3.47) & 63.05(+5.10) & \textbf{60.90(+4.06)} & 60.61(+1.56) & \multicolumn{1}{l|}{45.60(+2.73)} & 56.95(+4.54) & 58.69(+5.62) & 55.88(+4.81) & 57.62(+4.95) & 49.33(-4.95) \\
 & LamRA-Rank~\cite{liu2025lamra} & 8B & 63.34(+3.99) & 62.97(+5.02) & 58.61(+1.77) & 59.05(+0.00) & \multicolumn{1}{l|}{45.75(+2.88)} & 58.42(+6.01) & 60.16(+7.09) & 55.48(+4.41) & 58.16(+5.49) & 62.57(+8.29) \\ \cmidrule(l){2-13} 
 & Ours (Qwen3-VL-2B Surrogate) & 2.1B & \textbf{65.56(+6.21)} & \textbf{65.41(+7.46)} & 60.83(+3.99) & \textbf{61.42(+2.37)} & \multicolumn{1}{l|}{\textbf{47.89(+5.02)}} & 59.89(+7.48) & 60.16(+7.09) & 59.22(+8.15) & 68.85(+16.18) & \textbf{64.17(+9.89)} \\
 & Ours (Ovis2.5-2B Surrogate) & 2.6B & 64.97(+5.62) & 64.08(+6.13) & 60.53(+3.69) & 60.16(+1.11) & \multicolumn{1}{l|}{47.23(+4.36)} & \textbf{61.36(+8.95)} & \textbf{61.23(+8.16)} & 57.75(+6.68) & \textbf{69.12(+16.45)} & 62.30(+8.02) \\
 & Ours (In-Family Surrogate) &  & \textbf{65.56(+6.21)} & 64.60(+6.65) & 59.72(+2.88) & 60.16(+1.11) & \multicolumn{1}{l|}{47.08(+4.21)} & 59.89(+7.48) & 60.70(+7.63) & \textbf{59.63(+8.56)} & \textbf{69.12(+16.45)} & 62.17(+7.89) \\ \cmidrule(l){2-13} 
\multirow{-18}{*}{K=1} & \textbf{GT} (Human-Annotated Image as Input) & \textbf{} & 64.82(+5.47) & 64.15(+6.20) & 57.65(+0.81) & 62.97(+3.92) & \multicolumn{1}{l|}{46.49(+3.62)} & 60.96(+8.55) & 63.50(+10.43) & 58.69(+7.62) & 70.86(+18.19) & 64.04(+9.76) \\ \midrule
 & CLIP-B~\cite{radford2021learning_clip} & 151M & 61.42(+2.07) & 60.83(+2.88) & 57.80(+0.96) & 56.84(-2.21) & \multicolumn{1}{l|}{43.61(+0.74)} & 58.69(+6.28) & 61.63(+8.56) & 57.35(+6.28) & 58.42(+5.75) & 51.07(-3.21) \\
 & CLIP-L~\cite{radford2021learning_clip} & 428M & 59.94(+0.59) & 61.86(+3.91) & 57.72(+0.88) & 57.58(-1.47) & \multicolumn{1}{l|}{44.27(+1.40)} & 59.49(+7.08) & 62.57(+9.50) & 56.82(+5.75) & 58.56(+5.89) & 51.74(-2.54) \\
 & OpenCLIP~\cite{cherti2023reproducible_openclip} & 151M & 62.23(+2.88) & 62.97(+5.02) & 58.68(+1.84) & 59.65(+0.60) & \multicolumn{1}{l|}{43.83(+0.96)} & 60.83(+8.42) & 61.36(+8.29) & 58.02(+6.95) & 57.75(+5.08) & 53.61(-0.67) \\
 & SigLIP 2 Base~\cite{tschannen2025siglip} & 375M & 62.90(+3.55) & 62.23(+4.28) & 60.98(+4.14) & 58.39(-0.66) & \multicolumn{1}{l|}{44.86(+1.99)} & 59.36(+6.95) & 58.69(+5.62) & 57.89(+6.82) & 58.42(+5.75) & 54.81(+0.53) \\
 & SigLIP 2 So400m~\cite{tschannen2025siglip} & 1.1B & 63.27(+3.92) & 63.56(+5.61) & 60.83(+3.99) & 59.87(+0.82) & \multicolumn{1}{l|}{45.53(+2.66)} & 56.68(+4.27) & \textbf{63.64(+10.57)} & 57.62(+6.55) & 58.29(+5.62) & 51.34(-2.94) \\
 & BGE-VL-large~\cite{zhou2025megapairs} & 428M & 62.75(+3.40) & 61.79(+3.84) & 58.98(+2.14) & 59.65(+0.60) & \multicolumn{1}{l|}{44.72(+1.85)} & 57.75(+5.34) & 59.63(+6.56) & 55.35(+4.28) & 56.82(+4.15) & 51.60(-2.68) \\ \cmidrule(l){2-13} 
 & E5-V~\cite{jiang2024e5} & 7.8B & 62.53(+3.18) & 62.53(+4.58) & 59.13(+2.29) & 58.98(-0.07) & \multicolumn{1}{l|}{43.16(+0.29)} & 58.96(+6.55) & 59.76(+6.69) & 58.29(+7.22) & 60.29(+7.62) & 52.14(-2.14) \\
 & GME~\cite{zhang2024gme} & 2.2B & 65.41(+6.06) & 65.48(+7.53) & 61.79(+4.95) & 61.05(+2.00) & \multicolumn{1}{l|}{46.93(+4.06)} & 61.50(+9.09) & 60.29(+7.22) & 56.02(+4.95) & 53.88(+1.21) & 50.53(-3.75) \\
 & VLM2Vec~\cite{jiang2024vlm2vec} & 7.7B & 64.08(+4.73) & 63.71(+5.76) & 60.90(+4.06) & 59.87(+0.82) & \multicolumn{1}{l|}{45.53(+2.66)} & 60.03(+7.62) & 60.29(+7.22) & 57.35(+6.28) & 56.42(+3.75) & 53.07(-1.21) \\
 & VLM2Vec-V2.0~\cite{meng2025vlm2vecV2} & 2.2B & 64.82(+5.47) & 64.75(+6.80) & 61.12(+4.28) & 61.05(+2.00) & \multicolumn{1}{l|}{45.08(+2.21)} & 58.69(+6.28) & 60.83(+7.76) & 56.68(+5.61) & 56.95(+4.28) & 52.41(-1.87) \\
 & BGE-MLLM-S1~\cite{zhou2025megapairs} & 7.6B & 62.82(+3.47) & 63.78(+5.83) & 59.65(+2.81) & 60.09(+1.04) & \multicolumn{1}{l|}{45.90(+3.03)} & 59.89(+7.48) & 58.69(+5.62) & 58.69(+7.62) & 58.29(+5.62) & 52.01(-2.27) \\
 & BGE-MLLM-S2~\cite{zhou2025megapairs} & 7.6B & 64.52(+5.17) & 64.52(+6.57) & 60.68(+3.84) & 60.61(+1.56) & \multicolumn{1}{l|}{45.01(+2.14)} & 59.09(+6.68) & 56.02(+2.95) & 56.02(+4.95) & 57.75(+5.08) & 50.27(-4.01) \\
 & UniME-V2~\cite{gu2025unime} & 7.1B & 65.19(+5.84) & 64.75(+6.80) & 62.08(+5.24) & 61.57(+2.52) & \multicolumn{1}{l|}{45.97(+3.10)} & 59.22(+6.81) & 61.10(+8.03) & 58.02(+6.95) & 58.29(+5.62) & 53.34(-0.94) \\
 & LamRA-Rank~\cite{liu2025lamra} & 8B & 65.85(+6.50) & 64.89(+6.94) & 60.53(+3.69) & 61.27(+2.22) & \multicolumn{1}{l|}{46.27(+3.40)} & 61.23(+8.82) & 63.10(+10.03) & 58.29(+7.22) & 56.95(+4.28) & 53.88(-0.40) \\
 \cmidrule(l){2-13} 
 & Ours (Qwen3-VL-2B Surrogate) & 2.1B & \textbf{67.55(+8.20)} & \textbf{65.85(+7.90)} & \textbf{64.52(+7.68)} & \textbf{62.31(+3.26)} & \multicolumn{1}{l|}{47.08(+4.21)} & 63.77(+11.36) & 59.89(+6.82) & \textbf{60.43(+9.36)} & \textbf{63.24(+10.57)} & \textbf{55.75(+1.47)} \\
 & Ours (Ovis2.5-2B Surrogate) & 2.6B & 66.08(+6.73) & \textbf{65.85(+7.90)} & 62.23(+5.39) & 62.01(+2.96) & \multicolumn{1}{l|}{\textbf{47.15(+4.28)}} & \textbf{64.44(+12.03)} & 60.43(+7.36) & 57.75(+6.68) & 60.70(+8.03) & 53.88(-0.40) \\
 & Ours (In-Family Surrogate) &  & \textbf{67.55(+8.20)} & 65.63(+7.68) & 62.82(+5.98) & 62.01(+2.96) & \multicolumn{1}{l|}{\textbf{47.15(+4.28)}} & 63.77(+11.36) & 61.50(+8.43) & 57.89(+6.82) & 60.70(+8.03) & 55.08(+0.80) \\
 \cmidrule(l){2-13} 
\multirow{-18}{*}{K=3} & \textbf{GT} (Human-Annotated Image as Input) & \textbf{} & 69.03(+9.68) & 66.89(+8.94) & 58.39(+1.55) & 64.01(+4.96) & \multicolumn{1}{l|}{48.78(+5.91)} & 64.71(+12.30) & 62.17(+9.10) & 59.63(+8.56) & 64.71(+12.04) & 57.89(+3.61) \\ \midrule
 & CLIP-B~\cite{radford2021learning_clip} & 151M & 61.79(+2.44) & 62.08(+4.13) & 59.79(+2.95) & 58.91(-0.14) & \multicolumn{1}{l|}{43.75(+0.88)} & 64.57(+12.16) & 62.03(+8.96) & 59.76(+8.69) & 60.03(+7.36) & 55.48(+1.20) \\
 & CLIP-L~\cite{radford2021learning_clip} & 428M & 60.61(+1.26) & 60.98(+3.03) & 58.24(+1.40) & 57.95(-1.10) & \multicolumn{1}{l|}{42.65(-0.22)} & 61.76(+9.35) & 62.03(+8.96) & 59.09(+8.02) & 58.42(+5.75) & 52.81(-1.47) \\
 & OpenCLIP~\cite{cherti2023reproducible_openclip} & 151M & 63.56(+4.21) & 62.08(+4.13) & 61.27(+4.43) & 59.79(+0.74) & \multicolumn{1}{l|}{45.60(+2.73)} & 62.17(+9.76) & 62.57(+9.50) & 59.22(+8.15) & 59.49(+6.82) & 55.35(+1.07) \\
 & SigLIP 2 Base~\cite{tschannen2025siglip} & 375M & 64.30(+4.95) & 63.05(+5.10) & 61.94(+5.10) & 60.61(+1.56) & \multicolumn{1}{l|}{45.75(+2.88)} & 62.57(+10.16) & 60.96(+7.89) & 59.89(+8.82) & 59.49(+6.82) & 54.95(+0.67) \\
 & SigLIP 2 So400m~\cite{tschannen2025siglip} & 1.1B & 65.19(+5.84) & 63.34(+5.39) & 62.60(+5.76) & 60.09(+1.04) & \multicolumn{1}{l|}{45.08(+2.21)} & 60.56(+8.15) & 61.36(+8.29) & 54.95(+3.88) & 59.22(+6.55) & 55.21(+0.93) \\
 & BGE-VL-large~\cite{zhou2025megapairs} & 428M & 64.67(+5.32) & 63.41(+5.46) & 61.27(+4.43) & 59.94(+0.89) & \multicolumn{1}{l|}{45.31(+2.44)} & 58.16(+5.75) & 61.23(+8.16) & 56.68(+5.61) & 56.82(+4.15) & 51.74(-2.54) \\ \cmidrule(l){2-13} 
 & E5-V~\cite{jiang2024e5} & 7.8B & 63.93(+4.58) & 62.38(+4.43) & 60.38(+3.54) & 59.94(+0.89) & \multicolumn{1}{l|}{43.61(+0.74)} & 60.16(+7.75) & 61.63(+8.56) & 60.96(+9.89) & 57.89(+5.22) & 55.35(+1.07) \\
 & GME~\cite{zhang2024gme} & 2.2B & 67.04(+7.69) & 66.30(+8.35) & 61.79(+4.95) & 61.94(+2.89) & \multicolumn{1}{l|}{46.78(+3.91)} & 65.78(+13.37) & 62.97(+9.90) & 59.36(+8.29) & 57.22(+4.55) & 54.95(+0.67) \\
 & VLM2Vec~\cite{jiang2024vlm2vec} & 7.7B & 66.00(+6.65) & 64.45(+6.50) & 61.71(+4.87) & 61.27(+2.22) & \multicolumn{1}{l|}{46.56(+3.69)} & 62.43(+10.02) & 61.50(+8.43) & 58.56(+7.49) & 60.16(+7.49) & 54.01(-0.27) \\
 & VLM2Vec-V2.0~\cite{meng2025vlm2vecV2} & 2.2B & 66.22(+6.87) & 65.11(+7.16) & 62.45(+5.61) & 61.86(+2.81) & \multicolumn{1}{l|}{45.90(+3.03)} & 61.63(+9.22) & 59.09(+6.02) & 58.69(+7.62) & 55.21(+2.54) & 51.20(-3.08) \\
 & BGE-MLLM-S1~\cite{zhou2025megapairs} & 7.6B & 64.89(+5.54) & 63.49(+5.54) & 61.20(+4.36) & 60.31(+1.26) & \multicolumn{1}{l|}{44.64(+1.77)} & 62.97(+10.56) & 57.75(+4.68) & 57.75(+6.68) & 59.22(+6.55) & 52.41(-1.87) \\
 & BGE-MLLM-S2~\cite{zhou2025megapairs} & 7.6B & 65.11(+5.76) & 65.04(+7.09) & 61.49(+4.65) & 61.20(+2.15) & \multicolumn{1}{l|}{45.60(+2.73)} & 59.36(+6.95) & 57.35(+4.28) & 57.35(+6.28) & 59.89(+7.22) & 54.14(-0.14) \\
 & UniME-V2~\cite{gu2025unime} & 7.1B & 66.30(+6.95) & 65.34(+7.39) & 62.60(+5.76) & 62.82(+3.77) & \multicolumn{1}{l|}{45.31(+2.44)} & 61.76(+9.35) & 63.77(+10.70) & 60.43(+9.36) & 58.02(+5.35) & 56.15(+1.87) \\
 & LamRA-Rank~\cite{liu2025lamra} & 8B & 66.44(+7.09) & 66.67(+8.72) & 61.71(+4.87) & 63.05(+4.00) & \multicolumn{1}{l|}{46.05(+3.18)} & 64.04(+11.63) & \textbf{64.97(+11.90)} & 59.89(+8.82) & 58.82(+6.15) & 54.95(+0.67) \\
 \cmidrule(l){2-13} 
 & Ours (Qwen3-VL-2B Surrogate) & 2.1B & \textbf{67.55(+8.20)} & 66.37(+8.42) & \textbf{63.93(+7.09)} & \textbf{63.41(+4.36)} & \multicolumn{1}{l|}{\textbf{46.86(+3.99)}} & 64.84(+12.43) & 63.10(+10.03) & 60.83(+9.76) & \textbf{62.17(+9.50)} & 55.21(+0.93) \\
 & Ours (Ovis2.5-2B Surrogate) & 2.6B & 65.71(+6.36) & 64.67(+6.72) & 61.71(+4.87) & 62.38(+3.33) & \multicolumn{1}{l|}{46.71(+3.84)} & \textbf{66.18(+13.77)} & 63.50(+10.43) & \textbf{61.50(+10.43)} & 60.70(+8.03) & \textbf{56.28(+2.00)} \\
 & Ours (In-Family Surrogate) &  & \textbf{67.55(+8.20)} & \textbf{66.96(+9.01)} & 62.31(+5.47) & 62.38(+3.33) & \multicolumn{1}{l|}{46.49(+3.62)} & 64.84(+12.43) & 63.50(+10.43) & 57.49(+6.42) & 60.70(+8.03) & 55.61(+1.33) \\ \cmidrule(l){2-13} 
\multirow{-18}{*}{K=5} & \textbf{GT} (Human-Annotated Image as Input) & \textbf{} & 69.99(+10.64) & 68.00(+10.05) & 59.42(+2.58) & 64.89(+5.84) & \multicolumn{1}{l|}{47.60(+4.73)} & 66.84(+14.43) & 64.17(+11.10) & 61.10(+10.03) & 64.71(+12.04) & 58.82(+4.54) \\ \bottomrule
\end{tabular}
\end{lrbox}
\resizebox{\linewidth}{!}{\usebox{\myentiretablebox}}
\vspace{-2mm}
\caption{Main results on MRAG-Bench and Visual-RAG under different visual evidence selection methods and Top-$K$ settings. Numbers in parentheses denote absolute performance differences relative to the zero-shot baseline for the same model. \textbf{Bold} values indicate the best-performing retrieval-based method in each setting (excluding the ground-truth (GT) oracle). Additional results are provided in Appendix~\ref{appendix_sec:addtional_results}.}
\vspace{-6mm}
\label{tab:exp-main-results}
\end{table*}

\section{Experiments}
\label{sec:experiments}
\vspace{-0.5mm}

This section presents the experimental setup, main results, and further analyses
(\cref{subsec:experimental_setup,subsec:main_results,subsec:analysis_utility_estimation,subsec:surrogate-vs-main,subsec:ablation,subsec:additional_analyses}).

\vspace{-0.5mm}
\subsection{Experimental Setup}
\label{subsec:experimental_setup}
\paragraph{Benchmarks.}
We evaluate our method on two MM-RAG benchmarks: MRAG-Bench~\cite{hu2024mragbench} (1,353 image–question pairs with ground-truth images) and Visual-RAG~\cite{wu2025visualrag} (374 text-only questions with relevant images).
For each query, we build a fixed candidate pool and evaluate all selection methods on the same pool for fair comparison. Concretely, the pool is formed by combining benchmark-provided ground-truth (relevant) images when available with additional CLIP-style retrieved images from the benchmark corpus; duplicate images are removed. We then rank this fixed pool and report Top-$K$ performance for $K\in\{1,\dots,5\}$. For MRAG-Bench, the pool includes annotated ground-truth evidence plus retrieved distractors; For VisualRAG, we first include the benchmark-provided ground-truth images (up to five per query), and then fill the remaining slots with retriever-returned evidence so that each query is associated with a fixed candidate pool of 10 images. Full construction details, including pool size and dataset-specific preprocessing, are provided in Appendix~\ref{app:subsec:benchmarks}.

\paragraph{Baselines.}
We compare against several categories of baselines: (i) A zero-shot setting without retrieved images and a ground-truth (GT) oracle setting; (ii) CLIP-style retrievers, including CLIP~\cite{radford2021learning_clip}, OpenCLIP~\cite{cherti2023reproducible_openclip}, SigLIP~2~\cite{tschannen2025siglip}, and BGE-VL-Large~\cite{zhou2025megapairs}; (iii) MLLM-based retrievers and rerankers that use MLLMs for relevance estimation, including E5-V~\cite{jiang2024e5}, GME~\cite{zhang2024gme}, VLM2Vec~\cite{jiang2024vlm2vec,meng2025vlm2vecV2}, BGE-MLLM~\cite{zhou2025megapairs}, LamRA-Rank~\cite{hu2024mragbench}, and UniME-V2~\cite{gu2025unime}; and (iv) Answer-level objective methods (Problem~\ref{eq:utility_kl_max}), including average token probability and MC sampling.

\paragraph{Backbone Models.}
Our experiments cover five families of state-of-the-art open-source MLLMs (Qwen3-VL~\cite{bai2025qwen3vltechnicalreport}, MiniCPM4.5~\cite{yu2025minicpm}, Gemma3~\cite{team2025gemma3}, OVIS2.5~\cite{lu2025ovis2}, and InternVL3.5~\cite{wang2025internvl3}), spanning multiple architectures and parameter scales. 

\paragraph{Metrics.}
For MRAG-Bench, we report exact-match accuracy; for Visual-RAG, we follow the original protocol and evaluate answers using an LLM-as-Judge metric. We additionally measure computational cost in terms of FLOPs and latency.

Full details of benchmarks, baselines, model variants, prompt templates, and evaluation protocols are provided in Appendix~\ref{sec:appendix_setup}.

\begin{table}[!htbp]
\begin{lrbox}{\myentiretablebox}
\begin{tabular}{@{}ll|lll|lll@{}}
\toprule
\multirow{2}{*}{Dataset} & \multirow{2}{*}{Model} & \multicolumn{3}{c|}{Ours} & \multicolumn{3}{c}{Answer-level Estimation} \\ \cmidrule(l){3-8} 
 &  & Top-1 & Top-3 & Top-5 & Top-1 & Top-3 & Top-5 \\ \midrule
\multirow{5}{*}{MRAG-Bench} & Qwen3-VL-8B & 65.71 & 67.41 & 67.11 & 63.27(-2.44) & 64.67(-2.74) & 65.04(-2.07) \\
 & MiniCPM-V4.5 & 65.11 & 66.89 & 65.71 & 62.90(-2.21) & 64.15(-2.74) & 64.15(-1.56) \\
 & Gemma3-12B & 59.87 & 61.94 & 63.12 & 59.42(-0.45) & 61.57(-0.37) & 61.64(-1.48) \\
 & Ovis2.5-9B & 61.64 & 62.53 & 63.19 & 60.38(-1.26) & 61.20(-1.33) & 61.27(-1.92) \\
 & InternVL3.5-8B & 48.41 & 47.75 & 46.56 & 46.42(-1.99) & 47.08(-0.67) & 45.53(-1.03) \\ \midrule
\multirow{10}{*}{Visual-RAG} & Qwen3-VL-8B & 62.43 & 63.37 & 64.57 & 57.49(-4.94) & 59.76(-3.61) & 62.97(-1.60) \\
 & MiniCPM-V4.5 & 59.63 & 61.23 & 62.03 & 61.50(+1.87) & 61.10(-0.13) & 62.57(+0.54) \\
 & Gemma3-12B & 56.55 & 60.43 & 60.03 & 58.29(+1.74) & 55.61(-4.82) & 57.09(-2.94) \\
 & Ovis2.5-9B & 70.05 & 60.96 & 61.23 & 54.81(-15.24) & 54.41(-6.55) & 56.82(-4.41) \\
 & InternVL3.5-8B & 61.23 & 58.16 & 57.22 & 52.94(-8.29) & 54.68(-3.48) & 55.88(-1.34) \\ \cmidrule(l){2-8} 
 & Qwen3-VL-8B & 62.43 & 63.37 & 64.57 & 57.89(-4.54) & 60.29(-3.08) & 63.90(-0.67) \\
 & MiniCPM-V4.5 & 59.63 & 61.23 & 62.03 & 61.10(+1.47) & 63.24(+2.01) & 64.71(+2.68) \\
 & Gemma3-12B & 56.55 & 60.43 & 60.03 & 55.88(-0.67) & 57.75(-2.68) & 59.76(-0.27) \\
 & Ovis2.5-9B & 70.05 & 60.96 & 61.23 & 54.95(-15.10) & 55.08(-5.88) & 57.62(-3.61) \\
 & InternVL3.5-8B & 61.23 & 58.16 & 57.22 & 53.07(-8.16) & 57.22(-0.94) & 57.49(+0.27) \\ \bottomrule
\end{tabular}
\end{lrbox}
\resizebox{\linewidth}{!}{\usebox{\myentiretablebox}}
\vspace{-2mm}
\caption{Downstream task performance of our latent-helpfulness-based selection method versus answer-level uncertainty-based selection methods. The table reports performance on MRAG-Bench and Visual-RAG, rather than uncertainty values. For the compared baselines, candidate images are ranked by choice softmax entropy on MRAG-Bench, and by average token probability or MC sampling with NLI-based consistency on Visual-RAG. Parentheses denote performance differences relative to our method. Additional results for different $K$ values are provided in Appendix~\ref{appendix_sec:addtional_results}.}

\vspace{-5mm}
\label{tab:Y_vs_A}
\end{table}

\begin{table*}[!htbp]
\begin{lrbox}{\myentiretablebox}
\begin{tabular}{@{}l|l|lllll|lllll@{}}
\toprule
\multirow{2}{*}{Top-K} & \multirow{2}{*}{Image Selection Method} & \multicolumn{5}{c|}{\textbf{MRAG-Bench}} & \multicolumn{5}{c}{\textbf{Visual-RAG}} \\ \cmidrule(l){3-12} 
 &  & Qwen3-VL-8B & MiniCPM-V4.5 & Gemma3-12B & Ovis2.5-9B & InternVL3.5-8B & Qwen3-VL-8B & MiniCPM-V4.5 & Gemma3-12B & Ovis2.5-9B & InternVL3.5-8B \\ \midrule
\multirow{6}{*}{K=1} & Main Model & 65.71 & 65.11 & 59.87 & 61.64 & 48.41 & 62.43 & 59.63 & 56.55 & 70.05 & 61.23 \\ \cmidrule(l){2-12} 
 & Qwen3-VL-2B & 65.56(-0.15) & 65.41(+0.30) & 60.83(+0.96) & 61.42(-0.22) & 47.89(-0.52) & 59.89(-2.54) & 60.16(+0.53) & 59.22(+2.67) & 68.85(-1.20) & 64.17(+2.94) \\
 & Ovis2.5-2B & 64.97(-0.74) & 64.08(-1.03) & 60.53(+0.66) & 60.16(-1.48) & 47.23(-1.18) & 61.36(-1.07) & 61.23(+1.60) & 57.75(+1.20) & 69.12(-0.93) & 62.30(+1.07) \\
 & MiniCPM-V4.5-AWQ & 64.23(-1.48) & 64.60(-0.51) & 60.38(+0.51) & 61.71(+0.07) & 48.41(+0.00) & 60.16(-2.27) & 60.70(+1.07) & 58.29(+1.74) & 67.25(-2.80) & 62.43(+1.20) \\
 & Gemma3-4B & 64.15(-1.56) & 62.75(-2.36) & 59.72(-0.15) & 61.71(+0.07) & 45.23(-3.18) & 59.89(-2.54) & 61.50(+1.87) & 59.63(+3.08) & 68.05(-2.00) & 61.90(+0.67) \\
 & InternVL3.5-2B & 63.64(-2.07) & 62.60(-2.51) & 59.87(+0.00) & 60.24(-1.40) & 47.08(-1.33) & 57.62(-4.81) & 58.56(-1.07) & 55.88(-0.67) & 66.98(-3.07) & 62.17(+0.94) \\ \midrule
\multirow{6}{*}{K=3} & Main Model & 67.41 & 66.89 & 61.94 & 62.53 & 47.75 & 63.37 & 61.23 & 60.43 & 60.96 & 58.16 \\ \cmidrule(l){2-12} 
 & Qwen3-VL-2B & 67.55(+0.14) & 65.85(-1.04) & 64.52(+2.58) & 62.31(-0.22) & 47.08(-0.67) & 63.77(+0.40) & 59.89(-1.34) & 60.43(+0.00) & 63.24(+2.28) & 55.75(-2.41) \\
 & Ovis2.5-2B & 66.08(-1.33) & 65.85(-1.04) & 62.23(+0.29) & 62.01(-0.52) & 47.15(-0.60) & 64.44(+1.07) & 60.43(-0.80) & 57.75(-2.68) & 60.70(-0.26) & 53.88(-4.28) \\
 & MiniCPM-V4.5-AWQ & 66.96(-0.45) & 65.63(-1.26) & 63.41(+1.47) & 62.31(-0.22) & 47.38(-0.37) & 63.64(+0.27) & 61.50(+0.27) & 58.42(-2.01) & 60.03(-0.93) & 54.41(-3.75) \\
 & Gemma3-4B & 65.78(-1.63) & 64.45(-2.44) & 62.82(+0.88) & 61.57(-0.96) & 46.78(-0.97) & 62.30(-1.07) & 63.37(+2.14) & 57.89(-2.54) & 58.42(-2.54) & 52.14(-6.02) \\
 & InternVL3.5-2B & 65.48(-1.93) & 64.15(-2.74) & 63.12(+1.18) & 61.64(-0.89) & 47.15(-0.60) & 62.17(-1.20) & 63.37(+2.14) & 60.16(-0.27) & 62.30(+1.34) & 55.08(-3.08) \\ \midrule
\multirow{6}{*}{K=5} & Main Model & 67.11 & 65.71 & 63.12 & 63.19 & 46.56 & 64.57 & 62.03 & 60.03 & 61.23 & 57.22 \\ \cmidrule(l){2-12} 
 & Qwen3-VL-2B & 67.55(+0.44) & 66.37(+0.66) & 63.93(+0.81) & 63.41(+0.22) & 46.86(+0.30) & 64.84(+0.27) & 63.10(+1.07) & 60.83(+0.80) & 62.17(+0.94) & 55.21(-2.01) \\
 & Ovis2.5-2B & 65.71(-1.40) & 64.67(-1.04) & 61.71(-1.41) & 62.38(-0.81) & 46.71(+0.15) & 66.18(+1.61) & 63.50(+1.47) & 61.50(+1.47) & 60.70(-0.53) & 56.28(-0.94) \\
 & MiniCPM-V4.5-AWQ & 66.74(-0.37) & 66.96(+1.25) & 64.75(+1.63) & 62.38(-0.81) & 46.42(-0.14) & 66.31(+1.74) & 63.50(+1.47) & 60.29(+0.26) & 60.29(-0.94) & 56.02(-1.20) \\
 & Gemma3-4B & 66.15(-0.96) & 65.34(-0.37) & 62.31(-0.81) & 62.90(-0.29) & 47.52(+0.96) & 63.24(-1.33) & 63.10(+1.07) & 57.49(-2.54) & 59.89(-1.34) & 54.81(-2.41) \\
 & InternVL3.5-2B & 66.81(-0.30) & 64.45(-1.26) & 62.82(-0.30) & 62.60(-0.59) & 46.49(-0.07) & 62.83(-1.74) & 65.24(+3.21) & 60.29(+0.26) & 62.30(+1.07) & 55.61(-1.61) \\ \bottomrule
\end{tabular}
\end{lrbox} 

\centering
\resizebox{0.95\linewidth}{!}{\usebox{\myentiretablebox}}

\caption{Comparison between surrogate models and their corresponding large main models under the same utility estimation framework. Results are reported across Top-$K$ settings; additional results are provided in Appendix~\ref{appendix_sec:addtional_results}.}

\vspace{-3mm}
\label{tab:surrogate_vs_main_model}
\end{table*}

\subsection{Main Results: Comparative Analysis}
\label{subsec:main_results}
We first conduct a comprehensive comparison with state-of-the-art methods, including both CLIP-based and MLLM-based baselines, on two benchmark datasets, as in Table~\ref{tab:exp-main-results}. Overall, our method achieves the best performance in the majority of settings across both benchmarks, consistently outperforming baselines across various Top-$K$ settings. In particular, on Visual-RAG, our method yields substantial improvements of up to 
$+16.18$, outperforming most competing approaches by a clear margin. Notably, while many baselines rely on \textit{supervised retrieval training} and \textit{large-scale} MLLMs (often 7B+ parameters), our approach is \textit{training-free} and achieves comparable or better performance using only \textit{lightweight} 2B-scale surrogate models. \emph{These results demonstrate both the effectiveness and parameter efficiency of our utility-oriented evidence selection method.} Moreover, our method closely approaches the ground-truth (GT) oracle upper bound, and in some cases even slightly exceeds it. This suggests that \emph{our proposed utility-oriented visual evidence selection can, in certain cases, even identify evidence that is more task-effective for a given model’s generation than human annotations.} Lastly, we observe that the zero-shot setting outperforms several RAG baselines, further highlighting \emph{the necessity of utility-aware evidence selection rather than indiscriminate RAG.}

\subsection{Analysis on Utility Estimation Metrics}
\label{subsec:analysis_utility_estimation}

Next, in Table~\ref{tab:Y_vs_A}, we demonstrate the effectiveness of introducing the latent helpfulness variable $Z$ (Problem~\ref{eq:latent_kl_max}) compared to directly optimizing the answer-space objective $Y$ (Problem~\ref{eq:utility_kl_max}). For open-ended Visual-RAG questions, we include two $Y$-based baselines (softmax and MC sampling) for comprehensive comparison. As shown, our method consistently outperforms the $Y$ objective in the majority of evaluation settings. Notably, the performance gap is more advanced on Visual-RAG than on MRAG-Bench, likely due to its open-ended nature, where generation noise more strongly affects the reliability of answer-level uncertainty. Overall, these results support our central claim that \emph{directly modeling answer-space objectives conflates evidence utility with generation noise, whereas decoupling utility estimation from answer generation yields a more stable and reliable criterion for visual evidence selection}.

\subsection{Surrogate Transferability and Efficiency}
\label{subsec:surrogate-vs-main}

In Table~\ref{tab:surrogate_vs_main_model}, we then compare lightweight surrogate models with substantially larger main models under the same latent-variable-based utility estimation framework to evaluate ranking agreement and assess the effectiveness of surrogate acceleration. Overall, lightweight surrogate models achieve performance that is highly comparable to the larger main models. Across all evaluated backbones and Top-$K$ settings, surrogate models exhibit a small average absolute performance gap of $0.38$ in exact-match accuracy on MRAG-Bench and $0.18$ in LLM-as-judge scores on Visual-RAG, while matching or outperforming their corresponding main models in 36\% and 48\% of the evaluated settings, respectively.

\begin{table}[!htb]
\begin{lrbox}{\myentiretablebox}
\begin{tabular}{@{}l|l|l|l|ll|ll@{}}

\toprule
\multirow{2}{*}{Model Family} & \multirow{2}{*}{\#Params} & \multirow{2}{*}{Method} & \multirow{2}{*}{\#Tokens} & \multicolumn{2}{c|}{Prefill} & \multicolumn{2}{c}{Decode} \\
 &  &  &  & GFLOPs & Latency(ms) & GFLOPs & Latency(ms) \\ \midrule
\multirow{4}{*}{Qwen3-VL} & 2.1B & \multirow{2}{*}{\makecell{Discriminative\\Estimation}} & \multirow{2}{*}{289} & 991 & 0.0480 & 3 & 0.0291 \\
 & 8.1B &  &  & 4360 & 0.0638 & 15 & 0.0377 \\ \cmidrule(l){2-8} 
 & 2.1B & \multirow{2}{*}{\makecell{Answer-level\\UQ}} & \multirow{2}{*}{315} & 984 & 0.0471 & 101 & 0.8422 \\
 & 8.1B &  &  & 4328 & 0.0640 & 443 & 1.1067 \\ \midrule
\multirow{4}{*}{Ovis2.5} & 2.6B & \multirow{2}{*}{\makecell{Discriminative\\Estimation}} & \multirow{2}{*}{326} & 1176 & 0.0727 & 3 & 0.0252 \\
 & 9.2B &  &  & 5034 & 0.0871 & 15 & 0.0344 \\ \cmidrule(l){2-8} 
 & 2.6B & \multirow{2}{*}{\makecell{Answer-level\\UQ}} & \multirow{2}{*}{352} & 1169 & 0.0714 & 101 & 0.7286 \\
 & 9.2B &  &  & 5003 & 0.0875 & 443 & 1.0021 \\ \midrule
\multirow{4}{*}{InternVL3.5} & 2.3B & \multirow{2}{*}{\makecell{Discriminative\\Estimation}} & \multirow{2}{*}{403} & 1427 & 0.0532 & 4 & 0.0304 \\
 & 8.5B &  &  & 6203 & 0.0789 & 15 & 0.0393 \\ \cmidrule(l){2-8} 
 & 2.3B & \multirow{2}{*}{\makecell{Answer-level\\UQ}} & \multirow{2}{*}{430} & 1421 & 0.0527 & 103 & 0.8529 \\
 & 8.5B &  &  & 6172 & 0.0776 & 449 & 1.1180 \\ \midrule
\multirow{4}{*}{Gemma-3} & 4.3B & \multirow{2}{*}{\makecell{Discriminative\\Estimation}} & \multirow{2}{*}{365} & 2827 & 0.1107 & 8 & 0.0417 \\
 & 12.2B &  &  & 8560 & 0.1354 & 24 & 0.0607 \\ \cmidrule(l){2-8} 
 & 4.3B & \multirow{2}{*}{\makecell{Answer-level\\UQ}} & \multirow{2}{*}{388} & 2788 & 0.1097 & 227 & 1.2162 \\
 & 12.2B &  &  & 8442 & 0.1340 & 688 & 1.7628 \\ \bottomrule
\end{tabular}
\end{lrbox}
\resizebox{\linewidth}{!}{\usebox{\myentiretablebox}}
\caption{Computation cost comparison between our approach and answer-space estimation across multiple model families. We report the number of tokens, FLOPs, and latency for both the prefill and decoding stages.}

\vspace{-3mm}
\label{tab:computation_cost}
\end{table}

\begin{table*}[!htb]
\begin{lrbox}{\myentiretablebox}
\begin{tabular}{@{}l|l|lllll|lllll@{}}
\toprule
\multirow{2}{*}{Top-K} & \multirow{2}{*}{Method} & \multicolumn{5}{c|}{\textbf{MRAG-Bench}} & \multicolumn{5}{c}{\textbf{Visual-RAG}} \\ \cmidrule(l){3-12} 
 &  & Qwen3-VL-8B & MiniCPM-V4.5 & Gemma3-12B & Ovis2.5-9B & InternVL3.5-8B & Qwen3-VL-8B & MiniCPM-V4.5 & Gemma3-12B & Ovis2.5-9B & InternVL3.5-8B \\ \midrule
\multirow{3}{*}{K=1} & Ours & 65.71 & 65.11 & 59.87 & 61.64 & 48.41 & 62.43 & 59.63 & 56.55 & 70.05 & 61.23 \\
 & Verbalized UQ & 62.82(-2.89) & 58.46(-6.65) & 55.65(-4.22) & 57.95(-3.69) & 43.53(-4.88) & 59.36(-3.07) & 58.82(-0.81) & 56.68(+0.13) & 60.03(-10.02) & 48.93(-12.30) \\
 & Listwise Ranking & 58.54(-7.17) & 59.35(-5.76) & 56.91(-2.96) & 58.17(-3.47) & 43.39(-5.02) & 63.24(+0.81) & 62.43(+2.80) & 59.09(+2.54) & 58.56(-11.49) & 54.55(-6.68) \\ \midrule
\multirow{3}{*}{K=3} & Ours & 67.41 & 66.89 & 61.94 & 62.53 & 47.75 & 63.37 & 61.23 & 60.43 & 60.96 & 58.16 \\
 & Verbalized UQ & 64.67(-2.74) & 61.27(-5.62) & 58.91(-3.03) & 58.61(-3.92) & 44.05(-3.70) & 61.10(-2.27) & 61.90(+0.67) & 57.22(-3.21) & 58.82(-2.14) & 52.01(-6.15) \\
 & Listwise Ranking & 63.27(-4.14) & 61.86(-5.03) & 59.65(-2.29) & 58.91(-3.62) & 45.90(-1.85) & 61.90(-1.47) & 62.30(+1.07) & 58.82(-1.61) & 59.49(-1.47) & 52.94(-5.22) \\ \midrule
\multirow{3}{*}{K=5} & Ours & 67.11 & 65.71 & 63.12 & 63.19 & 46.56 & 64.57 & 62.03 & 60.03 & 61.23 & 57.22 \\
 & Verbalized UQ & 65.56(-1.55) & 62.82(-2.89) & 60.46(-2.66) & 60.09(-3.10) & 43.83(-2.73) & 65.64(+1.07) & 61.76(-0.27) & 58.82(-1.21) & 59.09(-2.14) & 55.35(-1.87) \\
 & Listwise Ranking & 65.11(-2.00) & 63.86(-1.85) & 59.72(-3.40) & 60.53(-2.66) & 44.64(-1.92) & 66.71(+2.14) & 61.63(-0.40) & 60.56(+0.53) & 60.70(-0.53) & 52.01(-5.21) \\ \bottomrule
\end{tabular}
\end{lrbox}

\centering
\resizebox{0.95\linewidth}{!}{\usebox{\myentiretablebox}}

\caption{Comparison of utility estimation and evidence ranking strategies under a unified experimental setting, including verbalized uncertainty quantification and listwise ranking over all candidate evidence. Results are reported across different Top-$K$ settings. Additional results for Top-$1$ to Top-$5$ are provided in Appendix~\ref{appendix_sec:addtional_results}.}

\vspace{-5mm}
\label{tab:ablation}
\end{table*}

In Table~\ref{tab:computation_cost}, we further report the computational cost (FLOPs and inference latency) of different utility estimation strategies. Across model families, lightweight surrogate models incur substantially lower cost (e.g., 2B vs. 8B–12B: <25\% prefill FLOPs), confirming the efficiency of surrogate-based utility estimation. Moreover, the $Z$ objective is over $20\times$ more efficient than $Y$-based estimation in decoding FLOPs, with the gap further widening under MC sampling due to linear cost scaling.

Overall, these results show that \emph{our method enables scalable and utility-aware evidence selection while offering substantial computational advantages for practical multimodal RAG pipelines}.

\subsection{Ablation Analysis and Alternative Designs}
\label{subsec:ablation}

Table~\ref{tab:ablation} compares our method with alternative evidence scoring and ranking strategies, including verbalized uncertainty quantification and listwise ranking. Across settings, our method provides a more stable criterion for evidence selection, whereas generation-level alternatives are less consistent.

We also test the robustness of the auxiliary helpfulness probe to prompt realization by considering lexical substitution, paraphrasing, and output-label variation while keeping the model, benchmark, and candidate pool fixed. Performance remains largely stable across these variants, suggesting that the probe is not overly sensitive to moderate changes in wording or output tokens. Full templates and complete results are provided in Appendix~\ref{sec:appendix_setup} and Appendix~\ref{app:subsec:prompt_sensitivity_results}.

\subsection{Additional Analyses}
\label{subsec:additional_analyses}

We provide additional analyses on assumption robustness, surrogate--main disagreement, and cross-scale transferability.

\subsubsection{Robustness under Noisy Candidate Pools}

We evaluate three increasingly perturbed candidate-pool regimes---\emph{Pure Retrieve}, \emph{GT + Hard Negatives}, and \emph{GT + Hard Negatives + Stochastic Perturbation}---while varying pool size from 10 to 20. Performance remains stable across these settings, indicating that the proposed ranking criterion is robust to heterogeneous and noisy candidate pools. Detailed construction protocols and full results are provided in Appendix~\ref{sec:appendix_setup} and Appendix~\ref{app:noisy_candidate_pool}. Appendix~\ref{app:subsec:emp-val-thm1} further reports an empirical study related to Theorem~\ref{thm:equivalent1}.

\subsubsection{Error Analysis of Surrogate--Main Disagreement}

\begin{table}[]
\resizebox{0.95\linewidth}{!}{
\begin{tabular}{l|l|l|l|l|l}
\hline
Benchmark & Main & Surrogate & Candidates & FPs & FP (\%) \\ \hline
MRAG-Bench & Qwen3-VL-8B & Qwen3-VL-2B & 11523 & 259 & 2.25 \\
MRAG-Bench & Ovis2.5-9B & Ovis2.5-2B & 11523 & 227 & 1.97 \\
Visual-RAG & Qwen3-VL-8B & Qwen3-VL-2B & 3740 & 131 & 3.50 \\
Visual-RAG & Ovis2.5-9B & Ovis2.5-2B & 3740 & 83 & 2.22 \\ \hline
\end{tabular}}
\caption{False Positive Analysis of Surrogate Model}
\label{tab:false-positive-analysis}
\end{table}

We further analyze false positives (FPs), defined as candidates ranked in the top 25\% by the surrogate but in the bottom 25\% by the corresponding main model. Results are shown in Table~\ref{tab:false-positive-analysis}. Across two benchmarks and two surrogate--main pairs, the FP ratio remains low (about 2\%--3.5\%), suggesting that severe ranking disagreement is relatively rare. Fine-grained category-level analysis is deferred to Appendix~\ref{app:subsec:finegrained-false-pos-analysis}.

\subsubsection{Cross-Scale Transferability of Surrogate Rankings}

To examine whether surrogate-based helpfulness ranking transfers to substantially larger models, we evaluate Qwen2.5-VL-3B-Instruct and Qwen2.5-VL-7B-Instruct as surrogates for Qwen2.5-VL-72B-Instruct on MRAG-Bench and Visual-RAG. We report \emph{GT hit rate}, i.e., the proportion of top-$K$ selected evidence that belongs to the ground-truth set.

\begin{table}[]
\resizebox{0.95\linewidth}{!}{
\begin{tabular}{l|l|lllll}
\hline
Benchmark & Model & Top-1 & Top-2 & Top-3 & Top-4 & Top-5 \\ \hline
\multirow{3}{*}{MRAG-Bench} & Qwen2.5-VL-3B & 80.49 & 80.11 & 78.74 & 77.64 & 75.56 \\
 & Qwen2.5-VL-7B & 83.93 & 81.70 & 78.99 & 76.29 & 73.23 \\
 & Qwen2.5-VL-72B & 83.96 & 81.85 & 80.96 & 80.09 & 78.38 \\ \hline
\multirow{3}{*}{Visual-RAG} & Qwen2.5-VL-3B & 71.93 & 69.12 & 64.88 & 61.63 & 58.24 \\
 & Qwen2.5-VL-7B & 72.46 & 68.32 & 65.60 & 62.63 & 60.59 \\
 & Qwen2.5-VL-72B & 80.21 & 74.60 & 71.21 & 67.65 & 64.12 \\ \hline
\end{tabular}}
\caption{Empirical validation on 72B-scale models. We report GT hit rates on 100 samples.}
\label{tab:emp-val-72B-scale}
\end{table}

As shown in Table~\ref{tab:emp-val-72B-scale}, the 3B/7B surrogates remain well aligned with the 72B main model across both datasets. On MRAG-Bench, the gap is consistently small across $K$; on Visual-RAG, absolute performance improves with scale, but the relative ranking trends remain similar. These results support the transferability of surrogate-based helpfulness ranking across model scales.
\section{Case Studies}

Finally, we provide two qualitative examples to illustrate the advantages of utility-oriented visual evidence selection. Figure~\ref{fig:case_study}(a) illustrates a car model identification example from MRAG-Bench. Relevance-based methods select a visually similar but incorrect model due to shared color, while our method correctly identifies a helpful image of the correct model by focusing on task-relevant visual cues rather than surface-level similarity. In Figure~\ref{fig:case_study}(b), we present a Visual-RAG example contrasting answer-level objectives with our proposed utility estimation, showing that the former can favor uninformative images with high confidence or consistency, whereas our approach correctly prioritizes evidence that reveals task-critical visual information. 
More detailed task explanations, together with additional analysis of representative failure cases, are provided in Appendix~\ref{app:case}.

\begin{figure}[!htb]
  \centering
  \includegraphics[width=0.5\textwidth]{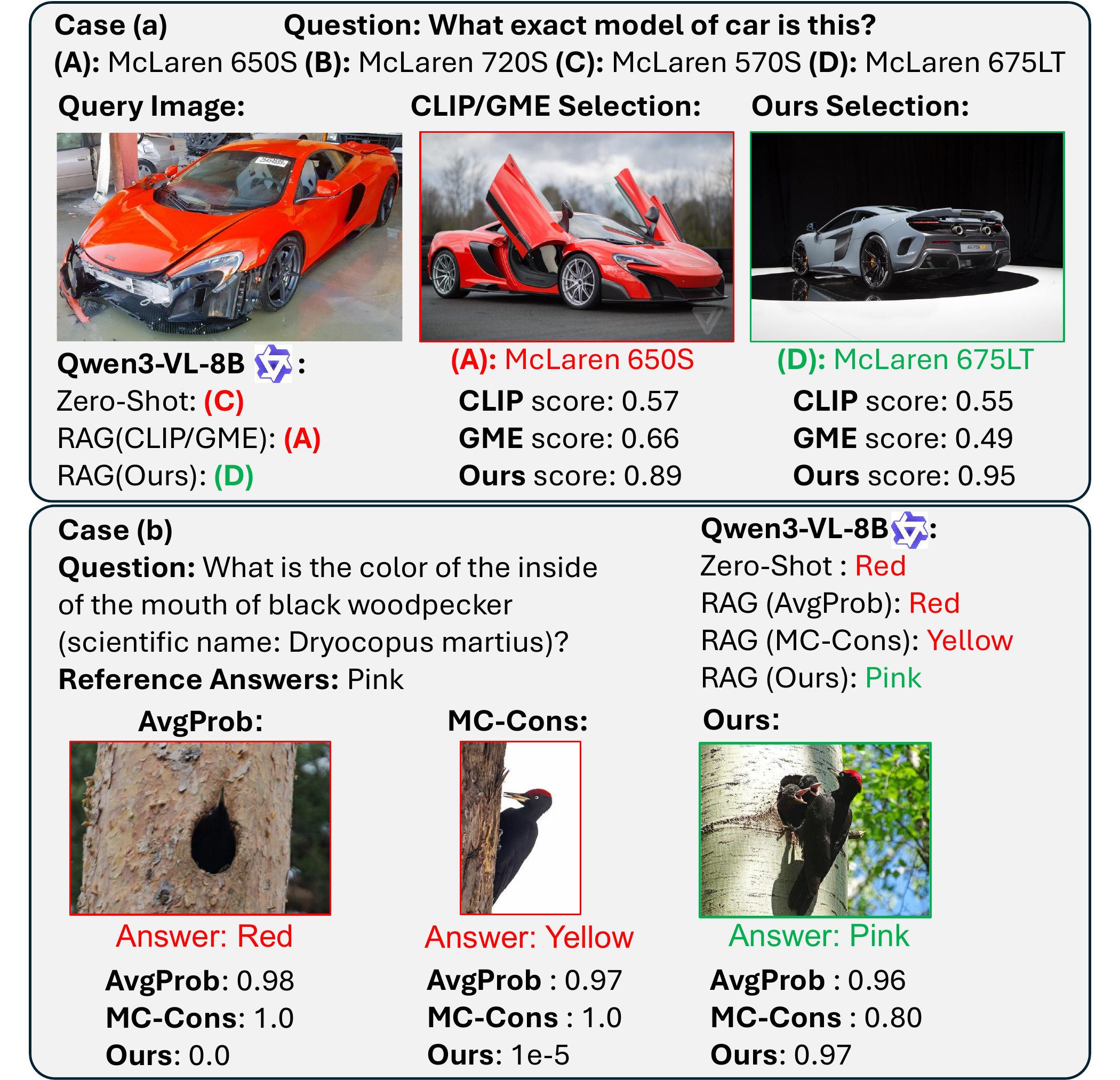}
    \caption{Qualitative case studies. \textbf{(a)} A case illustrating the gap between retrieval relevance and downstream utility: CLIP/GME scores denote relevance/similarity-based ranking scores, whereas our score denotes the surrogate helpfulness score used for utility-oriented selection. \textbf{(b)} A case comparing our latent-helpfulness-based ranking with answer-level uncertainty-based ranking: AvgProb and MC-Cons are answer-level uncertainty criteria, whereas our score denotes the surrogate helpfulness score.}

\vspace{-4mm}
\label{fig:case_study}
\end{figure}

\vspace{-1mm}
\section{Conclusion}

In this paper, we investigate the problem of visual evidence selection for multimodal RAG and argue that semantic relevance and answer-space estimation are insufficient for identifying truly helpful visual context. We propose a utility-oriented formulation based on information gain and a discriminative framework that decouples evidence utility from answer-space variability via a latent helpfulness variable, along with a surrogate-accelerated pipeline for efficient deployment. Experiments on two benchmarks across multiple model families show that our method consistently outperforms state-of-the-art RAG baselines while remaining computationally efficient. Beyond these empirical gains, our work highlights the importance of explicitly modeling evidence utility in multimodal RAG and suggests a promising direction for more principled and efficient evidence selection in future multimodal RAG pipelines.

\newpage

\section*{Limitations and Future Work}

While our approach demonstrates strong empirical performance and efficiency across multiple multimodal RAG benchmarks, several limitations point to promising directions for future work.

1). Our surrogate-accelerated pipeline leverages the observation that judgments of evidence helpfulness transfer well across model scales. While our experiments show that lightweight surrogate models closely match the utility preferences of larger models, exploring more systematic strategies for surrogate selection, adaptation, or dynamic scaling remains an interesting direction for future research.

2). We primarily study visual evidence selection in QA-style multimodal RAG settings. Although the proposed utility-oriented formulation is not tied to a specific answer format, our empirical validation is currently limited to benchmark tasks centered on question answering. Extending the framework to other multimodal generation settings, such as captioning, dialogue, long-form reasoning, or interleaved text--visual retrieval, would further clarify its generality. In addition, extending the current visual-only setting to other modalities (e.g., video and audio) or to mixed text--visual evidence selection may further broaden the applicability of this framework.

We believe these directions build naturally on the strengths of our approach and offer opportunities to develop more general and principled evidence selection mechanisms for multimodal RAG systems.

\section*{Ethical Statement}
This work studies visual evidence selection for multimodal retrieval-augmented generation, with the goal of improving the robustness and efficiency of model reasoning. Our research does not involve human subjects, personal data, or sensitive information. All experiments are conducted using publicly available benchmarks and open-source multimodal models.
We emphasize that our method is intended for research and benchmarking purposes, and we do not advocate its use in high-stakes or safety-critical applications without additional safeguards and human oversight. As with all large language and vision-language models, downstream deployment should follow established ethical guidelines and responsible AI practices.

\bibliography{latex/main}

\label{sec:appendix}

\clearpage 
\onecolumn
\appendix

\section{Extended Related Work}
\label{sec:appendix_related_work}

\subsection{Visual Evidence Selection for Multimodal RAG}

Early approaches to visual evidence selection in multimodal RAG adopt CLIP-style dual encoders that embed images and text into a shared representation space, enabling efficient retrieval via embedding similarity~\cite{radford2021learning_clip}. 
Subsequent variants improve training scale, data diversity, or loss design, leading to stronger dense vision--language retrievers such as OpenCLIP~\cite{cherti2023reproducible_openclip}, SigLIP~\cite{tschannen2025siglip}, and BGE-VL~\cite{zhou2025megapairs}. 
These methods learn image--text semantic alignment through contrastive objectives and form the foundation of modern vision--language retrieval.
Beyond bi-encoder architectures, other work explores alternative multimodal alignment designs, including cross-modal encoders that jointly process image--text pairs and are trained with contrastive or matching objectives~\cite{jia2021scaling_ALIGN, li2021align_ALBEF}. 
These relevance-based retrievers are widely adopted as the retrieval backbone in multimodal RAG systems, where retrieved images—often paired with captions or surrounding text—are treated analogously to retrieved text passages~\cite{chen2022murag, talmor2021multimodalqa}. 

More recent work improves relevance estimation by leveraging large multimodal language models (MLLMs) 
either as dense retrievers or as explicit relevance scorers.
Most existing approaches convert MLLMs into embedding encoders, 
representing queries and images as dense vectors and performing retrieval via cosine similarity 
\cite{lin2024mmembed,jiang2024vlm2vec,meng2025vlm2vecV2,zhou2025megapairs,zhang2024gme,jiang2024e5}. 
This design preserves the efficiency of dual-encoder retrieval while benefiting from stronger multimodal representations.
A smaller set of methods instead uses MLLMs as relevance scorers, 
directly estimating query-image relevance via pointwise or listwise scoring 
\cite{liu2023visual_llava,li2023blip2,chen2024ragllava,liu2025lamra,gu2025unime}.

Despite differences in architecture and scoring granularity, MLLM-based methods remain fundamentally relevance-driven. An image is selected if the model judges it to be semantically helpful or relevant to the query. While MLLMs provide a more expressive and flexible relevance estimator, they do not explicitly model whether the image contains sufficient or discriminative evidence to answer the question, which is the our focus in this paper.

\subsection{Uncertainty, Utility, and Efficiency in Evidence Selection}

\paragraph{Uncertainty estimation in LLMs.}
A substantial line of work studies answer-level uncertainty estimation for large language models, including token- or sequence-level confidence and likelihood, entropy-based measures and Monte Carlo approximations (e.g., predictive entropy and dropout), as well as consistency-based signals such as self-consistency across multiple generations and NLI-based entailment or contradiction checks~\citep{kadavath2022language,kuhn2023semantic,wang2022self,lin2023generating}. These signals are primarily designed to assess answer reliability and are commonly used for selective answering, abstention, or deciding whether additional retrieval is needed.

\paragraph{The Use of Uncertainty in RAG}
In retrieval-augmented generation, uncertainty signals have been used as indirect indicators of evidence reliability, for example to trigger adaptive retrieval when confidence is low~\citep{jiang2023active}, to compare answers generated from different retrieved documents for consistency~\citep{xu2025mega}, or to verify answer–evidence alignment via NLI-style checks~\citep{lin2023generating,kuhn2023semantic}. While there is existing study of uncertainty in RAG largely analyzing how uncertainty changes when retrieval is introduced, i.e., contrasting RAG with no-RAG settings~\citep{soudani2025why,ozaki2025understanding}, our work is complementary: we investigate how uncertainty signals can be used to \emph{horizontally compare the utility of different candidate evidence}, rather than to characterize the aggregate effect of adding retrieval.

Beyond their effectiveness, existing uncertainty estimation methods also exhibit fundamental limitations in signal quality. In particular, LLM uncertainty is often miscalibrated and prone to overconfidence, even when conditioned on misleading or irrelevant retrieved evidence~\citep{soudani2025why,ozaki2025understanding}. Moreover, many widely used uncertainty estimators, such as Monte Carlo sampling or self-consistency, are inherently inefficient, requiring repeated generations and significantly increasing inference cost.
Our work instead focuses on a discriminative notion of evidence utility that avoids costly uncertainty sampling and supports efficient, direct comparison among candidate evidence.

\section{Information-Theoretic Preliminaries and General Notations}
\label{sec:preliminary}

Let $X$, $Y$, and $Z$ be three general random variables respectively.  Let $\mu_{X,Y}$ and $\mu_{Y|X}$ denote the joint and conditional distributions of $(X, Y)$ receptively. The \textit{conditional entropy} of $Y$ given $X$ is defined as 
\[H(Y \mid X):=
-\mathbb{E}_{(x,y)\sim \mu_{X,Y}}
\big[ \log \mu_{Y\mid X}(y \mid x) \big], \]
which measures the uncertainty of the output given the input. Let $\mu$ and $\mu'$ be probability distributions supported on the same set $\mathcal{Y}$, such that $\mu'(y)=0$ implies $\mu(y)=0$. 
The \textit{Kullback--Leibler} (KL) divergence between $\mu$ and $\mu'$ is defined as 
\[
D_{\mathrm{KL}}(\mu \,\|\, \mu')
:=
\sum_{y}
\mu(y)\log\frac{\mu(y)}{\mu'(y)}.\]
The KL divergence is non-negative and equals zero if and only if $\mu=\mu'$. 

For a fixed input $x\in\mathcal{X}$ and a realization $z\in\mathcal{Z}$, the \textit{information gain} with respect to the output variable $Y$ is defined as
\[\mathrm{IG}(Y; Z=z \mid X=x)
:= D_{\mathrm{KL}}\!\left(
\mu_{Y \mid X=x, Z=z}
\;\middle\|\;
\mu_{Y \mid X=x}
\right).\]
Equivalently, the information gain can be expressed as a reduction in conditional entropy: 
\[\mathrm{IG}(Y; Z=z \mid X=x)
= H(Y \mid X=x) - H(Y \mid X=x, Z=z).\]

\section{Theoretical Results and Proof}
\label{sec:proof}

In this section, we rigorously present our theorems in the formal form with their proof.

To simplify notations, throughout this section, we omit the user query $q$ notation in the conditional probability as it is fixed throughout. We also naively assume that the optimal solutions to Problems \eqref{eq:utility_kl_max}, \eqref{eq:latent_kl_max}, and \eqref{eq:latent_1_max} exist. For instance, this holds naturally when $\mathcal{C}$ is a finite set.

First, we recall that $Z \in \{0,1\}$ is a Bernoulli random variable. In addition, recall that the (evidence-wise) information gain is 
\[
IG(Y;C=c) := D_{\mathrm{KL}}\!\big(P_{Y\mid C=c}\,\|\,P_Y\big),
\qquad
IG(Z;C=c) := D_{\mathrm{KL}}\!\big(P_{Z\mid C=c}\,\|\,P_Z\big),
\]
where $P_Y$ and $P_Z$ denote the marginals under the law of $(C,Y,Z)$.

Define $p_c = P(Z=1|C=c)$ and $\bar{p} = P(Z=1)$. Hence, $P_{Z\mid C=c}=\mathrm{Bern}(p_c)$ for any evidence $c$. 

\begin{assumption}[Constraints on helpful evidences]
\label{assumption:helpfulevidence}
Assume that when we solve Problem \eqref{eq:utility_kl_max}, Problem \eqref{eq:latent_kl_max}, or Problem \eqref{eq:latent_1_max}, we add a constraint on the set $\mathcal{C}$ by restricting the feasible set into the subset of the candidate evidence $\hat{\mathcal{C}}$ that generates useful information, in other words, $p_c \ge \bar{p}$ for all $c \in \hat{\mathcal{C}}$.    
\end{assumption}

\begin{assumption}[Connection between response and helpfulness]
\label{assumption:helpfulnessscore}
Assume the following conditions hold.

\begin{enumerate}
\item (Distribution for response) There exist distributions $P_0,P_1$
and a measurable map $\lambda:\mathcal{C}\to[0,1]$ such that for any evidence $c$,
\begin{equation}
P_{Y\mid C=c} \;=\; P_{\lambda(c)} := (1-\lambda(c))\,P_0+\lambda(c)\,P_1.
\label{eq:C2}
\end{equation}
Let $\bar\lambda := \mathbb E[\lambda(C)]$, so that the marginal satisfies
\begin{equation}
P_Y = \mathbb E_C[P_{Y\mid C}] = (1-\bar\lambda)\,P_0+\bar\lambda\,P_1 = P_{\bar\lambda}.
\label{eq:C3}
\end{equation}

To interpret, $P_0,P_1$ are two “reference” answer distributions where $P_0$ represents the LLM response distribution when it does not use any evidence, and $P_1$ represents the LLM response distribution when it does use the evidence in general.


\item (Monotone alignment between helpfulness and usage) For any two evidences $c_1,c_2\in \{c\in \hat{C}: \lambda(c)\ge \bar\lambda\}$,
\begin{equation}
IG(Z;C=c_1)\ge IG(Z;C=c_2) \quad\Longrightarrow\quad \lambda(c_1)\ge \lambda(c_2).
\label{eq:C4}
\end{equation}

When Assumption \ref{assumption:helpfulevidence} holds, this equation can also be replaced by
\begin{equation*}
p_{c_1}\ge p_{c_2} \quad\Longrightarrow\quad \lambda(c_1)\ge \lambda(c_2).
\end{equation*}

To interpret, a more helpful evidence is used more strongly in the response and leads to a higher weight toward the distribution $P_1$.
 
\end{enumerate}
\end{assumption}

We first prove Theorem \ref{thm:equivalent2} and then prove Theorem \ref{thm:equivalent1}.

\begin{theorem}[Theorem \ref{thm:equivalent2}]
\label{thm:equivalent2_new}
Suppose that Assumption \ref{assumption:helpfulevidence} holds. 
Then we have
\[
IG(Z;C=c_1)\ge IG(Z;C=c_2)
\quad\Longleftrightarrow\quad
P(Z=1\mid C=c_1)\ge P(Z=1\mid C=c_2).
\]
for all $c_1, c_2 \in \hat{C}$
and Problem \eqref{eq:latent_kl_max} has the identical solutions as in Problem \eqref{eq:latent_1_max}, i.e., 
\[
\argmax_{c\in \hat{C}} IG(Z;C=c) =
\argmax_{c\in \hat{C}} P(Z=1\mid C=c).
\]
\end{theorem}

\begin{proof}[Proof of Theorem \ref{thm:equivalent2_new}]

Since \(Z\in\{0,1\}\), it is easy to see that by definition,
\[
IG(Y;C=c)\;=\;D_{\mathrm{KL}}\!\big(\mathrm{Bern}(p_c)\,\|\,\mathrm{Bern}(p)\big)\;=:\;f(p_c),
\]
where for \(q\in(0,1)\), the function $f$ is defined as
\[
f(q)\;:=\;q\log\frac{q}{\bar{p}}+(1-q)\log\frac{1-q}{1-\bar{p}}.
\]

Step 1: \(f\) is increasing on \([\bar{p},1]\).
Differentiating \(f\) with respect to $q$:
\[
f'(q)
=\log\frac{q}{\bar{p}}-\log\frac{1-q}{1-\bar{p}}
=\log\frac{q(1-\bar{p})}{(1-q)\bar{p}}.
\]
If \(q\ge \bar{p}\), then \(\frac{q}{1-q}\ge \frac{\bar{p}}{1-\bar{p}}\), hence
\[
\frac{q(1-\bar{p})}{(1-\bar{p})p}\ge 1
\quad\Longrightarrow\quad
f'(q)\ge 0,
\]
with strict inequality when \(q>p\). Therefore \(f\) is strictly increasing on the set \([\bar{p},1]\).

Step 2: monotonicity.
Assume \(p_{c_1},p_{c_2}\in[\bar{p},1]\). Since \(f\) is increasing on \([\bar{p},1]\),
\[
IG(Z;C=c_1)\ge IG(Z;C=c_2)
\iff f(p_{c_1})\ge f(p_{c_2})
\iff p_{c_1}\ge p_{c_2}.
\]

Hence, under the condition \(p_{c_1},p_{c_2}\ge \bar{p}\),
\[
IG(Z;C=c_1)\ge IG(Z;C=c_2)
\quad\Longleftrightarrow\quad
P(Z=1\mid C=c_1)\ge P(Z=1\mid C=c_2).
\]
and
\[
\argmax_{c\in \hat{C}} IG(Z;C=c) =
\argmax_{c\in \hat{C}} P(Z=1\mid C=c).
\]

\end{proof}

\begin{theorem} [Theorem \ref{thm:equivalent1}]
\label{thm:equivalent1_new}
Suppose that Assumptions \ref{assumption:helpfulevidence} and \ref{assumption:helpfulnessscore} hold. Then we have
\[
IG(Z;C=c_1)\ge IG(Z;C=c_2)\quad\Longrightarrow\quad IG(Y;C=c_1)\ge IG(Y;C=c_2).
\]
for all $c_1, c_2 \in \{c\in \hat{C}: \lambda(c)\ge \bar\lambda\}$
and the optimal solutions to Problem \eqref{eq:latent_kl_max} must be optimal to Problem \eqref{eq:utility_kl_max}, i.e.,
\[
\argmax_{c\in \hat{C}: \lambda(c)\ge \bar\lambda} IG(Z;C=c) \in \argmax_{c\in \hat{C}: \lambda(c)\ge \bar\lambda} IG(Y;C=c).
\]
In particular, if the optimal solution to Problem \eqref{eq:utility_kl_max} is unique, then
\[
\argmax_{c\in \hat{C}: \lambda(c)\ge \bar\lambda} IG(Z;C=c) = \argmax_{c\in \hat{C}: \lambda(c)\ge \bar\lambda} IG(Y;C=c)
\]
for $K=1$.
\end{theorem}

\begin{proof}[Proof of Theorem \ref{thm:equivalent1_new}]
We proceed as follows.

Step 1: monotonicity of the map $\lambda\mapsto D_{\mathrm{KL}}(P_\lambda\|P_{\bar\lambda})$
on $[\bar\lambda,1]$. 
Define
\[
g(\lambda):=D_{\mathrm{KL}}(P_\lambda\,\|\,P_{\bar\lambda}),
\qquad P_\lambda=(1-\lambda)P_0+\lambda P_1.
\]
Note that by convexity of KL in its first argument and the linearity of $\lambda\mapsto P_\lambda$,
the function $g$ is convex on $[0,1]$. To see this, fix $\lambda_1,\lambda_2\in[0,1]$ and $t\in[0,1]$. By linearity of $\lambda\mapsto P_\lambda$,
\begin{align*}
P_{t\lambda_1+(1-t)\lambda_2}
&=(1-(t\lambda_1+(1-t)\lambda_2))P_0+(t\lambda_1+(1-t)\lambda_2)P_1\\
&=t\big((1-\lambda_1)P_0+\lambda_1 P_1\big)+(1-t)\big((1-\lambda_2)P_0+\lambda_2 P_1\big)\\
&=tP_{\lambda_1}+(1-t)P_{\lambda_2}.
\end{align*}
Next, use the convexity of KL in its first argument (Lemma \ref{lem:KL_convex_first_jensen}): for any probability measures $Q_1,Q_2,Q_3$,
\[
D_{\mathrm{KL}}(tQ_1+(1-t)Q_3\,\|\,Q_2)\ \le\ t\,D_{\mathrm{KL}}(Q_1\,\|\,Q_2)+(1-t)\,D_{\mathrm{KL}}(Q_3\,\|\,Q_2).
\]
Applying this with $Q_1=P_{\lambda_1}$, $Q_3=P_{\lambda_2}$, and $Q_2=P_{\bar\lambda}$ yields
\begin{align*}
g(t\lambda_1+(1-t)\lambda_2)
&=D_{\mathrm{KL}}(P_{t\lambda_1+(1-t)\lambda_2}\,\|\,P_{\bar\lambda})\\
&=D_{\mathrm{KL}}(tP_{\lambda_1}+(1-t)P_{\lambda_2}\,\|\,P_{\bar\lambda})\\
&\le t\,D_{\mathrm{KL}}(P_{\lambda_1}\,\|\,P_{\bar\lambda})+(1-t)\,D_{\mathrm{KL}}(P_{\lambda_2}\,\|\,P_{\bar\lambda})\\
&=t\,g(\lambda_1)+(1-t)\,g(\lambda_2),
\end{align*}
which is exactly convexity of $g$.

Moreover, by~\eqref{eq:C3},
\[
g(\bar\lambda)=D_{\mathrm{KL}}(P_{\bar\lambda}\,\|\,P_{\bar\lambda})=0,
\]
so $\bar\lambda$ is a global minimizer of $g$. A convex function on an interval that
attains its minimum at $\bar\lambda$ is nondecreasing on $[\bar\lambda,1]$.
Therefore, for all $c_1, c_2 \in \{c\in \hat{C}: \lambda(c)\ge \bar\lambda\}$
\begin{equation}
\lambda(c_1)\ge \lambda(c_2)\quad\Longrightarrow\quad
g(\lambda(c_1))\ge g(\lambda(c_2)).
\label{eq:step2}
\end{equation}

Step 2: conclude.
From the alignment assumption~\eqref{eq:C4},
\[
IG(Z;C=c_1)\ge IG(Z;C=c_2)\ \Longrightarrow\ \lambda(c_1)\ge\lambda(c_2).
\]
Applying~\eqref{eq:step2} and using~\eqref{eq:C2}--\eqref{eq:C3},
\[
IG(Y;C=c_1)=D_{\mathrm{KL}}(P_{Y\mid C=c_1}\,\|\,P_Y)
= D_{\mathrm{KL}}(P_{\lambda(c_1)}\,\|\,P_{\bar\lambda})
= g(\lambda(c_1))
\ge g(\lambda(c_2))
= IG(Y;C=c_2).
\]
This clearly implies 
\[
\argmax_{c\in \hat{C}: \lambda(c)\ge \bar\lambda} IG(Z;C=c) \in \argmax_{c\in \hat{C}: \lambda(c)\ge \bar\lambda} IG(Y;C=c).
\]
\end{proof}

\begin{lemma}[Convexity of KL in the first argument]
\label{lem:KL_convex_first_jensen}
Let $(\mathcal X,\mathcal F)$ be a measurable space and let $Q$ be a probability measure on it.
For any probability measures $P$ and $R$ on $(\mathcal X,\mathcal F)$ and any $t\in[0,1]$, define
\[
M:=tP+(1-t)R.
\]
Then
\begin{equation}
D_{\mathrm{KL}}(M\|Q)\ \le\ t\,D_{\mathrm{KL}}(P\|Q)+(1-t)\,D_{\mathrm{KL}}(R\|Q),
\tag{$\ast$}
\label{eq:KL_convex_first_general}
\end{equation}
with the convention that $D_{\mathrm{KL}}(\cdot\|\cdot)\in[0,\infty]$.
\end{lemma}

\begin{proof}
If either $P\not\ll Q$ or $R\not\ll Q$, then the right-hand side of \eqref{eq:KL_convex_first_general}
is $+\infty$, hence the inequality holds trivially. Thus assume $P\ll Q$ and $R\ll Q$.
Let
\[
p:=\frac{dP}{dQ},\qquad r:=\frac{dR}{dQ}.
\]
Then also $M\ll Q$ with Radon--Nikodym derivative
\[
m:=\frac{dM}{dQ}=t p+(1-t) r.
\]

Let $\phi:[0,\infty)\to\mathbb R$ be defined by $\phi(u):=u\log u$ (with the convention $0\log 0:=0$).
It is convex on $[0,\infty)$ since $\phi''(u)=1/u$ for $u>0$.

Recall the variational form of KL in terms of $\phi$:
\[
D_{\mathrm{KL}}(P\|Q)=\int \phi(p)\,dQ,\qquad
D_{\mathrm{KL}}(R\|Q)=\int \phi(r)\,dQ,\qquad
D_{\mathrm{KL}}(M\|Q)=\int \phi(m)\,dQ.
\]

Fix $x\in\mathcal X$.
Define an auxiliary Bernoulli random variable $B\sim\mathrm{Bern}(t)$ independent of everything else, and
define the real random variable
\[
Z_x :=
\begin{cases}
p(x), & B=1,\\
r(x), & B=0.
\end{cases}
\]
Then $\mathbb E[Z_x]=t p(x)+(1-t)r(x)=m(x)$ and
\[
\mathbb E[\phi(Z_x)] = t\,\phi(p(x))+(1-t)\,\phi(r(x)).
\]
By Jensen's inequality applied to the convex function $\phi$,
\[
\phi\!\big(\mathbb E[Z_x]\big)\ \le\ \mathbb E[\phi(Z_x)],
\]
i.e.,
\[
\phi\!\big(m(x)\big)\ \le\ t\,\phi\!\big(p(x)\big)+(1-t)\,\phi\!\big(r(x)\big).
\]
Integrating both sides with respect to $Q$ yields
\[
\int \phi(m)\,dQ
\ \le\
t\int \phi(p)\,dQ+(1-t)\int \phi(r)\,dQ,
\]
which is exactly
\[
D_{\mathrm{KL}}(M\|Q)\ \le\ t\,D_{\mathrm{KL}}(P\|Q)+(1-t)\,D_{\mathrm{KL}}(R\|Q).
\]
\end{proof}

We remark that Assumption \ref{assumption:helpfulnessscore} Part 2 can hold in some natural scenarios. For instance, we have the following lemma.

\begin{lemma}
\label{lemma:conditionforassumption}
Assume there exists a measurable ``helpfulness" score function $S:\mathcal{C} \to \mathbb R$,
and a real-valued noise random variable $\varepsilon$ with a increasing CDF that is independent of $C$, such that $Z \;=\; \mathbf 1\{S(C)+\varepsilon \ge 0\}$.
Moreover, there exists a non-decreasing function $\lambda':\mathbb R\to[0,1]$ such that $\lambda(c) = \lambda'(S(c))$. Then Assumption \ref{assumption:helpfulevidence} implies Assumption \ref{assumption:helpfulnessscore} Part 2.
\end{lemma}

\begin{proof}[Proof of Lemma \ref{lemma:conditionforassumption}]
Note that for any $c\in\mathcal C$,
\begin{equation*}
P(Z=1\mid C=c)
\;=\;
P(\varepsilon \ge -S(c))
\;=\;
1-F_\varepsilon(-S(c))
\;=:\;
h(S(a)),
\end{equation*}
where $h(\cdot)$ is increasing since $F_\varepsilon$ is increasing. In other words, a higher helpfulness score implies a higher helpfulness probability. When $p_{c_1}\ge p_{c_2}$, this naturally leads to $S(c_1)\ge S(c_2)$ and thus $\lambda(c_1) = \lambda'(S(c_1))\ge \lambda'(S(c_2))  = \lambda(c_2)$.
\end{proof}

\section{Experimental Setup Details}
\label{sec:appendix_setup}

\subsection{Benchmarks}
\label{app:subsec:benchmarks}

\paragraph{MRAG-Bench.}
MRAG-Bench~\cite{hu2024mragbench} is a vision-centric multimodal RAG benchmark designed to evaluate whether large vision-language models can effectively retrieve and utilize visual knowledge rather than textual evidence.
The dataset consists of 1,353 human-annotated multiple-choice questions paired with 16,130 images, spanning nine real-world scenarios where visual augmentation is crucial. These scenarios cover both perspective-based variations (e.g., angle, occlusion, partial views) and transformative changes (e.g., temporal evolution, deformation, biological processes).
This benchmark emphasizes situations where textual knowledge is insufficient or difficult to retrieve, thereby highlighting the importance of vision-centric retrieval and reasoning.
Following the original benchmark protocol, each query in MRAG-Bench is associated with a set of human-annotated ground-truth images. In our experiments, we construct the final candidate image set for each query by taking the union of the annotated ground-truth images and the images retrieved by the retriever, ensuring that both oracle visual evidence and realistically retrieved images are available during evaluation.

\paragraph{Visual-RAG.}
Visual-RAG~\cite{wu2025visualrag} is a recently proposed benchmark that focuses on visual evidence--centric retrieval-augmented generation. Unlike prior datasets, Visual-RAG uses text-only queries and requires models to retrieve images that explicitly contain the visual evidence necessary for answering the question.
The benchmark contains 374 fine-grained, knowledge-intensive queries grounded in the organism domain, paired with an image corpus of 99,017 images sourced from iNaturalist 2021. For each query, only a small fraction of images serve as \emph{clue images}, while the majority are visually similar hard negatives that lack the queried attribute.
This design enforces a challenging text-to-image retrieval setting and evaluates whether models can accurately extract and ground detailed visual knowledge from retrieved images, rather than relying on parametric prior knowledge.
To align the experimental setting with MRAG-Bench, we construct a fixed-size candidate image set for each query.
Specifically, we first select up to five ground-truth clue images with the highest retrieval scores produced by the BGE-VL-Large model~\cite{zhou2025megapairs} when available.
If fewer than five ground-truth images exist for a query, all available ground-truth images are included.
The remaining slots are filled with globally top-ranked images retrieved by BGE-VL-Large, such that each query is associated with a total of ten candidate images.

Together, MRAG-Bench and Visual-RAG provide complementary evaluations of multimodal RAG systems. MRAG-Bench focuses on leveraging additional visual perspectives and transformations to support recognition and reasoning, while Visual-RAG explicitly isolates the ability to retrieve and exploit fine-grained visual evidence for knowledge-intensive answer generation. Evaluating on both benchmarks allows us to assess the robustness and generality of our approach across diverse vision-centric RAG scenarios.

\paragraph{Assumption Stress Test Benchmark Construction}

We construct three progressively perturbed candidate regimes, including retrieval-only deployment, structured ambiguity with hard negatives, and combined structured–stochastic noise, to empirically examine the practical validity of the underlying assumptions under increasingly heterogeneous retrieval conditions.

1) Pure Retrieve.
This setting captures a practical end-to-end deployment scenario where evidence selection operates purely on retrieved candidates without oracle guidance. Specifically, the candidate pool is constructed solely using a widely adopted high-performing retriever (BGE-VL-Large), reflecting realistic retrieval conditions in which candidate sets are inherently heterogeneous and may contain partially misleading, ambiguous, or semantically overlapping images.

2) GT + Hard Negatives.
This setting simulates a realistic ambiguity scenario in which truly informative evidence coexists with highly similar but incorrect candidates. Specifically, we enforce the inclusion of five ground-truth (GT) images (ordered by retriever relevance score) to ensure the presence of informative evidence. The remaining positions are filled with top-ranked but non-relevant images retrieved by the same retriever, forming structured hard negatives that are semantically close yet potentially contradictory or misleading.

3) GT + Hard Negatives + Stochastic Perturbation.
This setting simulates a more severely perturbed retrieval regime in which informative evidence and structured hard negatives coexist with additional unstructured noise. Specifically, we retain five GT images and five fixed hard negatives, while randomly sampling the remaining candidates from non-relevant images. This configuration introduces both retrieval-aligned ambiguity and stochastic perturbations, reflecting practical scenarios where candidate pools contain a mixture of relevant evidence, confounding near-matches, and arbitrary distractors.

For each variant, we systematically vary the candidate pool size (10 / 15 / 20), thereby increasing the proportion of hard negatives and stochastic distractors and progressively amplifying perturbation intensity. We report Top-1 / Top-3 / Top-5 selection performance across two backbone models to assess stability under escalating noise levels.

\subsection{Baselines}

\paragraph{Zero-shot and oracle settings.}
We report a zero-shot setting where no retrieved images are provided, as well as a GT oracle setting where all ground-truth images are supplied.

\paragraph{Relevance-based retrievers.}
We include CLIP-style retrievers and their variants, including CLIP~\cite{radford2021learning_clip}, OpenCLIP~\cite{cherti2023reproducible_openclip}, SigLIP~2~\cite{tschannen2025siglip}, and BGE-VL-Large~\cite{zhou2025megapairs}, which rank candidate images by semantic similarity.

\paragraph{MLLM-based retrievers and rerankers.}
We evaluate recent MLLM-based methods for relevance estimation, including dense retrievers (E5-V~\cite{jiang2024e5}, GME~\cite{zhang2024gme}, VLM2Vec~\cite{jiang2024vlm2vec,meng2025vlm2vecV2}, BGE-MLLM~\cite{zhou2025megapairs}) as well as pointwise or listwise rerankers (LamRA-Rank~\cite{hu2024mragbench}, UniME-V2~\cite{gu2025unime}).

\paragraph{Answer-level uncertainty baselines.}
We include answer-level uncertainty-based selection strategies as baselines. Since MRAG-Bench and Visual-RAG involve different question formats, we adopt dataset-specific uncertainty measures. For MRAG-Bench, which consists of multiple-choice questions, we use choice-level softmax entropy. For Visual-RAG, which features open-ended questions, we consider two uncertainty measures: average token probability and Monte Carlo sampling with NLI-based consistency estimation. Detailed implementations of these uncertainty measures are provided in Appendix~\ref{app_subsec:implementation_details}.

\subsection{Backbone Models}
We evaluate our method across five families of recent open-source MLLMs, covering diverse architectural designs and parameter scales while maintaining comparable instruction-following capabilities:
\begin{itemize}[leftmargin=*,itemsep=1pt]
    \item \textbf{Qwen3-VL~\cite{bai2025qwen3vltechnicalreport}}: Qwen/Qwen3-VL-2B-Instruct\footnote{\url{https://huggingface.co/Qwen/Qwen3-VL-2B-Instruct}} and Qwen/Qwen3-VL-8B-Instruct\footnote{\url{https://huggingface.co/Qwen/Qwen3-VL-8B-Instruct}};
    \item \textbf{InternVL3.5~\cite{wang2025internvl3}}: OpenGVLab/InternVL3\_5-2B\footnote{\url{https://huggingface.co/OpenGVLab/InternVL3_5-2B}} and OpenGVLab/InternVL3\_5-8B\footnote{\url{https://huggingface.co/OpenGVLab/InternVL3_5-8B}};
    \item \textbf{Gemma3~\cite{team2025gemma3}}: google/gemma-3-4b-it\footnote{\url{https://huggingface.co/google/gemma-3-4b-it}} and google/gemma-3-12b-it\footnote{\url{https://huggingface.co/google/gemma-3-12b-it}};
    \item \textbf{OVIS2.5~\cite{lu2025ovis2}}: AIDC-AI/Ovis2.5-2B\footnote{\url{https://huggingface.co/AIDC-AI/Ovis2.5-2B}} and AIDC-AI/Ovis2.5-9B\footnote{\url{https://huggingface.co/AIDC-AI/Ovis2.5-9B}};
    \item \textbf{MiniCPM4.5~\cite{yu2025minicpm}}: openbmb/MiniCPM-V-4\_5\footnote{\url{https://huggingface.co/openbmb/MiniCPM-V-4_5}} and its quantized variant MiniCPM-V-4\_5-AWQ\footnote{\url{https://huggingface.co/openbmb/MiniCPM-V-4_5-AWQ}}.
\end{itemize}

For each model family, we evaluate two variants with different parameter sizes (ranging from 2B to 12B) to support our evaluation of the transferability of surrogate models. For MiniCPMV4.5, we additionally evaluate an AWQ-quantized variant to examine the robustness of our approach under memory-efficient deployment settings.
All backbone models share the same retrieval results and candidate image sets to ensure fair comparison.

\subsection{Evaluation Metrics}

For MRAG-Bench, we report exact-match accuracy. For Visual-RAG, we follow the original benchmark protocol and evaluate generated answers using an LLM-as-Judge metric (details are in Appendix~\ref{app_subsec:model_prompts}). We report Top-$K$ results for $K\in\{1,\dots,5\}$. Computational cost is measured using FLOPs (prefill and decoding) via the \texttt{calflops} library\footnote{\url{https://github.com/MrYxJ/calculate-flops.pytorch}}~\cite{calflops}, together with end-to-end inference latency.

\subsection{Implementation Details and Evaluation Protocols}
\label{app_subsec:implementation_details}

\paragraph{Ground-Truth Image Selection Protocol.}
For the ground-truth (GT) oracle setting reported in the main results, we construct the visual evidence set by sampling from the annotated ground-truth images associated with each query. Specifically, for a given query, we randomly shuffle its ground-truth image pool and select the first $K$ images as the GT evidence set. This procedure ensures that settings with larger $K$ strictly contain those with smaller $K$, analogous to pointwise top-$K$ selection strategies. All GT results reported in the main tables are obtained using this sampling protocol.

\paragraph{Uncertainty estimation on MRAG-Bench.}
MRAG-Bench consists of multiple-choice questions with four candidate options. For this setting, we adopt choice-level softmax entropy as the answer uncertainty measure. Specifically, given the final-layer logits corresponding to the option tokens (A, B, C, D), we apply a softmax operation to obtain the choice probability distribution $\{p_i\}_{i=1}^4$. The uncertainty score is then computed as the entropy of this distribution:
\begin{equation}
\mathcal{U}_{\text{MC}} = - \sum_{i=1}^4 p_i \log p_i.
\end{equation}
Higher entropy indicates greater uncertainty over the answer choices. Retrieved images are ranked according to the resulting uncertainty scores.

\paragraph{Uncertainty estimation on Visual-RAG.}
Visual-RAG involves open-ended question answering, for which choice-level uncertainty is not applicable. We therefore consider two commonly used answer-level uncertainty measures.

\emph{Average token probability.}
Given a generated answer sequence $a = (w_1, w_2, \dots, w_T)$, we compute the average token probability as an answer-level uncertainty measure. At each decoding step $t$, we apply a softmax over the vocabulary to obtain the token distribution $p(\cdot \mid w_{<t})$, and extract the probability of the generated token $w_t$. The uncertainty score is then computed as
\begin{equation}
\mathcal{U}_{\text{avg}} = \frac{1}{T} \sum_{t=1}^{T} p(w_t \mid w_{<t}).
\end{equation}
Lower average token probability indicates higher answer uncertainty.

\emph{Monte Carlo sampling with NLI-based consistency.}
We further implement a Monte Carlo sampling-based uncertainty estimator. For each input, we first generate a deterministic answer $a_0$ using greedy decoding (temperature $=0$). We then sample $N$ additional answers $\{a_j\}_{j=1}^N$ using temperature $=1.0$, where $N=5$ in our experiments. To quantify uncertainty, we evaluate the semantic consistency between the deterministic answer $a_0$ and each sampled answer $a_j$ using a natural language inference (NLI) model\footnote{\url{https://huggingface.co/microsoft/deberta-large-mnli}}. The final uncertainty score is computed by aggregating the pairwise consistency scores, following the implementation in the \texttt{uqlm} repository\footnote{\url{https://github.com/cvs-health/uqlm}}. Lower consistency indicates higher uncertainty.

\paragraph{Computation Cost Measurement.}
To evaluate the computational efficiency of different evidence selection strategies, we follow a unified and reproducible measurement protocol. Specifically, we randomly select 100 query IDs from the Visual-RAG benchmark and collect all associated candidate images, resulting in a total of 1{,}000 query-image pairs. This set is used as the benchmark for computation cost evaluation. We measure the floating-point operations (FLOPs) using the \texttt{calflops} library\footnote{\url{https://github.com/MrYxJ/calculate-flops.pytorch}}, reporting both prefill and decode FLOPs. Prefill FLOPs correspond to the cost of the first forward pass given the full input, while decode FLOPs correspond to the cost of each subsequent forward pass during generation. All reported FLOPs are averaged over the full set of 1{,}000 query-image pairs. Latency is measured following the same protocol, reporting the average prefill latency and per-step decode latency across all samples.

\subsection{Model Prompts}
\label{app_subsec:model_prompts}

This section provides the exact prompt templates and evaluation procedures used in our experiments, covering answer generation, auxiliary utility probing, and ablation baselines. These details are included for reproducibility and to complement the main text.

\medskip
\noindent\textbf{Global implementation note (image placeholders).} Some multimodal models do not accept inline textual ``\{Image\}'' placeholders in a unified text prompt (e.g., Gemma3). For those models, we remove the image placeholders and keep the same textual instruction and question. \emph{This handling applies to all prompt templates below}.

\paragraph{MRAG-Bench Answer Generation Prompts.}
We follow the prompt design used in the original MRAG-Bench evaluation~\cite{hu2024mragbench}. The templates below show the text prompts used for (i) the no-RAG (zero-shot) setting and (ii) the RAG setting where retrieved images are provided as additional inputs.
\begin{tcolorbox}[
  colback=white,
  colframe=black!50!white,    
  boxrule=0.5pt,
  arc=3pt,                    
  left=2pt, right=2pt, top=2pt, bottom=2pt,
  title=\small\textbf{Model Prompt: No RAG (MRAG-Bench)},
  fonttitle=\small\bfseries
]
\small
Instruction: Answer with the option's letter from the given choices directly.

\{Image\_placeholder\}

Question: \{QUESTION\}

Choices:

(A) \{OPTION\_A\}

(B) \{OPTION\_B\}

(C) \{OPTION\_C\}

(D) \{OPTION\_D\}

Answer:
\end{tcolorbox}

\begin{tcolorbox}[
  colback=white,
  colframe=black!50!white,
  boxrule=0.5pt,
  arc=3pt,
  left=2pt, right=2pt, top=2pt, bottom=2pt,
  title=\small\textbf{Model Prompt: RAG (MRAG-Bench)},
  fonttitle=\small\bfseries
]
\small
Instruction: You will be given one question concerning several images. The first image is the input image; the remaining images are retrieved examples to help you. Answer with the option's letter from the given choices directly.

\{Image\_placeholder\}

Question: \{QUESTION\}

Choices:

(A) \{OPTION\_A\}

(B) \{OPTION\_B\}

(C) \{OPTION\_C\}

(D) \{OPTION\_D\}

Answer:
\end{tcolorbox}

\paragraph{Visual-RAG Answer Generation Prompts.}
We use dataset-specific prompt templates for Visual-RAG. Two prompt modes are employed: \texttt{zero\_shot}, where no images are provided, and \texttt{image\_prompt}, where one or multiple candidate images are given. All prompts request a concise answer in the format \texttt{"Answer: \{answer\_text\}"}.
\begin{tcolorbox}[
  colback=white,
  colframe=black!50!white,
  boxrule=0.5pt,
  arc=3pt,
  left=2pt, right=2pt, top=2pt, bottom=2pt,
  title=\small\textbf{Visual-RAG: Zero-shot prompt},
  fonttitle=\small\bfseries
]
\small
Please answer the question regarding a visual feature of an organism (animal, plant, etc.). Please follow the answer format: ``Answer: \{answer\_text\}''

Question:
\{QUESTION\}
\end{tcolorbox}

\begin{tcolorbox}[
  colback=white,
  colframe=black!50!white,
  boxrule=0.5pt,
  arc=3pt,
  left=2pt, right=2pt, top=2pt, bottom=2pt,
  title=\small\textbf{Visual-RAG: Image-prompt (single image)},
  fonttitle=\small\bfseries
]
\small
Please answer the question regarding a visual feature of an organism (animal, plant, etc.). You will be provided with an image regarding that organism. If this image does not contain the key information for answering the question, please answer using your internal knowledge. Please follow the answer format: ``Answer: \{answer\_text\}''

\noindent Image(s): \texttt{<image>} (if supported)

Question:
\{QUESTION\}
\end{tcolorbox}

\begin{tcolorbox}[
  colback=white,
  colframe=black!50!white,
  boxrule=0.5pt,
  arc=3pt,
  left=2pt, right=2pt, top=2pt, bottom=2pt,
  title=\small\textbf{Visual-RAG: Image-prompt (multiple images)},
  fonttitle=\small\bfseries
]
\small
Please answer the question regarding a visual feature of an organism (animal, plant, etc.). You will be provided with several images; all of them relate to the organism, but not every image necessarily contains the key information for answering the question. If none of the images contains the key information, please answer using your internal knowledge. Please follow the answer format: ``Answer: \{answer\_text\}''

\noindent Image(s): \texttt{<image><image><image>...} (if supported)

Question:
\{QUESTION\}
\end{tcolorbox}

\paragraph{Visual-RAG LLM-as-Judge Prompt.}
Following the Visual-RAG evaluation protocol, we assess answer quality using an LLM-as-Judge rather than exact matching. We use \texttt{Qwen/Qwen3-8B}\footnote{\url{https://huggingface.co/Qwen/Qwen3-8B}} as the judging model to score generated answers against reference annotations. The judge is instructed to assign a continuous score in $[0,1]$ based on semantic correctness, accounting for partial correctness and visually induced ambiguity, while penalizing hallucinated or irrelevant content.

\begin{tcolorbox}[
  colback=white,
  colframe=black!50!white,
  boxrule=0.5pt,
  arc=3pt,
  left=2pt, right=2pt, top=2pt, bottom=2pt,
  title=\small\textbf{Visual-RAG: LLM-as-Judge prompt},
  fonttitle=\small\bfseries
]
\small
Please evaluate the answer to a question, score from 0 to 1. The reference answer is provided, and the reference is usually short phrases or a single keyword. If the student answer is containing the keywords or similar expressions (including similar or close color/pattern), without any additional guessed information, it is full correct. Similar or close color/pattern includes but not limited to the following cases: pale or light color can appear to be yellowish/greyish under different light condition; dark colors like dark brown, dark grey, dark purple can appear close to each other, and may appear as black as well; stripe pattern can appear as band or ring, dotted pattern can etc. If the student answer have missed some important part in the reference answer, please assign partial score. The reference answer can be in the form of a Python list, in this case, any one of the list item is correct.\\
If student answer contain irrelevant information not related to question, mark it with "Redundant", but it does not affect score if related parts are correct. (e.g. Question: what shape are leaves of XYZ plant, Student Answer: shape xxx, color yyy, color is Redundant answer) If student answer contain features not listed in reference answer, deduct 0.5 score and mark it with "Likely Hallucination". (e.g., Reference Answer: black and white. Student Answer: black white, with yellow dots, "Yellow dots" is not mentioned in reference). The reference answer sometimes contains additional information not asked in question, usually enclosed by brackets (), to help verifying hallucinations (e.g.: "Shape is xxx (color is yyy)"). Not mentioning additional information in answer is not considered wrong. For yes/no question, reference may contain explanations on why giving yes/no, but it is not necessary for student to answer the explanation; however, if student explains and differs with reference, it is considered as hallucination. Answering "I don’t know", "Not enough information" or similar is considered wrong, and please mark it with "No Answer".
\\\\
Format Instructions: Separate the remarks with score using "|", that is, use the syntax of: "Score: score | Likely Hallucination", "Score: score", "Score: score | Likely Hallucination | Redundant", "Score: 0 | No Answer". If any explanation on why giving the score is needed, do not start a new line and append after remark with brackets, e.g. "Score: score | Redundant | (Explanation: abc)".
\\\\
Following are few examples:
\\\\
Question: Is there any specific color marking around the eyes of a semipalmated plover (scientific name: Charadrius semipalmatus)?
Reference Answer: black eye-round feather, white stripe above eyes. (sometimes connected to the white forehead)
\\\\
Student Answer: Yes, the bird has a distinctive black line that runs through the eye, which is a key identifying feature.
Score: 0 | Likely Hallucination
\\\\
Student Answer: They have a black vertical band in front of the eye, a white band above the eye, and a single black band that wraps partially around the eye, creating a partial "mask" appearance.
Score: 1
\\\\
Student Answer: Yes, the semipalmated plover has a distinctive black/dark ring around its eye, surrounded by a bright white ring or patch
Score: 0.5 | Likely Hallucination (Explanation: not white ring, but only a line above the eye)
\\\\
Question: What is the typical color of the antennae of Harris’s checkerspot butterfly (scientific name: Chlosyne harrisii)?
Reference Answer: alternating black and white band, with yellow on the tip
\\\\
Student Answer: The antennae of Harris’s checkerspot butterfly are black with orange-tipped clubs.
Score: 0.5 (Explanation: not mentioning black and white)
\\\\
Student Answer: The typical color of the antennae of Harris’s checkerspot butterfly is black with white spots.
Score: 0.5 | Likely Hallucination (Explanation: not white spot but band. Not mentioning the tip)
\\\\
Question: Are the leaves of burro-weed (scientific name: Ambrosia dumosa) usually covered in small hairs?
Reference Answer: yes
\\\\
Student Answer: Yes, the leaves of burro-weed (Ambrosia dumosa) are typically covered in small hairs, giving them a grayish or whitish-green appearance.
Score: 1 | Redundant
\\\\
Now, Score the following question:

\noindent Image(s): \texttt{<image><image><image>...} (if supported)

Question:
\{QUESTION\}
\end{tcolorbox}

\paragraph{Auxiliary Utility Probing Prompts.}
For discriminative utility estimation, we instantiate the latent helpfulness variable using dataset-specific auxiliary prompts. These prompts elicit a binary judgment on whether a retrieved image provides useful information for answering the original question.
\begin{tcolorbox}[
  colback=white,
  colframe=black!50!white,
  boxrule=0.5pt,
  arc=3pt,
  left=2pt, right=2pt, top=2pt, bottom=2pt,
  title=\small\textbf{Auxiliary Prompt for MRAG-Bench},
  fonttitle=\small\bfseries
]
\small
You will be given two images and a multiple-choice question.
\begin{itemize}[leftmargin=*,itemsep=1pt]
  \item The first image is the input image that the question is about.
  \item The second image is a retrieved image intended to provide additional visual evidence.
\end{itemize}
The retrieved image does not need to answer the question by itself. It is only meant to help answer the question together with the input image.
\\\\
Question:
\{QUESTION\}
\\\\
Choices:
\{CHOICES\}
\\\\
Based on the images provided, does the retrieved image provide helpful visual or factual information that could assist in answering the question correctly?
\\\\
Answer with True or False.
\end{tcolorbox}

\begin{tcolorbox}[
  colback=white,
  colframe=black!50!white,
  boxrule=0.5pt,
  arc=3pt,
  left=2pt, right=2pt, top=2pt, bottom=2pt,
  title=\small\textbf{Auxiliary Prompt for Visual-RAG},
  fonttitle=\small\bfseries
]
\small
You will be given one image and a question about a visual attribute of an organism.
\\\\
The image is retrieved as potential visual evidence. Not all retrieved images contain the information needed to answer the question.
\\\\
Question:
\{QUESTION\}
\\\\
Based on the image provided, does this image contain the key visual information needed to answer the question?
\\\\
Answer with True or False.
\end{tcolorbox}

\paragraph{Ablation Baselines: Verbalized UQ and Listwise Ranking.}
We include two prompt-based ablation baselines. \emph{Verbalized UQ} prompts the model to produce an explicit self-assessed uncertainty score for each retrieved image, which is parsed into a scalar ranking signal. \emph{Listwise ranking} prompts the model to jointly consider multiple retrieved images and output an explicit ranking over candidates. The corresponding prompt templates for MRAG-Bench and Visual-RAG are shown below.
\begin{tcolorbox}[
  colback=white,
  colframe=black!50!white,
  boxrule=0.5pt,
  arc=3pt,
  left=2pt, right=2pt, top=2pt, bottom=2pt,
  title=\small\textbf{Verbalized UQ Prompt for MRAG-Bench},
  fonttitle=\small\bfseries
]
\small
You will be given:
\begin{itemize}[leftmargin=*,itemsep=1pt]
\item An input image that the question refers to.
\item One retrieved image.
\item A multiple-choice question.
\end{itemize}
The retrieved image is intended to serve as additional visual evidence
to help answer the question, together with the input image.
It does not need to answer the question by itself.
\\\\
Input image: (the first image in the provided images list)
Retrieved image: (the second image in the provided images list)
\\\\
Question:
\{QUESTION\}
\\\\
Choices:
\{CHOICES\}
\\\\
Task:
Rate how useful the retrieved image is for answering the question,
considering it together with the input image.
\\\\
Use the following scale:
\begin{itemize}[leftmargin=*,itemsep=1pt]
\item 0 means the retrieved image is completely unhelpful or irrelevant.
\item 100 means the retrieved image provides clear and decisive information that directly supports choosing the correct answer.
\item Intermediate values indicate partial usefulness.
\end{itemize}
Output format:
Output a single integer between 0 and 100.
Do not provide any explanation or additional text.
\end{tcolorbox}

\begin{tcolorbox}[
  colback=white,
  colframe=black!50!white,
  boxrule=0.5pt,
  arc=3pt,
  left=2pt, right=2pt, top=2pt, bottom=2pt,
  title=\small\textbf{Verbalized UQ Prompt for Visual-RAG},
  fonttitle=\small\bfseries
]
\small
You will be given:
\begin{itemize}[leftmargin=*,itemsep=1pt]
\item A question about a visual attribute of an organism.
\item One retrieved image.
\end{itemize}

Not all retrieved images contain the information needed to answer the question.
Some images may be irrelevant or uninformative.
\\\\
Retrieved image: (the image provided)
\\\\
Question:
\{QUESTION\}
\\\\
Task:
Rate how useful the retrieved image is for answering the question.
\\\\
Use the following scale:
\begin{itemize}[leftmargin=*,itemsep=1pt]
\item 0 means the image does not contain any relevant visual information.
\item 100 means the image alone contains clear and sufficient visual information to answer the question.
\item Intermediate values indicate partial usefulness.
\end{itemize}

Output format:
Output a single integer between 0 and 100.
Do not provide any explanation or additional text.
\end{tcolorbox}

\begin{tcolorbox}[
  colback=white,
  colframe=black!50!white,
  boxrule=0.5pt,
  arc=3pt,
  left=2pt, right=2pt, top=2pt, bottom=2pt,
  title=\small\textbf{Listwise Ranking Prompt for MRAG-Bench},
  fonttitle=\small\bfseries
]
\small
You will be given:
\begin{itemize}[leftmargin=*,itemsep=1pt]
\item An input image that the question refers to.
\item \{NUM\_IMAGES\} retrieved images, indexed from Image 1 to Image \{NUM\_IMAGES\}.
\item A multiple-choice question.
\end{itemize}

Each retrieved image is intended to serve as additional visual evidence
to help answer the question, together with the input image.
The retrieved images do not need to answer the question by themselves.
\\\\
Input image: (the first image in the provided images list)
\\\\
Retrieved images:
Image 1, Image 2, ..., Image \{NUM\_IMAGES\}
(The order above is arbitrary.)
\\\\
Question:
\{QUESTION\}
\\\\
Choices:
\{CHOICES\}
\\\\
Task:
Rank the retrieved images by how useful they are for answering the question,
considering them together with the input image.
\\\\
Usefulness is defined as how much the image provides visual or factual information
that helps determine the correct answer.
\\\\
Output format:
Output the indices of the retrieved images in descending order of usefulness,
using the format:
Image X > Image Y > Image Z > ...
\\\\
Only output the ranking. Do not provide any explanation or additional text.
Each image index from Image 1 to Image \{NUM\_IMAGES\} must appear exactly once.
\end{tcolorbox}

\begin{tcolorbox}[
  colback=white,
  colframe=black!50!white,
  boxrule=0.5pt,
  arc=3pt,
  left=2pt, right=2pt, top=2pt, bottom=2pt,
  title=\small\textbf{Listwise Ranking Prompt for Visual-RAG},
  fonttitle=\small\bfseries
]
\small
You will be given:
\begin{itemize}[leftmargin=*,itemsep=1pt]
\item A question about a visual attribute of an organism.
\item \{NUM\_IMAGES\} retrieved images, indexed from Image 1 to Image \{NUM\_IMAGES\}.
\end{itemize}

Not all retrieved images contain the information needed to answer the question.
Some images may be irrelevant or uninformative.
\\\\
Retrieved images:
Image 1, Image 2, ..., Image \{NUM\_IMAGES\}
(The order above is arbitrary.)
\\\\
Question:
\{QUESTION\}
\\\\
Task:
Rank the retrieved images by how useful they are for answering the question.
\\\\
Usefulness is defined as how much the image contains the key visual information
required to answer the question.
\\\\
Output format:
Output the indices of the retrieved images in descending order of usefulness,
using the format:
Image X > Image Y > Image Z > ...
\\\\
Only output the ranking. Do not provide any explanation or additional text.
Each image index from Image 1 to Image \{NUM\_IMAGES\} must appear exactly once.
\end{tcolorbox}

\noindent\textbf{Note.}
The model output is parsed into an explicit ranking over candidate images. For evaluation, we select the top-$K$ images according to the predicted ranking and pass them to the main generation model.

\paragraph{Prompt Sensitivity Templates.}
To assess the robustness of surrogate-based helpfulness estimation to prompt realization, we conduct a controlled prompt sensitivity analysis using three template variants:

\begin{description}[leftmargin=1.5em,style=nextline]
    \item[V1: Lexical substitution.]
    We minimally modify the original template by replacing the keyword \textit{helpful} with \textit{useful}, isolating the effect of a small lexical change.
    \item[V2: Paraphrasing.]
    We construct a semantically equivalent rephrased version that changes the sentence form while preserving the original intent. This variant tests robustness to syntactic and stylistic variation.
    \item[V3: Output-label variation.]
    We modify the binary output label space used for the decision step. On MRAG-Bench, we replace \texttt{True/False} with \texttt{Yes/No}; on Visual-RAG, we replace \texttt{True/False} with \texttt{Helpful/Not\ helpful}. This variant directly tests sensitivity to output token choice and the resulting logit differences.
\end{description}
The corresponding prompt templates are listed below.

\begin{tcolorbox}[
  colback=white,
  colframe=black!50!white,
  boxrule=0.5pt,
  arc=3pt,
  left=2pt, right=2pt, top=2pt, bottom=2pt,
  title=\small\textbf{Lexical substitution (V1) for MRAG-Bench},
  fonttitle=\small\bfseries
]
\small

You are shown two images and a multiple-choice question.

\begin{itemize}[leftmargin=*,itemsep=1pt]
\item Image A: the original image referenced by the question.
\item Image B: an image retrieved as supporting evidence.
\end{itemize}

The retrieved image may or may not add information that changes the answer.

Question:
\{question\_text\}

Options:
\{choices\}

Does Image B supply additional visual or factual cues that would help resolve the question correctly?

Respond with True or False.

\end{tcolorbox}

\begin{tcolorbox}[
  colback=white,
  colframe=black!50!white,
  boxrule=0.5pt,
  arc=3pt,
  left=2pt, right=2pt, top=2pt, bottom=2pt,
  title=\small\textbf{Paraphrasing (V2) for MRAG-Bench},
  fonttitle=\small\bfseries
]
\small

You will be given two images and a multiple-choice question.
\begin{itemize}[leftmargin=*,itemsep=1pt]
\item The first image is the input image that the question is about.
\item The second image is a retrieved image intended to provide additional visual evidence.
\end{itemize}

The retrieved image does not need to answer the question by itself.
\\
It is only meant to help answer the question together with the input image.

Question:
\{question\_text\}

Choices:
\{choices\}

Based on the images provided, does the retrieved image provide useful visual or factual information that could assist in answering the question correctly?

Answer with True or False.

\end{tcolorbox}

\begin{tcolorbox}[
  colback=white,
  colframe=black!50!white,
  boxrule=0.5pt,
  arc=3pt,
  left=2pt, right=2pt, top=2pt, bottom=2pt,
  title=\small\textbf{Output label variation (V3) for MRAG-Bench},
  fonttitle=\small\bfseries
]
\small

You are shown two images and a multiple-choice question.
\begin{itemize}[leftmargin=*,itemsep=1pt]
\item Image A: the original image referenced by the question.
\item Image B: an image retrieved as supporting evidence.
\end{itemize}

The retrieved image may or may not add information that changes the answer.
\\
Question:
\{question\_text\}
\\
Options:
\{choices\}
\\
Does the retrieved image provide useful evidence to aid answering the question?
\\
Answer with "Yes" or "No".

\end{tcolorbox}

\begin{tcolorbox}[
  colback=white,
  colframe=black!50!white,
  boxrule=0.5pt,
  arc=3pt,
  left=2pt, right=2pt, top=2pt, bottom=2pt,
  title=\small\textbf{Lexical substitution (V1) for Visual-RAG},
  fonttitle=\small\bfseries
]
\small

You are shown one image and a question about a visual trait of an organism.

The image is a retrieved candidate; it might or might not show the necessary detail.

Question:
\{question\_text\}

Does the provided image include the critical visual detail required to answer the question?

Respond with True or False.

\end{tcolorbox}

\begin{tcolorbox}[
  colback=white,
  colframe=black!50!white,
  boxrule=0.5pt,
  arc=3pt,
  left=2pt, right=2pt, top=2pt, bottom=2pt,
  title=\small\textbf{Paraphrasing (V2) for Visual-RAG},
  fonttitle=\small\bfseries
]
\small

You will be given one image and a question about a visual attribute of an organism.

The image is retrieved as potential visual evidence.
Not all retrieved images contain the information needed to answer the question.

Question:
\{question\_text\}

Based on the image provided, does this image contain useful visual information needed to answer the question?

Answer with True or False.

\end{tcolorbox}

\begin{tcolorbox}[
  colback=white,
  colframe=black!50!white,
  boxrule=0.5pt,
  arc=3pt,
  left=2pt, right=2pt, top=2pt, bottom=2pt,
  title=\small\textbf{Output label variation (V3) for Visual-RAG},
  fonttitle=\small\bfseries
]
\small

You are shown one image and a question about a visual trait of an organism.

The image is a retrieved candidate; it might or might not show the necessary detail.

Question:
\{question\_text\}

Please indicate whether this retrieved image is Helpful or Not helpful for answering the question.

Answer with "Helpful" or "Not helpful".

\end{tcolorbox}

\begin{table*}[]
\begin{lrbox}{\myentiretablebox}
\begin{tabular}{l|ll|llllllllll}
\hline
 &  &  & \multicolumn{5}{c|}{\textbf{MRAG-Bench}} & \multicolumn{5}{c}{\textbf{Visual-RAG}} \\ \cline{4-13} 
 &  &  & \multicolumn{10}{c}{Main Model} \\ \cline{4-13} 
 & \multirow{-3}{*}{Image Selection Method} & \multirow{-3}{*}{\# Params} & Qwen3-VL-8B & MiniCPM-V4.5 & Gemma3-12B & Ovis2.5-9B & \multicolumn{1}{l|}{InternVL3.5-8B} & Qwen3-VL-8B & MiniCPM-V4.5 & Gemma3-12B & Ovis2.5-9B & InternVL3.5-8B \\ \cline{2-13} 
\multirow{-4}{*}{Top K} & Zero-Shot &  & 59.35 & 57.95 & 56.84 & 59.05 & \multicolumn{1}{l|}{42.87} & 52.41 & 53.07 & 51.07 & 52.67 & 54.28 \\ \hline
 & CLIP-B & 151M & 57.80(-1.55) & 58.83(+0.88) & 55.95(-0.89) & 57.13(-1.92) & \multicolumn{1}{l|}{42.65(-0.22)} & 54.28(+1.87) & 58.56(+5.49) & 51.87(+0.80) & 66.04(+13.37) & \textbf{64.17(+9.89)} \\
 & CLIP-L & 428M & 57.80(-1.55) & 58.98(+1.03) & 56.25(-0.59) & 57.95(-1.10) & \multicolumn{1}{l|}{44.49(+1.62)} & 54.14(+1.73) & 58.29(+5.22) & 54.95(+3.88) & 62.03(+9.36) & 61.23(+6.95) \\
 & OpenCLIP & 151M & 58.46(-0.89) & 60.46(+2.51) & 57.50(+0.66) & 58.24(-0.81) & \multicolumn{1}{l|}{42.65(-0.22)} & 52.94(+0.53) & 58.56(+5.49) & 52.67(+1.60) & 64.97(+12.30) & 60.29(+6.01) \\
 & SigLIP 2 Base & 375M & 59.42(+0.07) & 61.27(+3.32) & 58.24(+1.40) & 58.91(-0.14) & \multicolumn{1}{l|}{44.35(+1.48)} & 52.94(+0.53) & 56.68(+3.61) & 55.48(+4.41) & 64.04(+11.37) & 61.76(+7.48) \\
 & SigLIP 2 So400m & 1.1B & 61.42(+2.07) & 61.57(+3.62) & 59.87(+3.03) & 59.35(+0.30) & \multicolumn{1}{l|}{44.49(+1.62)} & 53.34(+0.93) & 59.63(+6.56) & 53.34(+2.27) & 64.84(+12.17) & 62.57(+8.29) \\
 & BGE-VL-large & 428M & 58.98(-0.37) & 60.53(+2.58) & 55.65(-1.19) & 59.05(+0.00) & \multicolumn{1}{l|}{44.64(+1.77)} & 52.81(+0.40) & 57.35(+4.28) & 54.28(+3.21) & 65.64(+12.97) & 58.42(+4.14) \\
 & E5-V & 7.8B & 58.83(-0.52) & 59.20(+1.25) & 55.06(-1.78) & 57.95(-1.10) & \multicolumn{1}{l|}{43.46(+0.59)} & 53.88(+1.47) & 58.42(+5.35) & 55.48(+4.41) & 67.91(+15.24) & 62.57(+8.29) \\
 & GME & 2.2B & 64.38(+5.03) & 65.19(+7.24) & 59.42(+2.58) & 61.35(+2.30) & \multicolumn{1}{l|}{47.01(+4.14)} & 55.88(+3.47) & 59.09(+6.02) & 53.07(+2.00) & 67.51(+14.84) & 62.30(+8.02) \\
 & VLM2Vec & 7.7B & 61.71(+2.36) & 61.79(+3.84) & 58.76(+1.92) & 60.83(+1.78) & \multicolumn{1}{l|}{46.12(+3.25)} & 54.55(+2.14) & 57.49(+4.42) & 54.81(+3.74) & 49.33(-3.34) & 48.53(-5.75) \\
 & VLM2Vec-V2.0 & 2.2B & 60.68(+1.33) & 61.86(+3.91) & 57.50(+0.66) & 60.24(+1.19) & \multicolumn{1}{l|}{45.31(+2.44)} & 55.61(+3.20) & 55.88(+2.81) & 53.88(+2.81) & 50.80(-1.87) & 48.53(-5.75) \\
 & BGE-MLLM-S1 & 7.6B & 60.46(+1.11) & 60.90(+2.95) & 55.28(-1.56) & 58.61(-0.44) & \multicolumn{1}{l|}{45.68(+2.81)} & 56.15(+3.74) & 56.55(+3.48) & 56.55(+5.48) & 67.11(+14.44) & 63.37(+9.09) \\
 & BGE-MLLM-S2 & 7.6B & 60.31(+0.96) & 61.57(+3.62) & 57.58(+0.74) & 58.54(-0.51) & \multicolumn{1}{l|}{44.72(+1.85)} & 53.61(+1.20) & 55.35(+2.28) & 55.35(+4.28) & 65.11(+12.44) & 63.90(+9.62) \\
 & UniME-V2 & 7.1B & 62.82(+3.47) & 63.05(+5.10) & \textbf{60.90(+4.06)} & 60.61(+1.56) & \multicolumn{1}{l|}{45.60(+2.73)} & 56.95(+4.54) & 58.69(+5.62) & 55.88(+4.81) & 57.62(+4.95) & 49.33(-4.95) \\
 & LamRA & 8B & 63.34(+3.99) & 62.97(+5.02) & 58.61(+1.77) & 59.05(+0.00) & \multicolumn{1}{l|}{45.75(+2.88)} & 58.42(+6.01) & 60.16(+7.09) & 55.48(+4.41) & 58.16(+5.49) & 62.57(+8.29) \\
 & Ours (Qwen3-VL-2B Surrogate) & 2.1B & \textbf{65.56(+6.21)} & \textbf{65.41(+7.46)} & 60.83(+3.99) & \textbf{61.42(+2.37)} & \multicolumn{1}{l|}{\textbf{47.89(+5.02)}} & 59.89(+7.48) & 60.16(+7.09) & 59.22(+8.15) & 68.85(+16.18) & \textbf{64.17(+9.89)} \\
 & Ours (Ovis2.5-2B Surrogate) & 2.6B & 64.97(+5.62) & 64.08(+6.13) & 60.53(+3.69) & 60.16(+1.11) & \multicolumn{1}{l|}{47.23(+4.36)} & \textbf{61.36(+8.95)} & \textbf{61.23(+8.16)} & 57.75(+6.68) & \textbf{69.12(+16.45)} & 62.30(+8.02) \\
 & Ours (In-Family Surrogate) &  & \textbf{65.56(+6.21)} & 64.60(+6.65) & 59.72(+2.88) & \textbf{61.42(+2.37)} & \multicolumn{1}{l|}{47.08(+4.21)} & 59.89(+7.48) & 60.70(+7.63) & \textbf{59.63(+8.56)} & 68.85(+16.18) & 62.17(+7.89) \\ \cline{2-13} 
\multirow{-18}{*}{K=1} & GT (Human-Annotated Image as Input) & \textbf{} & 64.82(+5.47) & 64.15(+6.20) & 57.65(+0.81) & 62.97(+3.92) & \multicolumn{1}{l|}{46.49(+3.62)} & 60.96(+8.55) & 63.50(+10.43) & 58.69(+7.62) & 70.86(+18.19) & 64.04(+9.76) \\ \hline
 & CLIP-B & 151M & 59.35(+0.00) & 59.42(+1.47) & 56.91(+0.07) & 57.58(-1.47) & \multicolumn{1}{l|}{43.09(+0.22)} & 56.82(+4.41) & 59.09(+6.02) & 57.35(+6.28) & 58.02(+5.35) & 51.87(-2.41) \\
 & CLIP-L & 428M & 58.98(-0.37) & 60.46(+2.51) & 56.47(-0.37) & 57.21(-1.84) & \multicolumn{1}{l|}{45.08(+2.21)} & 57.62(+5.21) & 58.56(+5.49) & 57.49(+6.42) & 57.22(+4.55) & 52.27(-2.01) \\
 & OpenCLIP & 151M & 60.61(+1.26) & 61.12(+3.17) & 57.95(+1.11) & 58.46(-0.59) & \multicolumn{1}{l|}{44.49(+1.62)} & 56.95(+4.54) & 60.29(+7.22) & 56.82(+5.75) & 57.22(+4.55) & 53.21(-1.07) \\
 & SigLIP 2 Base & 375M & 61.49(+2.14) & 61.64(+3.69) & 58.68(+1.84) & 58.39(-0.66) & \multicolumn{1}{l|}{45.16(+2.29)} & 56.95(+4.54) & 56.42(+3.35) & 56.68(+5.61) & 56.95(+4.28) & 51.74(-2.54) \\
 & SigLIP 2 So400m & 1.1B & 62.97(+3.62) & 63.41(+5.46) & 60.46(+3.62) & 59.05(+0.00) & \multicolumn{1}{l|}{45.38(+2.51)} & 54.01(+1.60) & 59.76(+6.69) & 54.41(+3.34) & 58.42(+5.75) & 52.14(-2.14) \\
 & BGE-VL-large & 428M & 60.75(+1.40) & 60.90(+2.95) & 57.13(+0.29) & 59.05(+0.00) & \multicolumn{1}{l|}{44.05(+1.18)} & 57.35(+4.94) & 58.69(+5.62) & 56.68(+5.61) & 57.09(+4.42) & 52.94(-1.34) \\
 & E5-V & 7.8B & 60.90(+1.55) & 61.27(+3.32) & 58.24(+1.40) & 59.20(+0.15) & \multicolumn{1}{l|}{43.46(+0.59)} & 57.62(+5.21) & 58.02(+4.95) & 58.16(+7.09) & 59.09(+6.42) & 52.27(-2.01) \\
 & GME & 2.2B & 64.52(+5.17) & 64.30(+6.35) & 60.90(+4.06) & 61.20(+2.15) & \multicolumn{1}{l|}{46.27(+3.40)} & 59.89(+7.48) & 58.42(+5.35) & 57.09(+6.02) & 58.16(+5.49) & 52.81(-1.47) \\
 & VLM2Vec & 7.7B & 63.34(+3.99) & 63.56(+5.61) & 60.53(+3.69) & 59.72(+0.67) & \multicolumn{1}{l|}{46.27(+3.40)} & 56.82(+4.41) & 60.03(+6.96) & 54.55(+3.48) & 55.61(+2.94) & 52.27(-2.01) \\
 & VLM2Vec-V2.0 & 2.2B & 63.86(+4.51) & 64.15(+6.20) & 60.53(+3.69) & 59.50(+0.45) & \multicolumn{1}{l|}{46.12(+3.25)} & 55.61(+3.20) & 59.09(+6.02) & 54.95(+3.88) & 55.88(+3.21) & 53.48(-0.80) \\
 & BGE-MLLM-S1 & 7.6B & 62.90(+3.55) & 63.41(+5.46) & 57.95(+1.11) & 59.57(+0.52) & \multicolumn{1}{l|}{45.75(+2.88)} & 57.49(+5.08) & 58.56(+5.49) & 58.56(+7.49) & 58.96(+6.29) & 50.94(-3.34) \\
 & BGE-MLLM-S2 & 7.6B & 63.34(+3.99) & 63.27(+5.32) & 58.39(+1.55) & 58.39(-0.66) & \multicolumn{1}{l|}{44.64(+1.77)} & 58.29(+5.88) & 55.75(+2.68) & 55.75(+4.68) & 54.28(+1.61) & 50.67(-3.61) \\
 & UniME-V2 & 7.1B & 64.52(+5.17) & 63.41(+5.46) & 60.75(+3.91) & 60.38(+1.33) & \multicolumn{1}{l|}{44.42(+1.55)} & 58.69(+6.28) & 60.96(+7.89) & 56.95(+5.88) & 60.03(+7.36) & 54.81(+0.53) \\
 & LamRA & 8B & 64.82(+5.47) & 63.93(+5.98) & 59.57(+2.73) & 60.24(+1.19) & \multicolumn{1}{l|}{46.64(+3.77)} & 62.30(+9.89) & \textbf{63.50(+10.43)} & 59.22(+8.15) & 59.76(+7.09) & \textbf{56.02(+1.74)} \\
 & Ours (Qwen3-VL-2B Surrogate) & 2.1B & \textbf{65.71(+6.36)} & \textbf{66.74(+8.79)} & \textbf{64.15(+7.31)} & \textbf{62.68(+3.63)} & \multicolumn{1}{l|}{\textbf{48.26(+5.39)}} & \textbf{62.57(+10.16)} & 61.50(+8.43) & 60.16(+9.09) & 60.96(+8.29) & 54.95(+0.67) \\
 & Ours (Ovis2.5-2B Surrogate) & 2.6B & \textbf{65.71(+6.36)} & 65.34(+7.39) & 61.57(+4.73) & 61.20(+2.15) & \multicolumn{1}{l|}{47.52(+4.65)} & \textbf{62.57(+10.16)} & 61.36(+8.29) & 60.03(+8.96) & \textbf{61.76(+9.09)} & 54.95(+0.67) \\
 & Ours (In-Family Surrogate) &  & \textbf{65.71(+6.36)} & 65.56(+7.61) & 61.71(+4.87) & 61.20(+2.15) & \multicolumn{1}{l|}{46.56(+3.69)} & \textbf{62.57(+10.16)} & 60.03(+6.96) & \textbf{60.43(+9.36)} & \textbf{61.76(+9.09)} & 55.61(+1.33) \\ \cline{2-13} 
\multirow{-18}{*}{K=2} & GT (Human-Annotated Image as Input) &  & 67.11(+7.76) & 66.44(+8.49) & 59.42(+2.58) & 62.90(+3.85) & \multicolumn{1}{l|}{48.34(+5.47)} & 63.77(+11.36) & 64.17(+11.10) & 59.36(+8.29) & 64.84(+12.17) & 59.09(+4.81) \\ \hline
 & CLIP-B & 151M & 61.42(+2.07) & 60.83(+2.88) & 57.80(+0.96) & 56.84(-2.21) & \multicolumn{1}{l|}{43.61(+0.74)} & 58.69(+6.28) & 61.63(+8.56) & 57.35(+6.28) & 58.42(+5.75) & 51.07(-3.21) \\
 & CLIP-L & 428M & 59.94(+0.59) & 61.86(+3.91) & 57.72(+0.88) & 57.58(-1.47) & \multicolumn{1}{l|}{44.27(+1.40)} & 59.49(+7.08) & 62.57(+9.50) & 56.82(+5.75) & 58.56(+5.89) & 51.74(-2.54) \\
 & OpenCLIP & 151M & 62.23(+2.88) & 62.97(+5.02) & 58.68(+1.84) & 59.65(+0.60) & \multicolumn{1}{l|}{43.83(+0.96)} & 60.83(+8.42) & 61.36(+8.29) & 58.02(+6.95) & 57.75(+5.08) & 53.61(-0.67) \\
 & SigLIP 2 Base & 375M & 62.90(+3.55) & 62.23(+4.28) & 60.98(+4.14) & 58.39(-0.66) & \multicolumn{1}{l|}{44.86(+1.99)} & 59.36(+6.95) & 58.69(+5.62) & 57.89(+6.82) & 58.42(+5.75) & 54.81(+0.53) \\
 & SigLIP 2 So400m & 1.1B & 63.27(+3.92) & 63.56(+5.61) & 60.83(+3.99) & 59.87(+0.82) & \multicolumn{1}{l|}{45.53(+2.66)} & 56.68(+4.27) & \textbf{63.64(+10.57)} & 57.62(+6.55) & 58.29(+5.62) & 51.34(-2.94) \\
 & BGE-VL-large & 428M & 62.75(+3.40) & 61.79(+3.84) & 58.98(+2.14) & 59.65(+0.60) & \multicolumn{1}{l|}{44.72(+1.85)} & 57.75(+5.34) & 59.63(+6.56) & 55.35(+4.28) & 56.82(+4.15) & 51.60(-2.68) \\
 & E5-V & 7.8B & 62.53(+3.18) & 62.53(+4.58) & 59.13(+2.29) & 58.98(-0.07) & \multicolumn{1}{l|}{43.16(+0.29)} & 58.96(+6.55) & 59.76(+6.69) & 58.29(+7.22) & 60.29(+7.62) & 52.14(-2.14) \\
 & GME & 2.2B & 65.41(+6.06) & 65.48(+7.53) & 61.79(+4.95) & 61.05(+2.00) & \multicolumn{1}{l|}{46.93(+4.06)} & 61.50(+9.09) & 60.29(+7.22) & 56.02(+4.95) & 53.88(+1.21) & 50.53(-3.75) \\
 & VLM2Vec & 7.7B & 64.08(+4.73) & 63.71(+5.76) & 60.90(+4.06) & 59.87(+0.82) & \multicolumn{1}{l|}{45.53(+2.66)} & 60.03(+7.62) & 60.29(+7.22) & 57.35(+6.28) & 56.42(+3.75) & 53.07(-1.21) \\
 & VLM2Vec-V2.0 & 2.2B & 64.82(+5.47) & 64.75(+6.80) & 61.12(+4.28) & 61.05(+2.00) & \multicolumn{1}{l|}{45.08(+2.21)} & 58.69(+6.28) & 60.83(+7.76) & 56.68(+5.61) & 56.95(+4.28) & 52.41(-1.87) \\
 & BGE-MLLM-S1 & 7.6B & 62.82(+3.47) & 63.78(+5.83) & 59.65(+2.81) & 60.09(+1.04) & \multicolumn{1}{l|}{45.90(+3.03)} & 59.89(+7.48) & 58.69(+5.62) & 58.69(+7.62) & 58.29(+5.62) & 52.01(-2.27) \\
 & BGE-MLLM-S2 & 7.6B & 64.52(+5.17) & 64.52(+6.57) & 60.68(+3.84) & 60.61(+1.56) & \multicolumn{1}{l|}{45.01(+2.14)} & 59.09(+6.68) & 56.02(+2.95) & 56.02(+4.95) & 57.75(+5.08) & 50.27(-4.01) \\
 & UniME-V2 & 7.1B & 65.19(+5.84) & 64.75(+6.80) & 62.08(+5.24) & 61.57(+2.52) & \multicolumn{1}{l|}{45.97(+3.10)} & 59.22(+6.81) & 61.10(+8.03) & 58.02(+6.95) & 58.29(+5.62) & 53.34(-0.94) \\
 & LamRA & 8B & 65.85(+6.50) & 64.89(+6.94) & 60.53(+3.69) & 61.27(+2.22) & \multicolumn{1}{l|}{46.27(+3.40)} & 61.23(+8.82) & 63.10(+10.03) & 58.29(+7.22) & 56.95(+4.28) & 53.88(-0.40) \\
 & Ours (Qwen3-VL-2B Surrogate) & 2.1B & \textbf{67.55(+8.20)} & \textbf{65.85(+7.90)} & \textbf{64.52(+7.68)} & \textbf{62.31(+3.26)} & \multicolumn{1}{l|}{47.08(+4.21)} & 63.77(+11.36) & 59.89(+6.82) & \textbf{60.43(+9.36)} & \textbf{63.24(+10.57)} & \textbf{55.75(+1.47)} \\
 & Ours (Ovis2.5-2B Surrogate) & 2.6B & 66.08(+6.73) & \textbf{65.85(+7.90)} & 62.23(+5.39) & 62.01(+2.96) & \multicolumn{1}{l|}{\textbf{47.15(+4.28)}} & \textbf{64.44(+12.03)} & 60.43(+7.36) & 57.75(+6.68) & 60.70(+8.03) & 53.88(-0.40) \\
 & Ours (In-Family Surrogate) &  & \textbf{67.55(+8.20)} & 65.63(+7.68) & 62.82(+5.98) & \textbf{62.31(+3.26)} & \multicolumn{1}{l|}{\textbf{47.15(+4.28)}} & 63.77(+11.36) & 61.50(+8.43) & 57.89(+6.82) & \textbf{63.24(+10.57)} & 55.08(+0.80) \\ \cline{2-13} 
\multirow{-18}{*}{K=3} & GT (Human-Annotated Image as Input) & \textbf{} & 69.03(+9.68) & 66.89(+8.94) & 58.39(+1.55) & 64.01(+4.96) & \multicolumn{1}{l|}{48.78(+5.91)} & 64.71(+12.30) & 62.17(+9.10) & 59.63(+8.56) & 64.71(+12.04) & 57.89(+3.61) \\ \hline
 & CLIP-B & 151M & 61.20(+1.85) & 60.46(+2.51) & 58.46(+1.62) & 58.31(-0.74) & \multicolumn{1}{l|}{43.75(+0.88)} & 63.64(+11.23) & 61.63(+8.56) & 58.82(+7.75) & 60.43(+7.76) & 53.48(-0.80) \\
 & CLIP-L & 428M & 59.87(+0.52) & 62.53(+4.58) & 57.35(+0.51) & 58.61(-0.44) & \multicolumn{1}{l|}{42.87(+0.00)} & 60.03(+7.62) & 62.03(+8.96) & 55.48(+4.41) & 58.96(+6.29) & 52.94(-1.34) \\
 & OpenCLIP & 151M & 63.05(+3.70) & 63.05(+5.10) & 59.20(+2.36) & 59.72(+0.67) & \multicolumn{1}{l|}{43.98(+1.11)} & 60.56(+8.15) & 61.76(+8.69) & 58.02(+6.95) & 60.03(+7.36) & 54.55(+0.27) \\
 & SigLIP 2 Base & 375M & 63.34(+3.99) & 63.34(+5.39) & 62.68(+5.84) & 59.79(+0.74) & \multicolumn{1}{l|}{44.49(+1.62)} & 61.76(+9.35) & 58.16(+5.09) & 58.96(+7.89) & 57.62(+4.95) & 54.01(-0.27) \\
 & SigLIP 2 So400m & 1.1B & 64.82(+5.47) & 63.12(+5.17) & 61.71(+4.87) & 60.61(+1.56) & \multicolumn{1}{l|}{45.16(+2.29)} & 60.56(+8.15) & 62.97(+9.90) & 55.35(+4.28) & 60.03(+7.36) & 54.14(-0.14) \\
 & BGE-VL-large & 428M & 62.90(+3.55) & 63.78(+5.83) & 59.87(+3.03) & 59.50(+0.45) & \multicolumn{1}{l|}{44.86(+1.99)} & 58.42(+6.01) & 62.03(+8.96) & 56.95(+5.88) & 56.42(+3.75) & 51.87(-2.41) \\
 & E5-V & 7.8B & 62.68(+3.33) & 63.05(+5.10) & 59.28(+2.44) & 59.79(+0.74) & \multicolumn{1}{l|}{44.49(+1.62)} & 60.16(+7.75) & 61.10(+8.03) & 60.43(+9.36) & 59.36(+6.69) & 54.68(+0.40) \\
 & GME & 2.2B & 66.22(+6.87) & 66.08(+8.13) & 62.82(+5.98) & 62.53(+3.48) & \multicolumn{1}{l|}{45.38(+2.51)} & 63.64(+11.23) & 59.36(+6.29) & 57.35(+6.28) & 55.88(+3.21) & 53.88(-0.40) \\
 & VLM2Vec & 7.7B & 64.89(+5.54) & 64.75(+6.80) & 61.57(+4.73) & 60.83(+1.78) & \multicolumn{1}{l|}{45.01(+2.14)} & 61.76(+9.35) & 62.97(+9.90) & 58.42(+7.35) & 60.29(+7.62) & 53.07(-1.21) \\
 & VLM2Vec-V2.0 & 2.2B & 65.56(+6.21) & 65.11(+7.16) & 62.31(+5.47) & 62.08(+3.03) & \multicolumn{1}{l|}{44.35(+1.48)} & 61.90(+9.49) & 60.56(+7.49) & 59.22(+8.15) & 57.22(+4.55) & 53.07(-1.21) \\
 & BGE-MLLM-S1 & 7.6B & 62.97(+3.62) & 64.23(+6.28) & 58.98(+2.14) & 60.61(+1.56) & \multicolumn{1}{l|}{44.05(+1.18)} & 62.03(+9.62) & 59.63(+6.56) & 59.63(+8.56) & 60.83(+8.16) & 53.21(-1.07) \\
 & BGE-MLLM-S2 & 7.6B & 63.93(+4.58) & 64.08(+6.13) & 60.31(+3.47) & 60.53(+1.48) & \multicolumn{1}{l|}{45.23(+2.36)} & 58.96(+6.55) & 56.55(+3.48) & 56.55(+5.48) & 58.69(+6.02) & 51.60(-2.68) \\
 & UniME-V2 & 7.1B & 66.15(+6.80) & 65.34(+7.39) & 62.45(+5.61) & 62.31(+3.26) & \multicolumn{1}{l|}{44.72(+1.85)} & 60.70(+8.29) & 62.97(+9.90) & 60.56(+9.49) & 58.96(+6.29) & 54.68(+0.40) \\
 & LamRA & 8B & 66.44(+7.09) & \textbf{66.37(+8.42)} & 60.75(+3.91) & 62.53(+3.48) & \multicolumn{1}{l|}{\textbf{47.15(+4.28)}} & 62.17(+9.76) & 63.64(+10.57) & 59.22(+8.15) & 58.96(+6.29) & 54.81(+0.53) \\
 & Ours (Qwen3-VL-2B Surrogate) & 2.1B & \textbf{67.04(+7.69)} & 65.71(+7.76) & \textbf{64.82(+7.98)} & \textbf{62.97(+3.92)} & \multicolumn{1}{l|}{46.78(+3.91)} & 64.30(+11.89) & 62.70(+9.63) & 59.63(+8.56) & \textbf{62.97(+10.30)} & 56.68(+2.40) \\
 & Ours (Ovis2.5-2B Surrogate) & 2.6B & 65.93(+6.58) & 66.00(+8.05) & 61.79(+4.95) & 61.79(+2.74) & \multicolumn{1}{l|}{46.64(+3.77)} & \textbf{65.78(+13.37)} & \textbf{65.11(+12.04)} & \textbf{60.83(+9.76)} & 61.23(+8.56) & 54.95(+0.67) \\
 & Ours (In-Family Surrogate) &  & \textbf{67.04(+7.69)} & 65.85(+7.90) & 62.53(+5.69) & 61.79(+2.74) & \multicolumn{1}{l|}{46.93(+4.06)} & 64.30(+11.89) & 62.97(+9.90) & 57.35(+6.28) & 61.23(+8.56) & \textbf{56.82(+2.54)} \\ \cline{2-13} 
\multirow{-18}{*}{K=4} & GT (Human-Annotated Image as Input) &  & 69.18(+9.83) & 68.14(+10.19) & 59.50(+2.66) & 64.67(+5.62) & \multicolumn{1}{l|}{48.26(+5.39)} & 66.58(+14.17) & 63.37(+10.30) & 62.83(+11.76) & 65.24(+12.57) & 57.49(+3.21) \\ \hline
 & CLIP-B & 151M & 61.79(+2.44) & 62.08(+4.13) & 59.79(+2.95) & 58.91(-0.14) & \multicolumn{1}{l|}{43.75(+0.88)} & 64.57(+12.16) & 62.03(+8.96) & 59.76(+8.69) & 60.03(+7.36) & 55.48(+1.20) \\
 & CLIP-L & 428M & 60.61(+1.26) & 60.98(+3.03) & 58.24(+1.40) & 57.95(-1.10) & \multicolumn{1}{l|}{42.65(-0.22)} & 61.76(+9.35) & 62.03(+8.96) & 59.09(+8.02) & 58.42(+5.75) & 52.81(-1.47) \\
 & OpenCLIP & 151M & 63.56(+4.21) & 62.08(+4.13) & 61.27(+4.43) & 59.79(+0.74) & \multicolumn{1}{l|}{45.60(+2.73)} & 62.17(+9.76) & 62.57(+9.50) & 59.22(+8.15) & 59.49(+6.82) & 55.35(+1.07) \\
 & SigLIP 2 Base & 375M & 64.30(+4.95) & 63.05(+5.10) & 61.94(+5.10) & 60.61(+1.56) & \multicolumn{1}{l|}{45.75(+2.88)} & 62.57(+10.16) & 60.96(+7.89) & 59.89(+8.82) & 59.49(+6.82) & 54.95(+0.67) \\
 & SigLIP 2 So400m & 1.1B & 65.19(+5.84) & 63.34(+5.39) & 62.60(+5.76) & 60.09(+1.04) & \multicolumn{1}{l|}{45.08(+2.21)} & 60.56(+8.15) & 61.36(+8.29) & 54.95(+3.88) & 59.22(+6.55) & 55.21(+0.93) \\
 & BGE-VL-large & 428M & 64.67(+5.32) & 63.41(+5.46) & 61.27(+4.43) & 59.94(+0.89) & \multicolumn{1}{l|}{45.31(+2.44)} & 58.16(+5.75) & 61.23(+8.16) & 56.68(+5.61) & 56.82(+4.15) & 51.74(-2.54) \\
 & E5-V & 7.8B & 63.93(+4.58) & 62.38(+4.43) & 60.38(+3.54) & 59.94(+0.89) & \multicolumn{1}{l|}{43.61(+0.74)} & 60.16(+7.75) & 61.63(+8.56) & 60.96(+9.89) & 57.89(+5.22) & 55.35(+1.07) \\
 & GME & 2.2B & 67.04(+7.69) & 66.30(+8.35) & 61.79(+4.95) & 61.94(+2.89) & \multicolumn{1}{l|}{46.78(+3.91)} & 65.78(+13.37) & 62.97(+9.90) & 59.36(+8.29) & 57.22(+4.55) & 54.95(+0.67) \\
 & VLM2Vec & 7.7B & 66.00(+6.65) & 64.45(+6.50) & 61.71(+4.87) & 61.27(+2.22) & \multicolumn{1}{l|}{46.56(+3.69)} & 62.43(+10.02) & 61.50(+8.43) & 58.56(+7.49) & 60.16(+7.49) & 54.01(-0.27) \\
 & VLM2Vec-V2.0 & 2.2B & 66.22(+6.87) & 65.11(+7.16) & 62.45(+5.61) & 61.86(+2.81) & \multicolumn{1}{l|}{45.90(+3.03)} & 61.63(+9.22) & 59.09(+6.02) & 58.69(+7.62) & 55.21(+2.54) & 51.20(-3.08) \\
 & BGE-MLLM-S1 & 7.6B & 64.89(+5.54) & 63.49(+5.54) & 61.20(+4.36) & 60.31(+1.26) & \multicolumn{1}{l|}{44.64(+1.77)} & 62.97(+10.56) & 57.75(+4.68) & 57.75(+6.68) & 59.22(+6.55) & 52.41(-1.87) \\
 & BGE-MLLM-S2 & 7.6B & 65.11(+5.76) & 65.04(+7.09) & 61.49(+4.65) & 61.20(+2.15) & \multicolumn{1}{l|}{45.60(+2.73)} & 59.36(+6.95) & 57.35(+4.28) & 57.35(+6.28) & 59.89(+7.22) & 54.14(-0.14) \\
 & UniME-V2 & 7.1B & 66.30(+6.95) & 65.34(+7.39) & 62.60(+5.76) & 62.82(+3.77) & \multicolumn{1}{l|}{45.31(+2.44)} & 61.76(+9.35) & 63.77(+10.70) & 60.43(+9.36) & 58.02(+5.35) & 56.15(+1.87) \\
 & LamRA & 8B & 66.44(+7.09) & 66.67(+8.72) & 61.71(+4.87) & 63.05(+4.00) & \multicolumn{1}{l|}{46.05(+3.18)} & 64.04(+11.63) & \textbf{64.97(+11.90)} & 59.89(+8.82) & 58.82(+6.15) & 54.95(+0.67) \\
 & Ours (Qwen3-VL-2B Surrogate) & 2.1B & \textbf{67.55(+8.20)} & 66.37(+8.42) & \textbf{63.93(+7.09)} & \textbf{63.41(+4.36)} & \multicolumn{1}{l|}{\textbf{46.86(+3.99)}} & 64.84(+12.43) & 63.10(+10.03) & 60.83(+9.76) & \textbf{62.17(+9.50)} & 55.21(+0.93) \\
 & Ours (Ovis2.5-2B Surrogate) & 2.6B & 65.71(+6.36) & 64.67(+6.72) & 61.71(+4.87) & 62.38(+3.33) & \multicolumn{1}{l|}{46.71(+3.84)} & \textbf{66.18(+13.77)} & 63.50(+10.43) & \textbf{61.50(+10.43)} & 60.70(+8.03) & \textbf{56.28(+2.00)} \\
 & Ours (In-Family Surrogate) &  & \textbf{67.55(+8.20)} & \textbf{66.96(+9.01)} & 62.31(+5.47) & \textbf{63.41(+4.36)} & \multicolumn{1}{l|}{46.49(+3.62)} & 64.84(+12.43) & 63.50(+10.43) & 57.49(+6.42) & \textbf{62.17(+9.50)} & 55.61(+1.33) \\ \cline{2-13} 
\multirow{-18}{*}{K=5} & GT (Human-Annotated Image as Input) & \textbf{} & 69.99(+10.64) & 68.00(+10.05) & 59.42(+2.58) & 64.89(+5.84) & \multicolumn{1}{l|}{47.60(+4.73)} & 66.84(+14.43) & 64.17(+11.10) & 61.10(+10.03) & 64.71(+12.04) & 58.82(+4.54) \\ \hline
\end{tabular}
\end{lrbox}

\resizebox{\linewidth}{!}{\usebox{\myentiretablebox}}
\caption{Full version of Table~\ref{tab:exp-main-results}.Main results on MRAG-Bench and Visual-RAG under different visual evidence selection methods and Top-$K$ settings. Numbers in parentheses denote absolute performance differences relative to the zero-shot baseline for the same model. \textbf{Bold} values indicate the best-performing retrieval-based method in each setting (excluding the ground-truth (GT) oracle).}
\label{tab:appendix_main_results}
\end{table*}

\begin{table*}[]
\begin{lrbox}{\myentiretablebox}
\begin{tabular}{@{}llllllll@{}}
\toprule
\multirow{2}{*}{Dataset} & \multirow{2}{*}{Method} & \multicolumn{1}{l|}{\multirow{2}{*}{Model}} & \multicolumn{5}{c}{Top K} \\ \cmidrule(l){4-8} 
 &  & \multicolumn{1}{l|}{} & 1 & 2 & 3 & 4 & 5 \\ \midrule
\multirow{10}{*}{MRAG-Bench} & \multicolumn{1}{l|}{\multirow{5}{*}{Ours}} & Qwen3-VL-8B & 65.71 & 65.34 & 67.41 & 66.74 & 67.11 \\
 & \multicolumn{1}{l|}{} & MiniCPM-V4.5 & 65.11 & 66.52 & 66.89 & 66.08 & 65.71 \\
 & \multicolumn{1}{l|}{} & Gemma3-12B & 59.87 & 62.6 & 61.94 & 62.68 & 63.12 \\
 & \multicolumn{1}{l|}{} & Ovis2.5-9B & 61.64 & 62.53 & 62.53 & 62.45 & 63.19 \\
 & \multicolumn{1}{l|}{} & InternVL3.5-8B & 48.41 & 48.04 & 47.75 & 46.49 & 46.56 \\ \cmidrule(l){2-8} 
 & \multicolumn{1}{l|}{\multirow{5}{*}{A (Choices Softmax Entropy)}} & Qwen3-VL-8B & 63.27(-2.44) & 64.08(-1.26) & 64.67(-2.74) & 65.11(-1.63) & 65.04(-2.07) \\
 & \multicolumn{1}{l|}{} & MiniCPM-V4.5 & 62.90(-2.21) & 63.71(-2.81) & 64.15(-2.74) & 64.45(-1.63) & 64.15(-1.56) \\
 & \multicolumn{1}{l|}{} & Gemma3-12B & 59.42(-0.45) & 59.72(-2.88) & 61.57(-0.37) & 61.20(-1.48) & 61.64(-1.48) \\
 & \multicolumn{1}{l|}{} & Ovis2.5-9B & 60.38(-1.26) & 60.68(-1.85) & 61.20(-1.33) & 61.57(-0.88) & 61.27(-1.92) \\
 & \multicolumn{1}{l|}{} & InternVL3.5-8B & 46.42(-1.99) & 46.64(-1.40) & 47.08(-0.67) & 46.42(-0.07) & 45.53(-1.03) \\ \midrule
\multirow{15}{*}{Visual-RAG} & \multicolumn{1}{l|}{\multirow{5}{*}{Ours}} & Qwen3-VL-8B & 62.43 & 64.84 & 63.37 & 64.44 & 64.57 \\
 & \multicolumn{1}{l|}{} & MiniCPM-V4.5 & 59.63 & 58.16 & 61.23 & 61.23 & 62.03 \\
 & \multicolumn{1}{l|}{} & Gemma3-12B & 56.55 & 59.76 & 60.43 & 58.69 & 60.03 \\
 & \multicolumn{1}{l|}{} & Ovis2.5-9B & 70.05 & 59.63 & 60.96 & 62.03 & 61.23 \\
 & \multicolumn{1}{l|}{} & InternVL3.5-8B & 61.23 & 56.95 & 58.16 & 57.09 & 57.22 \\ \cmidrule(l){2-8} 
 & \multicolumn{1}{l|}{\multirow{5}{*}{A (Avg. Token Prob)}} & Qwen3-VL-8B & 57.49(-4.94) & 58.82(-6.02) & 59.76(-3.61) & 62.70(-1.74) & 62.97(-1.60) \\
 & \multicolumn{1}{l|}{} & MiniCPM-V4.5 & 61.50(+1.87) & 61.63(+3.47) & 61.10(-0.13) & 62.17(+0.94) & 62.57(+0.54) \\
 & \multicolumn{1}{l|}{} & Gemma3-12B & 58.29(+1.74) & 55.48(-4.28) & 55.61(-4.82) & 56.42(-2.27) & 57.09(-2.94) \\
 & \multicolumn{1}{l|}{} & Ovis2.5-9B & 54.81(-15.24) & 53.48(-6.15) & 54.41(-6.55) & 55.48(-6.55) & 56.82(-4.41) \\
 & \multicolumn{1}{l|}{} & InternVL3.5-8B & 52.94(-8.29) & 53.88(-3.07) & 54.68(-3.48) & 57.22(+0.13) & 55.88(-1.34) \\ \cmidrule(l){2-8} 
 & \multicolumn{1}{l|}{\multirow{5}{*}{A (MC sampling + NLI)}} & Qwen3-VL-8B & 57.89(-4.54) & 61.50(-3.34) & 60.29(-3.08) & 62.43(-2.01) & 63.90(-0.67) \\
 & \multicolumn{1}{l|}{} & MiniCPM-V4.5 & 61.10(+1.47) & 63.10(+4.94) & 63.24(+2.01) & 63.37(+2.14) & 64.71(+2.68) \\
 & \multicolumn{1}{l|}{} & Gemma3-12B & 55.88(-0.67) & 53.88(-5.88) & 57.75(-2.68) & 58.02(-0.67) & 59.76(-0.27) \\
 & \multicolumn{1}{l|}{} & Ovis2.5-9B & 54.95(-15.10) & 56.42(-3.21) & 55.08(-5.88) & 55.08(-6.95) & 57.62(-3.61) \\
 & \multicolumn{1}{l|}{} & InternVL3.5-8B & 53.07(-8.16) & 57.49(+0.54) & 57.22(-0.94) & 56.82(-0.27) & 57.49(+0.27) \\ \bottomrule
\end{tabular}
\end{lrbox}

\resizebox{\linewidth}{!}{\usebox{\myentiretablebox}}

\caption{We compare our latent-variable-based utility estimation with answer-level uncertainty quantification methods on MRAG-Bench and Visual-RAG. Answer uncertainty is measured by choice softmax entropy on MRAG-Bench, and by average token probability or MC sampling with NLI-based consistency on Visual-RAG. Parentheses denote performance differences relative to the zero-shot baseline.}

\label{tab:appendix_Y_vs_A}

\end{table*}

\begin{table*}[]
\begin{lrbox}{\myentiretablebox}
\begin{tabular}{@{}l|l|lllll|lllll@{}}
\toprule
 & \multicolumn{1}{c|}{\textbf{}} & \multicolumn{5}{c|}{\textbf{MRAG-Bench}} & \multicolumn{5}{c}{\textbf{Visual-RAG}} \\ \midrule
\multirow{2}{*}{Model} & \multirow{2}{*}{Image Selection Method} & \multicolumn{5}{c|}{Top K} & \multicolumn{5}{c}{Top K} \\ \cmidrule(l){3-12} 
 &  & \multicolumn{1}{c}{1} & \multicolumn{1}{c}{2} & \multicolumn{1}{c}{3} & \multicolumn{1}{c}{4} & \multicolumn{1}{c|}{5} & \multicolumn{1}{c}{1} & \multicolumn{1}{c}{2} & \multicolumn{1}{c}{3} & \multicolumn{1}{c}{4} & \multicolumn{1}{c}{5} \\ \midrule
\multirow{6}{*}{Qwen3-VL-8B} & Qwen3-VL-2B & 65.56(-0.15) & 65.71(+0.37) & 67.55(+0.14) & 67.04(+0.30) & 67.55(+0.44) & 59.89(-2.54) & 62.57(-2.27) & 63.77(+0.40) & 64.30(-0.14) & 64.84(+0.27) \\
 & Ovis2.5-2B & 64.97(-0.74) & 65.71(+0.37) & 66.08(-1.33) & 65.93(-0.81) & 65.71(-1.40) & 61.36(-1.07) & 62.57(-2.27) & 64.44(+1.07) & 65.78(+1.34) & 66.18(+1.61) \\
 & MiniCPM-V4.5-AWQ & 64.23(-1.48) & 66.89(+1.55) & 66.96(-0.45) & 66.67(-0.07) & 66.74(-0.37) & 60.16(-2.27) & 64.30(-0.54) & 63.64(+0.27) & 64.17(-0.27) & 66.31(+1.74) \\
 & Gemma3-4B & 64.15(-1.56) & 64.89(-0.45) & 65.78(-1.63) & 66.08(-0.66) & 66.15(-0.96) & 59.89(-2.54) & 61.76(-3.08) & 62.30(-1.07) & 63.37(-1.07) & 63.24(-1.33) \\
 & InternVL3.5-2B & 63.64(-2.07) & 64.23(-1.11) & 65.48(-1.93) & 66.59(-0.15) & 66.81(-0.30) & 57.62(-4.81) & 62.57(-2.27) & 62.17(-1.20) & 64.04(-0.40) & 62.83(-1.74) \\
 & Qwen3-VL-8B & 65.71 & 65.34 & 67.41 & 66.74 & 67.11 & 62.43 & 64.84 & 63.37 & 64.44 & 64.57 \\ \midrule
\multirow{6}{*}{MiniCPM-V4.5} & Qwen3-VL-2B & 65.41(+0.30) & 66.74(+0.22) & 65.85(-1.04) & 65.71(-0.37) & 66.37(+0.66) & 60.16(+0.53) & 61.50(+3.34) & 59.89(-1.34) & 62.70(+1.47) & 63.10(+1.07) \\
 & Ovis2.5-2B & 64.08(-1.03) & 65.34(-1.18) & 65.85(-1.04) & 66.00(-0.08) & 64.67(-1.04) & 61.23(+1.60) & 61.36(+3.20) & 60.43(-0.80) & 65.11(+3.88) & 63.50(+1.47) \\
 & MiniCPM-V4.5-AWQ & 64.60(-0.51) & 65.56(-0.96) & 65.63(-1.26) & 65.85(-0.23) & 66.96(+1.25) & 60.70(+1.07) & 60.03(+1.87) & 61.50(+0.27) & 62.97(+1.74) & 63.50(+1.47) \\
 & Gemma3-4B & 62.75(-2.36) & 64.38(-2.14) & 64.45(-2.44) & 65.34(-0.74) & 65.34(-0.37) & 61.50(+1.87) & 62.57(+4.41) & 63.37(+2.14) & 62.17(+0.94) & 63.10(+1.07) \\
 & InternVL3.5-2B & 62.60(-2.51) & 63.27(-3.25) & 64.15(-2.74) & 64.97(-1.11) & 64.45(-1.26) & 58.56(-1.07) & 61.23(+3.07) & 63.37(+2.14) & 63.77(+2.54) & 65.24(+3.21) \\
 & MiniCPM-V4.5 & 65.11 & 66.52 & 66.89 & 66.08 & 65.71 & 59.63 & 58.16 & 61.23 & 61.23 & 62.03 \\ \midrule
\multirow{6}{*}{Gemma3-12B} & Qwen3-VL-2B & 60.83(+0.96) & 64.15(+1.55) & 64.52(+2.58) & 64.82(+2.14) & 63.93(+0.81) & 59.22(+2.67) & 60.16(+0.40) & 60.43(+0.00) & 59.63(+0.94) & 60.83(+0.80) \\
 & Ovis2.5-2B & 60.53(+0.66) & 61.57(-1.03) & 62.23(+0.29) & 61.79(-0.89) & 61.71(-1.41) & 57.75(+1.20) & 60.03(+0.27) & 57.75(-2.68) & 60.83(+2.14) & 61.50(+1.47) \\
 & MiniCPM-V4.5-AWQ & 60.38(+0.51) & 63.34(+0.74) & 63.41(+1.47) & 63.56(+0.88) & 64.75(+1.63) & 58.29(+1.74) & 57.62(-2.14) & 58.42(-2.01) & 61.23(+2.54) & 60.29(+0.26) \\
 & Gemma3-4B & 59.72(-0.15) & 61.71(-0.89) & 62.82(+0.88) & 62.53(-0.15) & 62.31(-0.81) & 59.63(+3.08) & 60.43(+0.67) & 57.89(-2.54) & 57.35(-1.34) & 57.49(-2.54) \\
 & InternVL3.5-2B & 59.87(+0.00) & 62.16(-0.44) & 63.12(+1.18) & 62.97(+0.29) & 62.82(-0.30) & 55.88(-0.67) & 58.29(-1.47) & 60.16(-0.27) & 59.89(+1.20) & 60.29(+0.26) \\
 & Gemma3-12B & 59.87 & 62.6 & 61.94 & 62.68 & 63.12 & 56.55 & 59.76 & 60.43 & 58.69 & 60.03 \\ \midrule
\multirow{6}{*}{Ovis2.5-9B} & Qwen3-VL-2B & 61.42(-0.22) & 62.68(+0.15) & 62.31(-0.22) & 62.97(+0.52) & 63.41(+0.22) & 68.85(-1.20) & 60.96(+1.33) & 63.24(+2.28) & 62.97(+0.94) & 62.17(+0.94) \\
 & Ovis2.5-2B & 60.16(-1.48) & 61.20(-1.33) & 62.01(-0.52) & 61.79(-0.66) & 62.38(-0.81) & 69.12(-0.93) & 61.76(+2.13) & 60.70(-0.26) & 61.23(-0.80) & 60.70(-0.53) \\
 & MiniCPM-V4.5-AWQ & 61.71(+0.07) & 62.53(+0.00) & 62.31(-0.22) & 62.68(+0.23) & 62.38(-0.81) & 67.25(-2.80) & 60.16(+0.53) & 60.03(-0.93) & 61.23(-0.80) & 60.29(-0.94) \\
 & Gemma3-4B & 61.71(+0.07) & 61.49(-1.04) & 61.57(-0.96) & 63.34(+0.89) & 62.90(-0.29) & 68.05(-2.00) & 59.89(+0.26) & 58.42(-2.54) & 60.29(-1.74) & 59.89(-1.34) \\
 & InternVL3.5-2B & 60.24(-1.40) & 61.49(-1.04) & 61.64(-0.89) & 62.68(+0.23) & 62.60(-0.59) & 66.98(-3.07) & 60.29(+0.66) & 62.30(+1.34) & 60.70(-1.33) & 62.30(+1.07) \\
 & Ovis2.5-9B & 61.64 & 62.53 & 62.53 & 62.45 & 63.19 & 70.05 & 59.63 & 60.96 & 62.03 & 61.23 \\ \midrule
\multirow{6}{*}{InternVL3.5-8B} & Qwen3-VL-2B & 47.89(-0.52) & 48.26(+0.22) & 47.08(-0.67) & 46.78(+0.29) & 46.86(+0.30) & 64.17(+2.94) & 54.95(-2.00) & 55.75(-2.41) & 56.68(-0.41) & 55.21(-2.01) \\
 & Ovis2.5-2B & 47.23(-1.18) & 47.52(-0.52) & 47.15(-0.60) & 46.64(+0.15) & 46.71(+0.15) & 62.30(+1.07) & 54.95(-2.00) & 53.88(-4.28) & 54.95(-2.14) & 56.28(-0.94) \\
 & MiniCPM-V4.5-AWQ & 48.41(+0.00) & 47.82(-0.22) & 47.38(-0.37) & 47.67(+1.18) & 46.42(-0.14) & 62.43(+1.20) & 56.02(-0.93) & 54.41(-3.75) & 56.28(-0.81) & 56.02(-1.20) \\
 & Gemma3-4B & 45.23(-3.18) & 46.78(-1.26) & 46.78(-0.97) & 47.38(+0.89) & 47.52(+0.96) & 61.90(+0.67) & 55.75(-1.20) & 52.14(-6.02) & 54.95(-2.14) & 54.81(-2.41) \\
 & InternVL3.5-2B & 47.08(-1.33) & 46.56(-1.48) & 47.15(-0.60) & 46.93(+0.44) & 46.49(-0.07) & 62.17(+0.94) & 55.61(-1.34) & 55.08(-3.08) & 56.82(-0.27) & 55.61(-1.61) \\
 & InternVL3.5-8B & 48.41 & 48.04 & 47.75 & 46.49 & 46.56 & 61.23 & 56.95 & 58.16 & 57.09 & 57.22 \\ \bottomrule
\end{tabular}
\end{lrbox}

\resizebox{\linewidth}{!}{\usebox{\myentiretablebox}}

\caption{We compare surrogate models with fewer parameters to their corresponding large main models under the same latent-variable-based utility estimation framework. Results are reported on MRAG-Bench and Visual-RAG across different Top-K settings, using identical evidence selection procedures.}

\label{tab:appendix_surrogate_vs_main_model}
\end{table*}

\begin{table*}[]
\begin{lrbox}{\myentiretablebox}
\begin{tabular}{@{}l|l|lllll|lllll@{}}
\toprule
\multirow{3}{*}{Model} & \multirow{3}{*}{Method} & \multicolumn{5}{c|}{MRAG-BENCH} & \multicolumn{5}{c}{Visual-RAG} \\ \cmidrule(l){3-12} 
 &  & \multicolumn{5}{c|}{Top K} & \multicolumn{5}{c}{Top K} \\ \cmidrule(l){3-12} 
 &  & 1 & 2 & 3 & 4 & 5 & 1 & 2 & 3 & 4 & 5 \\ \midrule
\multirow{3}{*}{Qwen3-VL-8B} & Ours & 65.71 & 65.34 & 67.41 & 66.74 & 67.11 & 62.43 & 64.84 & 63.37 & 64.44 & 64.57 \\
 & Verbalized UQ & 62.82(-2.89) & 63.71(-1.63) & 64.67(-2.74) & 65.34(-1.40) & 65.56(-1.55) & 59.36(-3.07) & 61.76(-3.08) & 61.10(-2.27) & 62.97(-1.47) & 65.64(+1.07) \\
 & Listwise Ranking & 58.54(-7.17) & 61.35(-3.99) & 63.27(-4.14) & 64.01(-2.73) & 65.11(-2.00) & 63.24(+0.81) & 63.37(-1.47) & 61.90(-1.47) & 66.44(+2.00) & 66.71(+2.14) \\ \midrule
\multirow{3}{*}{MiniCPM-V4.5} & Ours & 65.11 & 66.52 & 66.89 & 66.08 & 65.71 & 59.63 & 58.16 & 61.23 & 61.23 & 62.03 \\
 & Verbalized UQ & 58.46(-6.65) & 59.72(-6.80) & 61.27(-5.62) & 61.71(-4.37) & 62.82(-2.89) & 58.82(-0.81) & 59.49(+1.33) & 61.90(+0.67) & 62.83(+1.60) & 61.76(-0.27) \\
 & Listwise Ranking & 59.35(-5.76) & 60.46(-6.06) & 61.86(-5.03) & 62.68(-3.40) & 63.86(-1.85) & 62.43(+2.80) & 62.83(+4.67) & 62.30(+1.07) & 62.03(+0.80) & 61.63(-0.40) \\ \midrule
\multirow{3}{*}{Gemma3-12B} & Ours & 59.87 & 62.6 & 61.94 & 62.68 & 63.12 & 56.55 & 59.76 & 60.43 & 58.69 & 60.03 \\
 & Verbalized UQ & 55.65(-4.22) & 57.72(-4.88) & 58.91(-3.03) & 60.01(-2.67) & 60.46(-2.66) & 56.68(+0.13) & 54.95(-4.81) & 57.22(-3.21) & 57.09(-1.60) & 58.82(-1.21) \\
 & Listwise Ranking & 56.91(-2.96) & 58.91(-3.69) & 59.65(-2.29) & 59.72(-2.96) & 59.72(-3.40) & 59.09(+2.54) & 56.95(-2.81) & 58.82(-1.61) & 60.16(+1.47) & 60.56(+0.53) \\ \midrule
\multirow{3}{*}{Ovis2.5-9B} & Ours & 61.64 & 62.53 & 62.53 & 62.45 & 63.19 & 70.05 & 59.63 & 60.96 & 62.03 & 61.23 \\
 & Verbalized UQ & 57.95(-3.69) & 58.83(-3.70) & 58.61(-3.92) & 60.90(-1.55) & 60.09(-3.10) & 60.03(-10.02) & 58.96(-0.67) & 58.82(-2.14) & 58.16(-3.87) & 59.09(-2.14) \\
 & Listwise Ranking & 58.17(-3.47) & 58.46(-4.07) & 58.91(-3.62) & 59.94(-2.51) & 60.53(-2.66) & 58.56(-11.49) & 60.70(+1.07) & 59.49(-1.47) & 61.10(-0.93) & 60.70(-0.53) \\ \midrule
\multirow{3}{*}{InternVL3.5-8B} & Ours & 48.41 & 48.04 & 47.75 & 46.49 & 46.56 & 61.23 & 56.95 & 58.16 & 57.09 & 57.22 \\
 & Verbalized UQ & 43.53(-4.88) & 43.68(-4.36) & 44.05(-3.70) & 43.75(-2.74) & 43.83(-2.73) & 48.93(-12.30) & 53.07(-3.88) & 52.01(-6.15) & 54.81(-2.28) & 55.35(-1.87) \\
 & Listwise Ranking & 43.39(-5.02) & 45.08(-2.96) & 45.90(-1.85) & 45.16(-1.33) & 44.64(-1.92) & 54.55(-6.68) & 55.88(-1.07) & 52.94(-5.22) & 52.41(-4.68) & 52.01(-5.21) \\ \bottomrule
\end{tabular}
\end{lrbox}

\resizebox{\linewidth}{!}{\usebox{\myentiretablebox}}

\caption{We compare different utility estimation and evidence ranking strategies under the same experimental setting, including our logit-based discriminative utility estimation, verbalized uncertainty quantification, and listwise ranking that jointly considers all candidate evidence. Results are reported on MRAG-Bench and Visual-RAG across different Top-K settings.}

\label{tab:appendix_ablation}

\end{table*}

\section{Additional Experimental Results}
\label{appendix_sec:addtional_results}
This appendix provides the complete experimental results referenced in the main paper. Specifically, we report the full Top-$K$ ($K=1$ to $5$) performance tables for all evaluated models, benchmarks, and evidence selection methods.

These tables (\cref{tab:appendix_main_results,tab:appendix_Y_vs_A,tab:appendix_surrogate_vs_main_model,tab:appendix_ablation}) serve as a comprehensive extension of \cref{tab:exp-main-results,tab:Y_vs_A,tab:surrogate_vs_main_model,tab:ablation} in the main paper, which present representative or summarized results for clarity. No additional experimental settings are introduced in this appendix, and all results follow the same evaluation protocols described in Section~\ref{sec:experiments}.

\subsection{Prompt Sensitivity Results}
\label{app:subsec:prompt_sensitivity_results}

Table~\ref{tab:prompt_sensitivity} presents the prompt sensitivity experiment results. Across settings, performance remains close to that of the original template, with no consistent degradation pattern. These results suggest that the probe is robust to moderate lexical, syntactic, and output-token variations.

\begin{table}[]
\resizebox{0.95\linewidth}{!}{
\begin{tabular}{l|l|l|l|lll}
\hline
Dataset & Main Model & Surrogate Model & Prompt Variant & Top-1 & Top-3 & Top-5 \\ \hline
\multirow{8}{*}{MRAG-Bench} & \multirow{4}{*}{Qwen3-VL-8B} & \multirow{4}{*}{Qwen3-VL-2B} & Original & 65.56 & 67.55 & 67.55 \\
 &  &  & V1 & 64.75 & 66.89 & 67.63 \\
 &  &  & V2 & 65.34 & 66.81 & 66.81 \\
 &  &  & V3 & 64.75 & 66.81 & 66.89 \\ \cline{2-7} 
 & \multirow{4}{*}{Ovis2.5-9B} & \multirow{4}{*}{Ovis2.5-2B} & Original & 60.16 & 62.01 & 62.38 \\
 &  &  & V1 & 60.61 & 61.79 & 62.53 \\
 &  &  & V2 & 60.09 & 62.68 & 62.97 \\
 &  &  & V3 & 60.01 & 61.79 & 61.86 \\ \hline
\multirow{8}{*}{Visual-RAG} & \multirow{4}{*}{Qwen3-VL-8B} & \multirow{4}{*}{Qwen3-VL-2B} & Original & 59.89 & 63.77 & 64.84 \\
 &  &  & V1 & 60.83 & 62.57 & 64.84 \\
 &  &  & V2 & 60.03 & 63.64 & 63.90 \\
 &  &  & V3 & 58.56 & 61.10 & 63.64 \\ \cline{2-7} 
 & \multirow{4}{*}{Ovis2.5-9B} & \multirow{4}{*}{Ovis2.5-2B} & Original & 69.12 & 60.70 & 60.70 \\
 &  &  & V1 & 62.43 & 58.69 & 59.36 \\
 &  &  & V2 & 60.29 & 58.56 & 58.02 \\
 &  &  & V3 & 58.82 & 57.75 & 57.35 \\ \hline
\end{tabular}}
\caption{Prompt Sensitivity Analysis}
\label{tab:prompt_sensitivity}
\end{table}

\subsection{Robustness under Noisy Candidate Pools}
\label{app:noisy_candidate_pool}

To assess the practical robustness of Assumptions~\ref{assumption:helpfulevidence_main} and~\ref{assumption_informal:helpfulnessscore}, we construct three progressively perturbed candidate-pool regimes:
\begin{itemize}
    \item \textbf{Pure Retrieve:} the candidate pool is formed solely from retrieved images, reflecting realistic end-to-end deployment.
    \item \textbf{GT + Hard Negatives:} the pool includes ground-truth images together with top-ranked but non-relevant retrieved images, yielding structured ambiguity.
    \item \textbf{GT + Hard Negatives + Stochastic Perturbation:} additional randomly sampled non-relevant images are introduced on top of the previous regime, creating both structured and unstructured noise.
\end{itemize}

For each regime, we vary the pool size from 10 to 20. Across both backbones and both benchmarks, performance remains stable under increasing perturbation, supporting the practical robustness of the proposed ranking criterion.

\begin{table*}[]
\resizebox{0.95\linewidth}{!}{
\begin{tabular}{llll|lll|lll}
\hline
\multicolumn{4}{l|}{Benchmark} & \multicolumn{3}{c|}{MRAG-Bench} & \multicolumn{3}{c}{Visual-RAG} \\ \hline
\multicolumn{1}{l|}{Main Model} & \multicolumn{1}{l|}{Surrogate Model} & \multicolumn{1}{l|}{Candidate Pool Variant} & Candidate Pool Size & Top-1 & Top-3 & Top-5 & Top-1 & Top-3 & Top-5 \\ \hline
\multicolumn{1}{l|}{\multirow{9}{*}{Qwen3-VL-8B}} & \multicolumn{1}{l|}{\multirow{9}{*}{Qwen3-VL-2B}} & \multicolumn{1}{l|}{\multirow{3}{*}{Pure Retrieve}} & 10 & 58.76 & 51.79 & 55.36 & 56.42 & 58.56 & 60.29 \\
\multicolumn{1}{l|}{} & \multicolumn{1}{l|}{} & \multicolumn{1}{l|}{} & 15 & 58.68 & 50.00 & 43.75 & 57.62 & 58.69 & 59.76 \\
\multicolumn{1}{l|}{} & \multicolumn{1}{l|}{} & \multicolumn{1}{l|}{} & 20 & 58.61 & 50.00 & 43.75 & 57.75 & 59.89 & 60.43 \\ \cline{3-10} 
\multicolumn{1}{l|}{} & \multicolumn{1}{l|}{} & \multicolumn{1}{l|}{\multirow{3}{*}{GT + Hard Negatives}} & 10 & 60.90 & 61.64 & 68.75 & 60.56 & 63.24 & 65.91 \\
\multicolumn{1}{l|}{} & \multicolumn{1}{l|}{} & \multicolumn{1}{l|}{} & 15 & 60.53 & 62.94 & 64.29 & 60.03 & 62.70 & 64.04 \\
\multicolumn{1}{l|}{} & \multicolumn{1}{l|}{} & \multicolumn{1}{l|}{} & 20 & 60.31 & 63.79 & 66.67 & 60.03 & 64.04 & 63.10 \\ \cline{3-10} 
\multicolumn{1}{l|}{} & \multicolumn{1}{l|}{} & \multicolumn{1}{l|}{\multirow{3}{*}{GT + Hard Negatives + Stochastic Perturbation}} & 10 & 60.90 & 61.57 & 68.75 & 60.16 & 63.77 & 65.51 \\
\multicolumn{1}{l|}{} & \multicolumn{1}{l|}{} & \multicolumn{1}{l|}{} & 15 & 60.53 & 61.05 & 66.96 & 59.89 & 63.77 & 64.97 \\
\multicolumn{1}{l|}{} & \multicolumn{1}{l|}{} & \multicolumn{1}{l|}{} & 20 & 60.46 & 61.12 & 66.96 & 60.16 & 63.90 & 64.71 \\ \hline
\multicolumn{1}{l|}{\multirow{9}{*}{Ovis2.5-9B}} & \multicolumn{1}{l|}{\multirow{9}{*}{Ovis2.5-2B}} & \multicolumn{1}{l|}{\multirow{3}{*}{Pure Retrieve}} & 10 & 56.98 & 56.61 & 55.95 & 57.62 & 56.15 & 56.55 \\
\multicolumn{1}{l|}{} & \multicolumn{1}{l|}{} & \multicolumn{1}{l|}{} & 15 & 56.47 & 56.54 & 56.10 & 58.42 & 56.15 & 57.49 \\
\multicolumn{1}{l|}{} & \multicolumn{1}{l|}{} & \multicolumn{1}{l|}{} & 20 & 56.47 & 56.54 & 56.10 & 57.22 & 56.15 & 59.49 \\ \cline{3-10} 
\multicolumn{1}{l|}{} & \multicolumn{1}{l|}{} & \multicolumn{1}{l|}{\multirow{3}{*}{GT + Hard Negatives}} & 10 & 58.91 & 58.61 & 59.50 & 61.23 & 59.09 & 59.76 \\
\multicolumn{1}{l|}{} & \multicolumn{1}{l|}{} & \multicolumn{1}{l|}{} & 15 & 58.91 & 58.68 & 59.05 & 60.29 & 58.69 & 58.42 \\
\multicolumn{1}{l|}{} & \multicolumn{1}{l|}{} & \multicolumn{1}{l|}{} & 20 & 58.31 & 58.31 & 58.68 & 59.76 & 57.09 & 60.70 \\ \cline{3-10} 
\multicolumn{1}{l|}{} & \multicolumn{1}{l|}{} & \multicolumn{1}{l|}{\multirow{3}{*}{GT + Hard Negatives + Stochastic Perturbation}} & 10 & 59.05 & 58.83 & 59.72 & 60.96 & 58.82 & 59.09 \\
\multicolumn{1}{l|}{} & \multicolumn{1}{l|}{} & \multicolumn{1}{l|}{} & 15 & 58.83 & 58.76 & 60.09 & 60.56 & 58.29 & 57.75 \\
\multicolumn{1}{l|}{} & \multicolumn{1}{l|}{} & \multicolumn{1}{l|}{} & 20 & 58.83 & 58.91 & 60.09 & 60.16 & 58.56 & 57.75 \\ \hline
\end{tabular}}
\caption{Empirical validation of Assumption~\ref{assumption:helpfulevidence_main} and \ref{assumption_informal:helpfulnessscore}}
\label{tab:emp-val-assumptions}
\end{table*}

\subsection{Empirical Validation of Theorem~\ref{thm:equivalent1}}
\label{app:subsec:emp-val-thm1}

\begin{table}[]
\begin{tabular}{l|ll|ll}
\hline
 & \multicolumn{2}{c|}{GT} & \multicolumn{2}{c}{non-GT} \\ \hline
Variable & Mean & Median & Mean & Median \\ \hline
Information Gain (KL divergence) & 0.3816 & 0.0723 & 0.3236 & 0.0790 \\ \hline
Normalized logit ("True") & 0.5868 & 0.6792 & 0.3385 & 0.2227 \\ \hline
\end{tabular}
\caption{Empirical validation of Theorem~\ref{thm:equivalent1}}
\label{emp-val-thm-equivalent1}
\end{table}

We further examine whether latent helpfulness scores exhibit monotonic alignment with answer-space belief shifts, as suggested by the theoretical formulation.

Directly computing answer-space information gain over open-ended outputs is intractable (Section~\ref{sec:info_gain}). We therefore construct a restricted empirical proxy in a controlled multiple-choice setting on MRAG-Bench. Specifically, we approximate $IG(Z; C=c)$ using the surrogate model's helpfulness logit (Ovis2.5-2B), i.e., the logit of the \texttt{True} decision in the auxiliary binary task. For the answer-space term, we approximate the model output distribution $P_Y(\cdot)$ using normalized choice-token probabilities, and estimate $IG(Y; C=c)$ as the KL divergence between (i) the choice-token distribution conditioned on candidate image $c$ and (ii) the corresponding zero-shot distribution without external evidence.

This restricted setup is not intended to optimize the full answer-space objective directly; rather, it provides a tractable proxy for testing whether higher latent utility scores correspond to larger shifts in the target model's answer distribution.

Table~\ref{emp-val-thm-equivalent1} shows a consistent positive alignment between latent helpfulness signals and answer-space belief shifts. Candidates assigned higher surrogate helpfulness confidence tend to induce larger KL divergence in the target model's answer distribution. In particular, ground-truth (GT) images, which typically receive higher surrogate confidence, also exhibit larger mean KL divergence than non-GT candidates. More broadly, the ordering induced by surrogate helpfulness scores is positively aligned with the ordering of estimated answer-space information gain. This monotonic empirical pattern is consistent with the implication of Theorem~\ref{thm:equivalent1}, and provides supporting evidence for the practical validity of the surrogate approximation.

\subsection{Fine-Grained False Positive Analysis on MRAG-Bench}
\label{app:subsec:finegrained-false-pos-analysis}

\begin{table}[]
\resizebox{0.95\linewidth}{!}{
\begin{tabular}{l|l|l|l|l|l}
\hline
Main Model & Surrogate Model & Scenario & \#Candidate Images & \#False Positive Candidates & Percentage \\ \hline
\multirow{9}{*}{Qwen3-VL-8B} & \multirow{9}{*}{Qwen3-VL-2B} & Angle & 2763 & 69 & 2.5 \\
 &  & Biological & 957 & 6 & 0.63 \\
 &  & Deformation & 966 & 33 & 3.42 \\
 &  & Incomplete & 546 & 37 & 6.78 \\
 &  & Obstruction & 885 & 13 & 1.47 \\
 &  & Others & 1007 & 21 & 2.09 \\
 &  & Partial & 2196 & 13 & 0.59 \\
 &  & Scope & 857 & 10 & 1.17 \\
 &  & Temporal & 1346 & 57 & 4.23 \\ \hline
\multirow{9}{*}{Ovis2.5-9B} & \multirow{9}{*}{Ovis2.5-2B} & Angle & 2763 & 45 & 1.63 \\
 &  & Biological & 957 & 15 & 1.57 \\
 &  & Deformation & 966 & 48 & 4.97 \\
 &  & Incomplete & 546 & 26 & 4.76 \\
 &  & Obstruction & 885 & 11 & 1.24 \\
 &  & Others & 1007 & 21 & 2.09 \\
 &  & Partial & 2196 & 20 & 0.91 \\
 &  & Scope & 857 & 8 & 0.93 \\
 &  & Temporal & 1346 & 33 & 2.45 \\ \hline
\end{tabular}}
\caption{Per-catagory False Positive Analysis of Surrogate Model}
\label{tab:per-catagory-false-pos-analysis}
\end{table}

Table~\ref{tab:per-catagory-false-pos-analysis} reports a fine-grained false positive analysis across scenario categories in MRAG-Bench. We observe relatively higher FP rates in \emph{Incomplete}, \emph{Deformation}, and \emph{Temporal} categories. These scenarios typically require transformation-invariant and structure-aware visual reasoning, such as recognizing objects under temporal change, geometric deformation, or partial occlusion.

Such cases depend less on surface-level similarity and more on modeling object continuity and structural consistency. The elevated disagreement in these categories suggests that smaller surrogate models and larger main models may differ in how they encode object structure and physical transformations. Larger models may develop stronger invariance to deformation and temporal variation, which can lead to divergent ranking judgments in these settings.

Importantly, even in these more challenging categories, the absolute FP rates remain modest. Overall, the analysis suggests that: (i) ranking disagreement between surrogate and main models is limited in magnitude; (ii) when disagreements occur, they are concentrated in structurally complex scenarios involving transformation or object continuity; and (iii) these findings reveal interpretable representational differences across model scales while still supporting the overall alignment between surrogate and main models.

\section{Extended Case Studies}
\label{app:case}
Figure~\ref{fig:case_study}(a) presents an example from MRAG-Bench that compares relevance-based methods (CLIP-Large and GME) with our approach ( using Qwen3-VL-2B as the surrogate model), for identifying the exact model of a car. In this case, the query image depicts a McLaren 675LT. However, relevance-based methods incorrectly select an image of a McLaren 650S, likely because the retrieved candidate shares the same red color with the query image, a salient but irrelevant attribute. Although visually similar, the retrieved image does not provide the discriminative visual cues required to answer the question correctly. 

Figure~\ref{fig:case_study}(b) presents an example from Visual-RAG that contrasts answer-level UQ with our discriminative utility estimation. The task is to identify the inside color of a black woodpecker’s mouth. For average token probability-based UQ, the model assigns high confidence ($0.98$) to an image that does not reveal the interior of the mouth, leading to an incorrect selection. Similarly, MC sampling-based UQ favors an uninformative side-view image, where identical sampled answers (“yellow”) yield a misleadingly high consistency score ($1.0$). Conversely, the truly helpful image induces multiple plausible but different answers (“red” in one sample and “pink” in others), resulting in a lower consistency score ($0.81$). By contrast, our method assigns a substantially higher utility score to the image that exposes the inside of the mouth ($0.97$), while assigning near-zero scores to uninformative candidates. This example highlights that answer confidence does not necessarily reflect evidence utility, and that directly estimating evidence helpfulness leads to more reliable selection.

\subsection{Failure Case Analysis}
\label{app:failure_case_flower}

\paragraph{Task.}
\textbf{Question:} Among these characteristics, which one is improbable for this fruit after oxidation?

\textbf{Choices:}
(A) Its skin develops a white, powdery substance.
(B) It shrivels and becomes wrinkled.
(C) It develops blue molds.
(D) It develops dark brown or black spots.

\textbf{Ground-truth answer:} (C) It develops blue molds.

Figure~\ref{fig:failure-case-flower} shows a representative failure case in which the selected evidence is semantically related to the query scene but not truly informative for the target reasoning task. The query asks which visual characteristic is \emph{improbable} for the fruit after oxidation, and the correct answer is \textbf{C} (\emph{It developed blue molds}). However, our method selects a flower image as the top evidence, with a predicted helpfulness score of $p_{\text{true}}=0.852$, which is higher than the two ground-truth evidence images ($0.706$ and $0.777$).

A plausible explanation is that the surrogate confuses \emph{botanical or contextual relevance} with \emph{state-relevant utility}. The selected flower image shares plant-level visual cues with the query image, such as green foliage, vine-like structure, and a natural growing context, and may therefore appear semantically related to the underlying plant or fruit category. However, the question is not about identifying the fruit species; rather, it requires reasoning about \emph{appearance changes after oxidation}. For this question, the truly useful evidence should depict oxidized or spoiled grapes, such as wrinkling, darkening, or surface changes, as shown by the ground-truth images.

This example highlights a limitation of the helpfulness probe: when the task depends on \emph{state-specific visual changes} rather than object identity, the surrogate may assign a high score to candidates that are botanically related but do not contain the discriminative evidence needed to answer the question. In other words, the model captures \emph{what the object is associated with}, but fails to capture \emph{how the object changes under the queried condition}.

\begin{figure*}[t]
    \centering
    \begin{subfigure}[t]{0.23\textwidth}
        \centering
        \includegraphics[width=\linewidth]{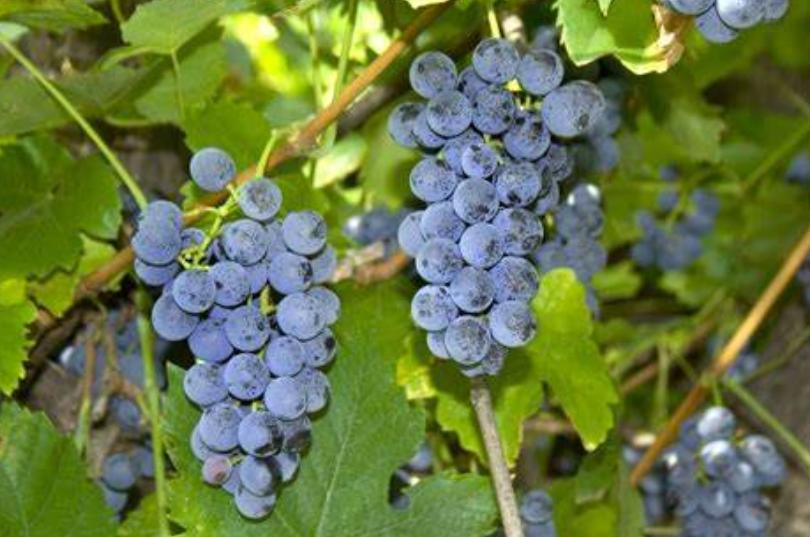}
        \caption{Query image.}
        \label{fig:failure-case-query}
    \end{subfigure}
    \hfill
    \begin{subfigure}[t]{0.23\textwidth}
        \centering
        \includegraphics[width=\linewidth]{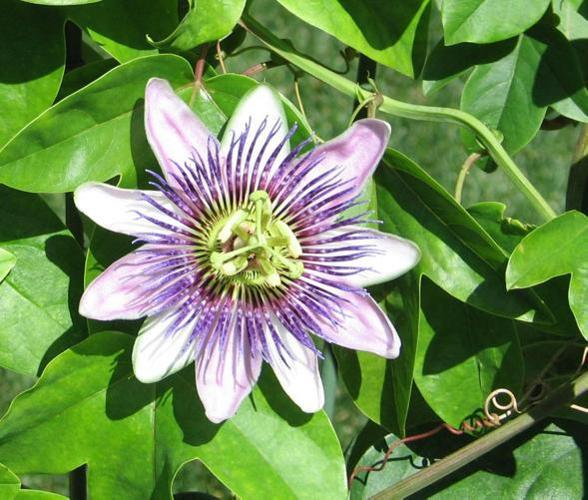}
        \caption{Top evidence selected by our method. Although assigned the highest helpfulness score, it depicts a flower rather than the target fruit.}
        \label{fig:failure-case-selected}
    \end{subfigure}
    \hfill
    \begin{subfigure}[t]{0.23\textwidth}
        \centering
        \includegraphics[width=\linewidth]{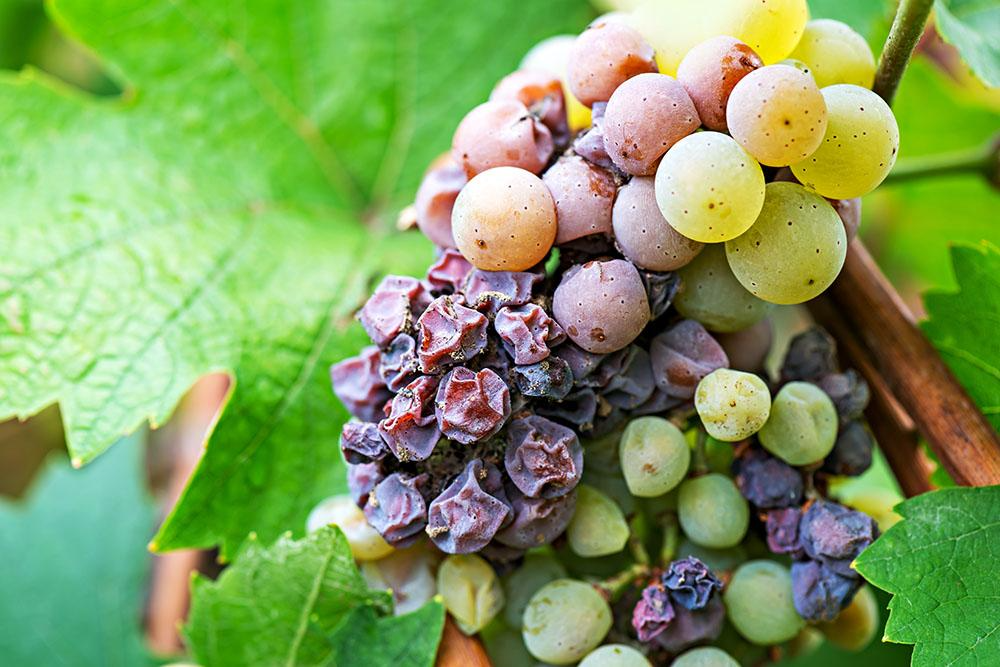}
        \caption{Ground-truth evidence image 1. This image directly reflects oxidation-related changes of the fruit.}
        \label{fig:failure-case-gt1}
    \end{subfigure}
    \hfill
    \begin{subfigure}[t]{0.23\textwidth}
        \centering
        \includegraphics[width=\linewidth]{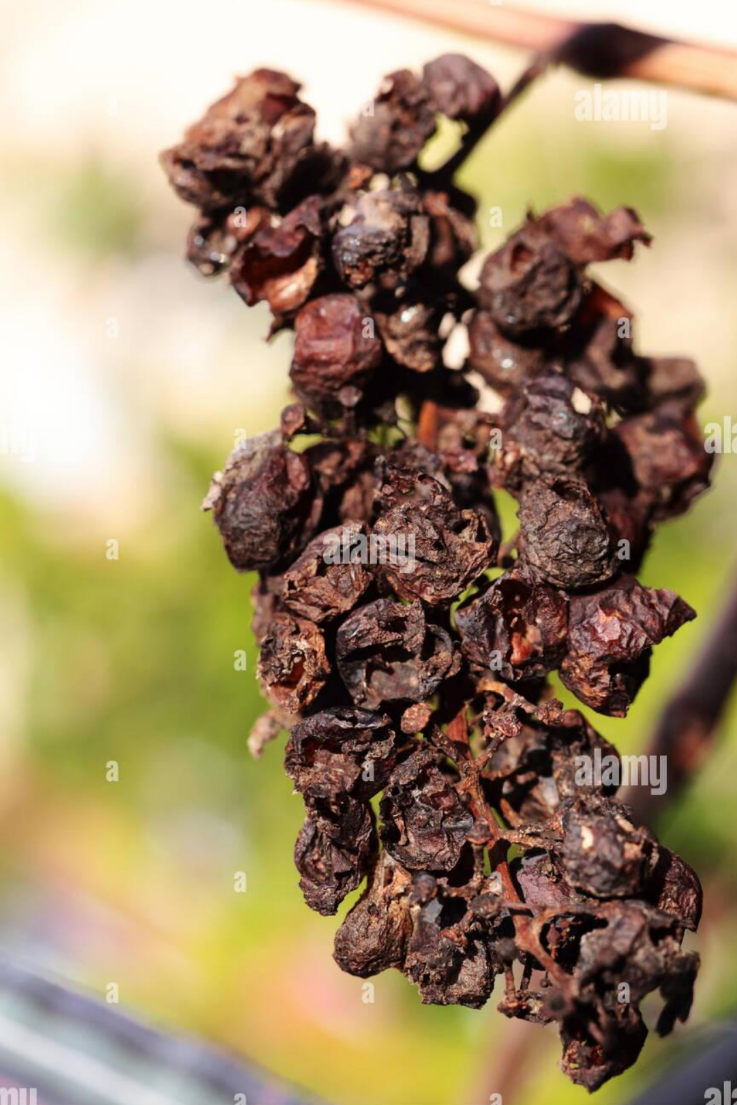}
        \caption{Ground-truth evidence image 2. This image also provides direct evidence of post-oxidation fruit appearance.}
        \label{fig:failure-case-gt2}
    \end{subfigure}

    \caption{A representative failure case of surrogate-based helpfulness estimation. The query asks about the improbable visual characteristic of a fruit after oxidation. Our method assigns the highest helpfulness score to a flower image, likely because it is botanically or contextually related to the query scene. However, this evidence is not useful for the actual reasoning objective, which requires recognizing oxidation-induced state changes of grapes. In contrast, the two ground-truth evidence images directly depict relevant post-oxidation appearances and are more informative for answering the question.}
    \label{fig:failure-case-flower}
\end{figure*}

\section{Acknowledgment of AI Assistance in Writing and Revision}
We utilized LLMs for revising and enhancing writing of this paper.

\end{document}